\documentclass{article}
\usepackage{iclr-2027-style-files/iclr2027/iclr2027_conference,times}
\usepackage[colorlinks=true,urlcolor=blue,linkcolor=red,citecolor=teal,backref=page]{hyperref}
\usepackage{array}
\usepackage{makecell}
\usepackage{float,xcolor,tcolorbox,longtable, tabularx}
\usepackage[textsize=tiny]{todonotes}

\usepackage{mathtools}
\usepackage{amsmath, amssymb, natbib, graphicx, url}
\usepackage{algorithm}
\usepackage{algpseudocode}
\usepackage{wrapfig}
\usepackage{booktabs}

\usepackage{tcolorbox}
\usepackage[shortlabels]{enumitem}
\usepackage{dsfont, bm}
\usepackage{footnote}
\usepackage{caption}
\usepackage{subcaption}
\usepackage{mwe}
\usepackage{cancel}
\usepackage{multirow}

\makesavenoteenv{tabular}
\makesavenoteenv{table}

\newcommand{\E}{\mathbb{E}}

\newcommand{\algo}{\textsc{V-share}\xspace}

\usepackage{amsthm}
\newtheorem{theorem}{Theorem}[section]
\newtheorem{corollary}[theorem]{Corollary}
\newtheorem{lemma}[theorem]{Lemma}
\newtheorem{proposition}[theorem]{Proposition}

\newtheorem{remark}{Remark}

\newtheorem{assumption}{Assumption}

\makeatletter
\newtheorem*{rep@theorem}{\rep@title}
\newcommand{\newreptheorem}[2]{%
	\newenvironment{rep#1}[1]{%
		\def\rep@title{\textbf{#2} \ref{##1}}%
		\begin{rep@theorem}}%
		{\end{rep@theorem}}}
\makeatother
\newreptheorem{theorem}{Theorem}
\newreptheorem{lemma}{Lemma}
\newreptheorem{proposition}{Proposition}
\newreptheorem{assumption}{Assumption}
\newreptheorem{corollary}{Corollary}

\allowdisplaybreaks%
\newif\ifsinglecol
\singlecolfalse

\usepackage[nohints]{minitoc}

\usepackage{mathrsfs}

\usepackage{tikz}
\usetikzlibrary{arrows.meta}
\usetikzlibrary{automata}
\usetikzlibrary{bayesnet, arrows, calc, shapes, backgrounds,arrows,decorations.pathmorphing,fit,positioning}
\tikzset{
   container/.style = {rectangle, rounded corners, draw=yellow, dashed,
fit=#1, inner sep=6mm, node contents={}},
circle-label/.style = {circle, draw}
        }
\tikzset{box/.style={draw, diamond, thick, text centered, minimum height=0.5cm, minimum width=1cm, text width=0.9cm}}
\tikzset{line/.style={draw, thick, -latex'}}
\allowdisplaybreaks
\usepackage{pgfplots}

\usepackage{layouts}

\def\bbR{\mathbb{R}}

\def\cA{\mathcal{A}}

\def\cO{\mathcal{O}}

\def\cS{\mathcal{S}}

\def\E{\mathbb{E}}

\def\det{{\sf det}}

\usepackage{amsmath,amsfonts,bm}

\def\eqref#1{equation~\ref{#1}}

\def\1{\bm{1}}

\DeclareMathAlphabet{\mathsfit}{\encodingdefault}{\sfdefault}{m}{sl}
\SetMathAlphabet{\mathsfit}{bold}{\encodingdefault}{\sfdefault}{bx}{n}

\newcommand{\R}{\mathbb{R}}

\usepackage{xcolor}

\title{\algo: Provable Convergence Guarantees for Advantage-based Algorithms}

\author{
Kintan Saha$^{\dagger}$,
Hamed Masoudi$^{*}$,
Xinyi Sheng$^{*}$,
Dominik Baumann$^{*}$,\\
\textbf{Joni Pajarinen}$^{*}$\textbf{,}
\textbf{Arsenii Mustafin}$^{*}$\\
\\
$^{\dagger}$Indian Institute of Science\\
\texttt{kintansaha@iisc.ac.in}\\
$^{*}$Aalto University\\
\texttt{\{hamed.masoudi,xinyi.sheng,dominik.baumann,}\\
\texttt{joni.pajarinen,arsenii.mustafin\}@aalto.fi}
}

\iclrfinalcopy
\begin{document}

\maketitle

\setcounter{parttocdepth}{3}
\doparttoc 
\faketableofcontents 
\begin{abstract} 

In reinforcement learning, advantage-based value-learning methods, unlike standard Q-learning, share information across actions by decomposing the action-value function into a state-dependent value component and action-specific advantages. While this mechanism has proven effective in both tabular and deep reinforcement learning, its theoretical understanding remains incomplete. We introduce \textbf{\algo}, an advantage-based value-learning algorithm that combines cross-action value sharing with an importance-weighted Bellman-residual correction. To our knowledge, we provide the first complete theoretical characterization of an advantage-based algorithm, including both asymptotic convergence and non-asymptotic finite-sample guarantees. We identify a regime in which \algo mitigates the rare-action bottleneck of asynchronous Q-learning and uncover a contraction--variance tradeoff between stronger deterministic contraction and increased importance-weighting noise. Motivated by this tradeoff, we develop decreasing residual-preconditioning schedules that combine faster early contraction with lower asymptotic noise. We further develop a two-timescale analysis that removes the single-timescale residual-weight threshold while preserving convergence guarantees. Experiments on discrete Markov decision processes (MDPs) and deep reinforcement-learning benchmarks corroborate our results.
\end{abstract}

\section{Introduction}


Reinforcement learning (RL) \citep{sutton1998reinforcement} algorithms learn through trial and error: they try actions, receive a reward in return indicating the quality of that action and, over time, try to learn an optimal mapping from the current state to the best action to take, which we call the policy.
Q-learning~\citep{watkins1992q} is one of the canonical RL algorithms that learns a Q-function as an intermediate object.
The Q-function estimates, for every state-action pair, the long-term return, i.e., the sum of received rewards over a trajectory, obtained by taking that action in that state and behaving optimally thereafter.
If the state and action spaces are finite, we can explicitly store estimated action-values for each state-action pair.
In the standard asynchronous setting, the learner observes one transition at a time and updates only the estimate for the sampled state-action pair. 
This simple coordinate-wise structure, together with the contraction of the Bellman optimality operator, underlies the classical convergence theory of Q-learning \citep{tsitsiklis1994asynchronous, jaakkola1993convergence}. 
The same basic principle also proved remarkably successful with function approximation through Deep Q-Networks (DQN)~\citep{mnih2015human}, showing that Q-learning can scale to high-dimensional control problems.

With function approximation, however, Q-learning becomes more susceptible to instability, and estimating the full action-value of every action separately can be inefficient when much of the information is shared. These observations
motivated the dueling architecture of \citet{wang2016dueling}, which represents the Q-function as the sum of a state-value function that assigns a value to every state defined by the long-term return obtained by behaving optimally after visiting said state, and an action-dependent advantage function learned separately. Intuitively,
rather than repeatedly estimating the full value of every action, the learner can first estimate
a potentially large value component common to all actions and then learn the comparatively
smaller action-dependent advantage components. This value-advantage perspective has subsequently proved useful both
empirically and theoretically. Advantage-based architectures improve learning
performance in deep reinforcement learning \citep{wang2016dueling, nagarajan2026accelerating}, while
recent tabular advantage-based algorithms exhibit accelerated learning relative to
standard Q-learning
\citep{tang2023va,daley2025analysis,lee2026spectral}. 

The algorithmic idea itself dates back to Advantage Updating
\citep{baird1993advantage}, and the central mechanism of all these approaches is \emph{value sharing}: updating the state-value component
simultaneously changes the represented values of all actions at the state being updated, allowing information
obtained from one action to propagate immediately to the others.

Value sharing, however, comes with a theoretical cost. In Q-learning, each
state-action value is updated through its own temporal-difference error, which leads
to a simple contraction-based analysis. Advantage-based methods couple the action-values
within a state through a shared value component, and this coupling breaks that standard
argument. Existing analyses therefore rely on stronger structural or sampling assumptions,
such as synchronous updates over all state-action pairs, to establish
convergence in the tabular setting
\citep{lee2026spectral,tang2023va}. This raises a natural question:

\emph{Can we retain the acceleration benefits of cross-action value sharing while recovering
a tractable convergence theory and finite-sample guarantees under stochastic sampling?}



\textbf{Contributions.}
We answer this question through \algo{}, an advantage-based value-learning
algorithm that augments cross-action value sharing with a tunable,
importance-weighted Bellman-residual correction parametrized by $\beta$. This leads to a residual
preconditioning view of value sharing, showing when cross-action coupling accelerates learning, when it breaks the usual contraction argument, and how to recover convergence guarantees.


\begin{enumerate}[leftmargin=*,nosep]
    \item We show that the iterates generated by \algo are stable and converge almost surely to the optimal action-value function $Q^*$ under Markovian sampling and suitable parameter choices.

    \item
    We derive finite-time guarantees using constant- and diminishing-stepsize schedules.
    The resulting rates quantify a contraction-variance tradeoff and yield the first finite-time guarantees of an advantage-based algorithm under Markovian sampling. We show how decreasing-$\beta$ and decay-to-zero stepsize schedules retain strong early
    contraction while reducing terminal stochastic error.

    \item
     We also analyze a two-timescale
    formulation that removes the single-timescale large-$\beta$ requirement and show convergence guarantees for all $\beta>0$. The analysis requires extending
    single-timescale Markovian stochastic approximation to vanishing drift
    perturbations, yielding results that may be of independent interest.

    \item
    We evaluate \algo  and its variants on tabular MDPs, where the experiments expose
    the predicted effects of value sharing. We
    additionally extend \algo  to deep reinforcement learning and compare it with recent
    advantage-based methods on MinAtar \citep{young2019minatar} and Classic Control tasks.
\end{enumerate}

\textbf{Related work.}
Since the number of Q-learning variants and analyses is vast, we focus here only on the key theoretical results, major practical milestones, and methods that explicitly learn advantages. The term \emph{advantage-based} is used in several different senses in
reinforcement learning. Policy-gradient and actor--critic methods such as PPO
\citep{schulman2017proximal} use advantage estimates to construct policy
updates, often through generalized advantage estimation
\citep{schulman2015high}. However, the advantage is used primarily as an estimator of
relative action quality to train the policy. Our focus is instead on
\emph{advantage-based value learning}: algorithms that explicitly represent
action values through a shared state-value component and action-specific
advantages, so that updating the value component propagates information across
actions. This viewpoint dates back to Advantage Updating \citep{baird1993advantage}
and underlies the dueling architecture of \citet{wang2016dueling}.

Recent work has revisited this value--advantage decomposition in both tabular
and deep RL \citep{tang2023va,daley2025analysis, nagarajan2026accelerating}. These methods  exploit the shared value component to
accelerate learning, while \cite{lee2026spectral} studies the dynamics of tabular
dueling Q-learning through spectral arguments. \cite{pmlr-v258-mustafin25a} introduces the RB-S algorithm, which uses advantage-preserving transformations to recover an optimal policy. Our approach is complementary:
we modify the coupled value-sharing dynamics through residual preconditioning
and analyze the resulting algorithm under Markovian
sampling, obtaining both asymptotic and finite-time guarantees. Technically,
our analysis builds on the ODE method for stochastic approximation
\citep{doi:10.1137/S0363012997331639,liu2025ode} and the generalized Moreau-envelope framework
for contractive stochastic approximation
\citep{NEURIPS2020_5d44ee6f,chen2024lyapunov}.

For brevity, we henceforth refer to advantage-based value-learning algorithms
as \emph{advantage-based algorithms}.


\section{Problem Setup}
\label{sec:background}

We first formalize the MDP and sampling model, and then express Q-learning and value-sharing updates that motivate our analysis.

\textbf{Discounted MDP and sampling model.}
A \textbf{finite discounted Markov decision process (MDP)} is defined by 
$(\mathcal S,\mathcal A,P,r,\gamma)$, where $\mathcal S$ and $\mathcal A$
are finite state and action spaces, respectively, $P(s'|s,a)$ is the transition kernel, $r(s,a) \in [0,1]$ is the bounded
reward, and $\gamma\in[0,1)$ is the discount factor. 
The goal is to find a policy $\pi:\,\mathcal S\to\mathcal A$ that maximizes the expected cumulative discounted reward $\E[\sum_{t=0}^\infty\gamma^tr_t]$.
For a given MDP, we define the Q-function as $Q(s,a) = \mathbb{E}_{\pi} [\sum_{k=0}^\infty\gamma^kr_{t+k}\mid s_t=s, a_t=a]$.
The
\textbf{Bellman optimality operator} $\mathcal T$ is
$(\mathcal TQ)(s,a)
=
r(s,a)+\gamma\sum_{s'}P(s'|s,a)\max_{a'}Q(s',a')$.
Its unique fixed point is the \textbf{optimal action-value function}
$Q^\star$.

Synchronous algorithms update all state-action pairs at every iteration,
which is computationally expensive and unrealistic in large MDPs. In the
asynchronous setting, updates are driven by individual transitions sampled
along a trajectory. Our algorithms operate in this regime, learning from a
single trajectory generated by a fixed behavior policy $\pi_b$. We make a standard assumption about this policy:

\begin{assumption}[Ergodic behavior policy]
\label{ass:behavior}
The behavior policy $\pi_b : \cS \rightarrow \Delta(\cA)$ has full support, $\pi_b(a|s)>0$ for every $(s,a)$, and the
state Markov chain induced by $\pi_b$ is irreducible and aperiodic. Hence it admits a unique
stationary distribution $d\in \R^{|\cS|}$ satisfying $d(s)>0$ for every $s$, and  stationary
state-action distribution $\mu(s,a)=d(s)\pi_b(a|s)$. Further, let
$Y_t:=(s_{t-1},a_{t-1},s_t)$ denote the Markov chain of observed transitions,
with stationary distribution $d_Y(s,a,s')=d(s)\pi_b(a|s)P(s'|s,a)$. For
$\delta>0$, define its mixing time by
$t_\delta:=\min\{m\ge 0:\max_y||\mathbb P(Y_m\in\cdot\mid Y_0=y)-d_Y||_{\mathrm{TV}}\le\delta\}$, where $||\cdot||_{\rm TV}$ is the total variation distance.
Since ${Y_t}$ is a finite-state ergodic Markov chain, it is geometrically
mixing, and hence $t_\delta=O(\log(1/\delta))$.
\end{assumption}
At iteration $t$, the sampled transition is
$(s_t,a_t,r(s_t, a_t),s_{t+1})$, where $a_t\sim \pi_b(\cdot|s_t)$ and
$s_{t+1}\sim P(\cdot|s_t,a_t)$. The Markovian sample is defined as $Y_{t+1}:=(s_{t},a_{t},s_{t+1})$. For the current iterate $Q_t$,
define the \textbf{temporal-difference (TD) error}
$
\delta_t
=
r(s_t, a_t)
+\gamma\max_{a'}Q_t(s_{t+1},a')
-Q_t(s_t,a_t).
$
For a stochastic-approximation (SA) recursion
$Q_{t+1}=Q_t+\alpha_t\widehat h_t(Q_t, Y_{t+1})$, we call
$h(Q):=\mathbb E_{Y \sim d_Y}[\widehat h_t(Q, Y)]$ its \textbf{stationary mean operator} and
$\dot Q=h(Q)$ its associated \textbf{limiting ODE}
\citep{borkar2022stochastic}. These dynamics let us compare Q-learning with cross-action value sharing.

\textbf{Q-learning and value sharing.}
Asynchronous \textbf{Q-learning} with stepsize schedule $\{\alpha_t\}_{t=0}^{\infty}$ updates only the sampled
state-action coordinate,
\begin{align*}
Q_{t+1}(s,a)
=
Q_t(s,a)
+
\alpha_t\mathbf 1\{s_t=s,a_t=a\}\delta_t, 
\end{align*}
where $\mathbf 1$ denotes the indicator function.
Let
$D_\mu:=\operatorname{diag}\{\mu(s,a)\}_{s,a}
=\operatorname{diag}\{d(s)\pi_b(a|s)\}_{s,a}$.
Its limiting ODE is
$
\dot Q=D_\mu(\mathcal TQ-Q).
$
Thus asynchronous Q-learning applies a diagonal sampling preconditioner to the
\textbf{Bellman residual} $\mathcal TQ-Q$, preserving the coordinate-wise
structure underlying its standard contraction analysis.

We are interested instead in algorithms that share information across actions.
To formalize this, consider the \textbf{value-advantage decomposition}
$Q_t(s,a)=V_t(s)+A_t(s,a)$ and the shared-value update of hard Regularized Dueling Q-learning (RDQ)
\citep{daley2025analysis},
\begin{align}
\label{eq:hard_rdq}
V_{t+1}(s_t)
=V_t(s_t)+\alpha_t\delta_t,\quad 
A_{t+1}(s_t,a)
=
A_t(s_t,a)
+
\alpha_t\mathbf 1\{a_t=a\}\delta_t .
\end{align}
Although $(V_t,A_t)$ is not uniquely identifiable, by defining $(LV)(s,a):=V(s)$, the iterate
$Q_t=LV_t+A_t$ is well defined.

To understand the effect of value sharing on the $Q$-dynamics, let
$b_s:=(\pi_b(a|s))_{a\in\mathcal A}$, \textbf{e} be the all-ones vector in $\R^{|\cA|}$, 
$B_s:=\textbf{e} b_s^\top$, and
$B:=\operatorname{blkdiag}(B_s:s\in\mathcal S)$. Also define the lifted
state-visitation matrix
$D:=\operatorname{diag}\{d(s):(s,a)\in\mathcal S\times\mathcal A\}$.
The  $Q$-iterate then satisfies
$$
Q_{t+1}(s,a)-Q_t(s,a)
=
\alpha_t\mathbf 1\{s_t=s\}
\bigl(1+\mathbf 1\{a_t=a\}\bigr)\delta_t,
$$
with limiting ODE
$
\dot Q=(DB+D_\mu)(\mathcal TQ-Q).
$

The distinction from asynchronous Q-learning is the following: The term
$D_\mu(\mathcal TQ-Q)$ performs the usual action-specific update, whereas
$DB(\mathcal TQ-Q)$ propagates information across all actions at the visited
state. This cross-action coupling is the value-sharing mechanism
we study in this paper. Appendix~\ref{app:background-derivations} derives the stationary mean operators of asynchronous Q-learning and hard RDQ and compares the
structure of their respective preconditioners and limiting ODEs.

\begin{algorithm}[H]
\caption{\algo{}}
\label{alg:rpq_single}
\begin{algorithmic}[1]
\State Initialize $V_0,A_0$, state $S_0$, stepsizes $\{\alpha_t\}_{t=0}^\infty$, and choose $\beta>0$.
\For{$t=0,1,\ldots$}
    \State Sample $a_t\sim \pi_b(\cdot\mid s_t)$.
    \State Observe $R_{t+1}=r(s_t,a_t)$ and
    $s_{t+1}\sim P(\cdot\mid s_t,a_t)$.
    \State Compute
    $$
    \delta_t
    =
    R_{t+1}
    +
    \gamma\max_{a'}Q_t(s_{t+1},a')
    -
    Q_t(s_t,a_t).
    $$
    \State Update, for all $s\in\cS$ and $a\in\cA$,
    $$
    V_{t+1}(s)
    =
    V_t(s)
    +
    \mathbf 1\{s_t=s\}\alpha_t\delta_t, \quad 
    A_{t+1}(s,a)
    =
    A_t(s,a)
    +
    \alpha_t
    \frac{\beta\,\mathbf 1\{s_t=s,a_t=a\}}
    {\pi_b(a\mid s_t)}
    \delta_t.
    $$
\EndFor
\end{algorithmic}
\end{algorithm}
\section{Main Results}
The value-sharing update \eqref{eq:hard_rdq} propagates a
sampled TD error across all actions at the visited state, but introduces
cross-action coupling that prevents the standard Q-learning contraction
argument. We therefore introduce \textbf{\algo} in section~\ref{sec:single_rpq} (Algorithm~\ref{alg:rpq_single}), which augments the shared value update with an
importance-weighted action-specific Bellman-residual correction. Our analysis casts \algo as a stochastic approximation scheme for a fixed-point iteration. In Section~\ref{sec:asymptotic_guarantees}, we prove almost-sure convergence of \algo
and provide finite-time guarantees in Section~\ref{sec:finite_time}.
Finally, we investigate convergence properties when we decrease the parameter $\beta$ of \algo over time in Section~\ref{sec:beta-schedules} and under a two-timescale formulation in Section~\ref{sec:two_timescale}.

\subsection{\algo  Algorithm}
\label{sec:single_rpq}

Let $\beta>0$ denote the \textbf{residual-preconditioning parameter}.
Given a sample $(s_t,a_t,r(s_t, a_t),s_{t+1})$ generated by the behaviour policy $\pi_b$, \algo  performs
\begin{align}
V_{t+1}(s_t)
&=V_t(s_t)+\alpha_t\delta_t,
&
A_{t+1}(s_t,a)
&=
A_t(s_t,a)
+
\alpha_t
\frac{\beta\,\mathbf 1\{a_t=a\}}{\pi_b(a|s_t)}
\delta_t,
\label{eq:rpq_va_update}
\end{align}
with $Q_t(s,a)=V_t(s)+A_t(s,a)$ and all unvisited-state coordinates
unchanged.

The induced action-value recursion admits the following stochastic
approximation (SA) form and limiting ODE,
\begin{equation}
Q_{t+1}
=
Q_t+\alpha_t
\bigl(h_\beta(Q_t)+M_{t+1}\bigr),
\quad
h_\beta(Q)
=
D(B+\beta I)(\mathcal TQ-Q), \quad \dot Q=h_\beta(Q),
\label{eq:rpq_sa}
\end{equation}
where $M_{t+1}$ denotes the Markovian noise induced by the trajectory
$\{Y_t\}$. The term $DB$ implements cross-action value sharing, whereas
$\beta D$ provides the diagonal residual correction. We show that the preconditioner matrix $D(B+\beta I)$ is invertible, thus ensuring that $Q^*$ is the fixed point of the limiting ODE in Appendix~\ref{app:rpq_contraction}.
\subsection{Asymptotic Guarantees}
\label{sec:asymptotic_guarantees}
\textbf{A contractive fixed-point representation.}
Our convergence analysis will use the Lyapunov framework for
\emph{contractive stochastic approximation} \citep{chen2024lyapunov}. To place
\eqref{eq:rpq_sa} in this form, we introduce an auxiliary analysis parameter
$\eta>0$ and define the \textbf{\algo  fixed-point map}
\begin{equation}
F_{\eta,\beta}(Q)
:=
Q+\eta h_\beta(Q).
\label{eq:rpq_fixed_point_map}
\end{equation}
Using Equation~\ref{eq:rpq_fixed_point_map}, Equation~\ref{eq:rpq_sa} can be rewritten as
\begin{equation}
\label{eq:rpq_contraction}
Q_{t+1}
=
Q_t+
\frac{\alpha_t}{\eta}
\left[F_{\eta,\beta}(Q_t)-Q_t+\eta M_{t+1}
\right].
\end{equation}
We show that $F_{\eta,\beta}$ is a contraction in the sup-norm
(Lemma~\ref{lem:app_rpq_contraction} in Appendix~\ref{app:rpq_contraction})
whenever
$\kappa_\beta
:=
\min_{s,a}
d(s)\big[(1-\gamma)\beta-(1+\gamma)+2\pi_b(a|s)\big]
>0$ and $\eta \le L_\beta^{-1}, L_\beta
:=
\max_{s,a}d(s)\bigl(\beta+\pi_b(a|s)\bigr)$. This contraction is used to establish the asymptotic convergence of the algorithm as follows.

\begin{theorem}[Almost-sure convergence of \algo under Markovian sampling]
\label{thm:rpq_asymptotic}
Under Assumption \ref{ass:behavior}, let $\kappa_\beta>0$, and
$\alpha_t=\alpha_0/(t+h)^p$ for some $p\in(1/2,1]$ and $\alpha_0, h>0$.
Then, the \algo iterates (\ref{eq:rpq_va_update}) are almost surely bounded and satisfy
$$
Q_t\to Q^\star
\qquad\text{a.s.}
$$
\end{theorem}

\textbf{Proof sketch.}
The full proof is given in Appendix \ref{app:rpq_asymptotic}. Let
$W(Q):=\|Q-Q^\star\|_\infty$. The contraction property established in
Lemma~\ref{lem:app_rpq_contraction} implies that the upper-right Dini
derivative~\citep{khalil2002nonlinear} of $W$ along the limiting ODE
$\dot Q=h_\beta(Q)$
satisfies
\[
D^+W(Q;h_\beta(Q))
\le -\kappa_\beta W(Q).
\]
Hence, $Q^\star$ is the unique globally exponentially stable equilibrium
of the limiting ODE.
Now, writing the \algo recursion as
$Q_{t+1}=Q_t+\alpha_tH_\beta(Q_t,Y_{t+1})$,
where $Y_t$ is the Markovian sample, we show that $H_\beta(Q,Y)$ is globally
Lipschitz, its scaled limit exists, and the averaged limiting ODE has the
origin as a globally exponentially stable equilibrium by the contraction
argument of Lemma~\ref{lem:app_rpq_contraction}. Assumption \ref{ass:behavior} verifies the required Markov averaging conditions of
\cite{liu2025ode}. \citet[Theorem 7]{liu2025ode} therefore gives
$\sup_t\|Q_t\|_\infty<\infty$ a.s. Their ODE convergence result \citep[Corollary 8]{liu2025ode} then implies
that the iterates approach a bounded invariant set of
$\dot Q=h_\beta(Q)$. Since this ODE has the unique globally exponentially
stable equilibrium $Q^\star$, the only bounded invariant set is
$\{Q^\star\}$, and hence $Q_t\to Q^\star$ a.s.

\textit{Implications of Theorem~\ref{thm:rpq_asymptotic}.} To our knowledge, Theorem~\ref{thm:rpq_asymptotic} gives the first
almost-sure convergence guarantee for an advantage-based algorithm under Markovian sampling. We show below that the same contractive
regime also yields mean-square error bounds for the iterates, thereby
providing finite-sample guarantees.

\subsection{Finite-Time Guarantees}
\label{sec:finite_time}
We now quantify the finite-sample behavior of \algo  under Markovian sampling.
Recall the SA recursion \eqref{eq:rpq_sa} and suppose that
$\kappa_\beta>0$, so that the \algo  fixed-point map is contractive by
Lemma~\ref{lem:app_rpq_contraction}. We state the results for constant stepsizes here. The bounds for diminishing and
decaying stepsizes are given in
Theorem~\ref{thm:app_finite_rpq} of Appendix~\ref{app:finite-time-rpq}.

\begin{theorem}[Finite-time convergence of \algo ]
\label{thm:finite_rpq}
Under Assumption~\ref{ass:behavior}, suppose $\kappa_\beta>0$. Let $\alpha_k\equiv\alpha$, $\alpha$ sufficiently small. Define, $K_{\rm mix} = \min\{k\ge 0: k \ge t_{\alpha L_\beta}\}$, $$
c_{2,\beta}
:=
\left(
(2+
\frac{(1+\beta/b_{\min})(1+\gamma)}{L_\beta})\|Q^\star\|_\infty+\frac{(1+\beta/b_{\min})}{L_\beta}
\right)^2.
$$
Then, $\forall k \ge K_{\rm mix}$,
\begin{align}
\mathbb E\|Q_k-Q^\star\|_\infty^2
\le
& \cO\left(||Q_0-Q^*||_\infty^2
\left(1-\frac{\alpha\kappa_\beta}{2}\right)^{
k-t_{\alpha L_\beta}}
+
\frac{\alpha L_\beta^3\log |\cS||\cA|}
{\kappa_\beta^2}c_{2,\beta }
\,t_{\alpha L_\beta}\right).
\label{eq:finite_rpq_constant}
\end{align}

\end{theorem}

\textbf{Proof sketch.}
We provide the main steps here; the complete proof is given in
Appendix~\ref{app:finite-time-rpq}.
\begin{enumerate}[leftmargin=*,nosep]
    \item \textbf{Contractive-SA reduction.}
    With $\eta=L_\beta^{-1}$, Lemma~\ref{lem:app_rpq_contraction} gives a contraction factor
    $\rho=1-\kappa_\beta/L_\beta$ for the \algo fixed point map. Hence the recursion is a contractive
    Markovian SA with effective stepsize $\epsilon_k=\alpha_kL_\beta$.

    \item \textbf{Moreau-envelope drift under Markovian sampling.}
    Apply the generalized Moreau-envelope construction of \citet{chen2024lyapunov} to obtain a smooth Lyapunov function equivalent
    to $\|Q-Q^\star\|_\infty^2$. Contractivity yields a negative drift of
    order $\epsilon_k\kappa_\beta/L_\beta$, while temporal dependence is
    controlled over the mixing window, producing the
    Markovian error term proportional to 
    $\epsilon_k\epsilon_{k-t_k,k-1}$ where $\alpha_{i,j} = \sum_{k=i}^j\alpha_k$.

    \item \textbf{One-step recursion and unrolling.}
    Substituting the \algo{} sample-map growth bounds into the Markovian
    Moreau recursion and imposing the mixing-window condition allows the
    correlation term to be absorbed into the negative drift. Unrolling the
    resulting scalar recursion for constant stepsizes yields
    \eqref{eq:finite_rpq_constant}.
\end{enumerate}

\textit{Implications of Theorem~\ref{thm:finite_rpq}.} To our knowledge, Theorem~\ref{thm:finite_rpq} establishes the first finite-time convergence guarantees for an advantage-based algorithm under the challenging regime of  Markovian sampling. 

\begin{remark}[Residual-preconditioning tradeoff]
\label{rem:beta_tradeoff}
For large $\beta$,
$L_\beta=\Theta(\beta)$ and $\kappa_\beta=\Theta(\beta)$, while
$c_{2,\beta}$ remains bounded in $\beta$ and scales as
$O(b_{\min}^{-2})$. Hence the constant-stepsize stochastic term in
Theorem~\ref{thm:finite_rpq} scales as (ignoring the dimensionality dependence)
$$
\nu_\beta := \frac{\alpha L_\beta^3c_{2,\beta}}{\kappa_\beta^2}
t_{\alpha L_\beta}
=
\Theta\!\left(
\alpha\beta\,c_{2,\beta}t_{\alpha L_\beta}
\right).
$$
Thus, larger $\beta$ strengthens transient contraction but enlarges the
constant-stepsize stochastic radius. We observe this effect in Figure~\ref{fig:tabular-constant}. This motivates the decreasing-$\beta$
schedules of Section~\ref{sec:beta-schedules}.
\end{remark}

\begin{remark}[Comparison with asynchronous Q-learning]
\label{rem:finite_q_comparison}
Asynchronous Q-learning contracts at gap
$\mu_{\min}(1-\gamma)$, where
$\mu_{\min}:=\min_{s,a}d(s)\pi_b(a|s)$, whereas for large $\beta$,
$\kappa_\beta\sim d_{\min}(1-\gamma)\beta$. Thus \algo{} removes the
rare-action factor from the leading deterministic contraction gap and yields a faster contraction.
For a detailed comparison, we defer the details to Appendix~\ref{app:qlearning-comparison}.
\end{remark}

\subsection{Decreasing Residual-Preconditioning Schedules}
\label{sec:beta-schedules}

Remark~\ref{rem:beta_tradeoff} shows that stronger residual preconditioning
improves transient contraction but enlarges the constant-stepsize stochastic
radius $\nu_\beta$. This suggests using a large value of $\beta$ early in learning and
decreasing it as the iterates approach $Q^\star$.

As a motivating example, consider the two-stage schedule
$$
\beta_k=
\begin{cases}
\beta_h, & k<K,\\
\beta_\ell, & k\ge K,
\end{cases}
\qquad
\beta_h\ge\beta_\ell,
\qquad
\kappa_{\beta_\ell}>0.
$$
The first stage contracts more aggressively, while the second stage reduces
the stochastic error induced by residual preconditioning. Importantly,
errors accumulated during the high-$\beta$ phase are themselves contracted
during the subsequent low-$\beta$ phase.

In the following theorem, for notational ease we skip the subscript $\beta$ in the various constants of Theorem~\ref{thm:finite_rpq} and instead write them with the subscripts $h$ and $l$ to denote $\beta_h$ and $\beta_l$ respectively.
Let $K_{\rm mix} = \min\{k\ge0: k \ge t_{\alpha L_h}\}$ and $K> K_{\rm mix}.$
\begin{theorem}[Two-stage residual preconditioning]
\label{thm:two_stage_beta}
Suppose the constant stepsize $\alpha$ satisfies the uniform Markovian
mixing-window condition over
$\beta\in[\beta_\ell,\beta_h]$. Then, $\forall T\ge K$,
\begin{align*}
\mathbb E\|Q_T-Q^\star\|_\infty^2
=\;&
\cO\left(
||Q_0-Q^*||_\infty^2
\exp\left\{
-\frac{\alpha}{2}
\left(
\kappa_h(K-t_{\alpha L_h})+\kappa_\ell(T-K-t_{\alpha L_l})
\right)
\right\}
\right)
\\
&+
\cO\left(
\nu_{h}
\exp\left\{
-\frac{\alpha\kappa_\ell}{2}(T-K)
\right\}+\nu_{l}
\right).
\end{align*}
\end{theorem}

\textbf{Proof sketch.}
The full proof is given in Appendix~\ref{app:beta_schedules}.
We apply the fixed-$\beta$ Markovian bound of
Theorem~\ref{thm:finite_rpq} separately on the two stages.
During the first $K$ iterations, we use
$\eta_h=L_h^{-1}$, yielding contraction at rate $\kappa_h$ and a
high-$\beta$ stochastic radius $\nu_h$.

At the switching time $K$, we restart the analysis conditionally on
$(Q_K,Y_K)$ with $\eta_\ell=L_\ell^{-1}$. Since the mixing-time bound is
uniform over the initial Markov state, Theorem~\ref{thm:finite_rpq}
applies to the second stage with the low-$\beta$ constants
$(L_\ell,\kappa_\ell,c_{2,\ell})$. Consequently, the error inherited from
the first stage is further contracted by
$\exp\{-\alpha\kappa_\ell(T-K)/2\}$, while noise generated after the switch
contributes only the low-$\beta$ stochastic radius
$\nu_l$.
Combining the two stagewise bounds gives the result.

\textit{Implications of Theorem~\ref{thm:two_stage_beta}.}
The bound makes the benefit of scheduling explicit: large $\beta$ can be
used to accelerate the early transient, while the effect of its larger
stochastic error is exponentially forgotten during the low-$\beta$ tail.
After a sufficiently long second stage, the dominant terminal term is $\nu_l$.
Thus, the low-$\beta$ phase governs the terminal stochastic scale. 


We derive explicit bounds for smoothly decreasing schedules such as linear, polynomial, and
normalized-exponential in
Appendix~\ref{app:beta_schedules}. We compare the performance of the various schedules in Figure~\ref{fig:tabular-beta-schedules}.

\subsection{Two-Timescale \algo }
\label{sec:two_timescale}

The preceding single-timescale theory requires $\beta$ to be large enough
that the \algo fixed point map is contractive (Lemma \ref{lem:app_rpq_contraction}). Two-timescale learning provides
a complementary route that removes this restriction entirely. The key idea is
to let the advantage component evolve faster than the shared value component.
For every fixed $\beta>0$, the fast advantage dynamics then reduce, after a
simple change of variables, to standard Bellman dynamics with contraction gap
$d_{\min}(1-\gamma)$.

Two-timescale \algo (Algorithm~\ref{alg:rpq_two}) uses
\begin{align}
V_{t+1}(S_t)
&=
V_t(S_t)+\alpha_t\delta_t,
&
A_{t+1}(S_t,a)
&=
A_t(S_t,a)
+
\eta_t
\frac{\beta\,\mathbf 1\{U_t=a\}}{\pi_b(a|S_t)}
\delta_t .
\end{align}
For fixed $\beta>0$, define the \textbf{effective advantage stepsize}
$\tau_t:=\beta\eta_t$.
We show the asymptotic convergence as follows:
\begin{theorem}[Almost-sure convergence of two-timescale \algo ]
\label{thm:two_timescale_asymptotic}
Fix any $\beta>0$ and let $\tau_t:=\beta\eta_t$. Suppose
Assumption~\ref{ass:behavior} holds and
$$
\tau_t=\frac{\tau_0}{(t+h)^p},
\qquad
\alpha_t=\frac{\alpha_0}{(t+h)^q},
\qquad
\frac12<p<q\le1.
$$
Then,
$
A_t-\bigl(Q^\star-LV_t\bigr)\to0
\quad\text{a.s.} \implies
Q_t=LV_t+A_t\to Q^\star
\quad\text{a.s.}
$
\end{theorem}

\textbf{Proof sketch.}
The full proof is given in Appendix~\ref{app:two_timescale_asymptotic}. For a frozen value function $V$, the fast limiting ODE is
$\dot A=D[\mathcal T(LV+A)-(LV+A)]$. Under the change of variables $Q=LV+A$, this becomes $\dot Q=D(\mathcal TQ-Q)$, whose unique globally exponentially stable equilibrium is $Q^\star$, with contraction gap $\kappa_{\rm fast}=d_{\min}(1-\gamma)$. Hence the frozen fast recursion has equilibrium $\lambda(V)=Q^\star-LV$.
The key observation is that the tracking error
$A_t-\lambda(V_t)=Q_t-Q^\star$
evolves as a single-timescale Markovian stochastic approximation on the fast timescale, with the slow value update appearing only through the vanishing perturbation as $\alpha_t/\tau_t\to0$. We therefore extend the Markovian ODE result of \citet{liu2025ode} to stochastic approximation with a vanishing drift perturbation. We show that under their assumptions, stability and convergence to invariant sets of the unperturbed limiting ODE are preserved (Proposition~\ref{prop:markov_vanishing_perturbation}). This is the Markovian analogue of the vanishing-measurement-error extension in \citet[Section~2.2]{borkar2022stochastic}. Since the unperturbed limiting ODE has the unique equilibrium $Q^\star$, it follows that $A_t-\lambda(V_t)=Q_t-Q^\star\to0$ a.s.

\textit{Implications of Theorem~\ref{thm:two_timescale_asymptotic}.}
\textit{First}, Theorem~\ref{thm:two_timescale_asymptotic} shows that timescale separation removes the single-timescale
threshold on $\beta$: every fixed $\beta>0$ is admissible once
$\alpha_t/\tau_t\to0$. \textit{Second}, the proof
establishes a Markovian analogue of the classical vanishing-measurement-error
extension of stochastic approximation, showing that the stability and ODE
limit theory of \citet{liu2025ode} is preserved under a state-dependent drift
perturbation whose coefficient vanishes in the limit. 


\textbf{Finite-time tracking.}
The asymptotic result shows that timescale separation eliminates the
single-timescale threshold on $\beta$. We next quantify how quickly the fast
advantage iterate tracks its moving equilibrium. Define
$
\Delta_t:=A_t-A^\star(V_t)=Q_t-Q^\star.
$
The leading dynamics of $\Delta_t$ inherit the Bellman contraction gap
$\kappa_{\rm fast}=d_{\min}(1-\gamma)$, while motion of the value component
acts only as a slower perturbation controlled by $\alpha_t/\tau_t$. Let $K:=\min\{k: k\ge t_{\tau_k}\}.$

\begin{theorem}[Finite-time tracking of two-timescale \algo{}]
\label{thm:two_timescale_tracking}
Let Assumption~\ref{ass:behavior} hold and the timescale separation
satisfy
$
\frac{\alpha_t}{\tau_t}
\le c_{sep} = \Theta(
\frac{\kappa_{\rm fast}}
{2d_{\max}(1+\gamma)})
\quad
\forall t.
$
For
$
\tau_t=\frac{\tau_0}{(t+h)^p},
\quad
\alpha_t=\frac{\alpha_0}{(t+h)^q},
\quad
\frac12<p<q\le1,
$ with appropriately chosen $\alpha_0, \tau_0, h$ yields $\forall t>K,$
$$
\mathbb E\|Q_t-Q^\star\|_\infty^2
=
\widetilde \cO\!\left(
\frac{1}{t^p}
\right).
$$
\end{theorem}

\textbf{Proof sketch.}
The full proof is given in
Appendix~\ref{app:two_timescale_finite}. As in the asymptotic analysis, define
$\Delta_t:=A_t-(Q^\star-LV_t)=Q_t-Q^\star$. Its recursion is autonomous and
can be written as
$
\Delta_{t+1}
=
\Delta_t
+
\tau_t\!\left[
H_A(\Delta_t,Y_{t+1})
+
\frac{\alpha_t}{\tau_t}H_V(\Delta_t,Y_{t+1})
\right].
$
We view the advantage update as the dominant contractive Markovian SA and the
value update as a state-dependent perturbation of size
$\alpha_t/\tau_t$. The stationary mean operator of the fast recursion is contractive
with gap $\kappa_{\rm fast}=d_{\min}(1-\gamma)$, so we construct a fixed
generalized Moreau-envelope Lyapunov function from this fast contraction. 
We then apply our extension of the Markovian finite-time analysis of
\citet{chen2024lyapunov} (Section~\ref{app:perturbed_markovian_sa}). The perturbation contributes an additional mean-drift
term, which is absorbed into the fast negative drift when
$\alpha_t/\tau_t$ is sufficiently small, and an additional Markovian
correlation term of the same mixing-window order as the original sampling
noise. This yields the same scalar Lyapunov recursion as in the
single-timescale analysis, up to modified constants. Unrolling the recursion
gives the desired rate.

\textit{Implications of Theorem \ref{thm:two_timescale_tracking}.} Theorem~\ref{thm:two_timescale_tracking} establishes $\widetilde \cO(1/t^p)$ mean square error rates for \algo for all $\beta>0$ removing the restriction of $\kappa_\beta$ in the single-timescale setting. We evaluate the performance of two-timescale \algo and compare it with the baselines of Asynchronous Q-learning and Dueling Q-learning in Figure~\ref{fig:tabular-two-timescale}, Appendix~\ref{app:experiments-mdp}.

\section{Experiments}
\label{sec:experiments}

\subsection{Discrete MDP}

We first evaluate \textsc{V-Share} on tabular MDPs.
Inspired by~\cite{tang2023va}, we consider a randomly
generated finite MDP with 10 states and 5 actions. For each state-action pair $(s,a)$, the transition
distribution $P(\cdot \mid s,a)$ is sampled independently from a symmetric
Dirichlet distribution,
$
P(\cdot \mid s,a)
\sim
\operatorname{Dirichlet}(\alpha,\ldots,\alpha),
\alpha=0.5$.
The rewards are sampled independently and uniformly from $[0,1]$, and
the discount factor is set to $\gamma=0.95$. 
We compare \textsc{V-Share} and its variants against tabular Dueling
Q-learning~\citep{daley2025analysis} and asynchronous Q-learning.
Samples are generated by a fixed behavior policy
$
\pi_b
=
(1-\epsilon)\pi_{\mathrm{det}}
+
\epsilon u$
where $\pi_{\mathrm{det}}$ is a randomly sampled deterministic policy,
$u$ is the uniform policy over actions, and $\epsilon=0.8$.
The MDP and behavior policy are fixed across all repetitions, while the
sample trajectory is varied across $10$ independent seeds.
We observe that under constant stepsizes, increasing $\beta$ substantially accelerates the
initial decay but eventually leads to a larger error floor (Figure~\ref{fig:tabular-constant}), illustrating the
predicted contraction-variance tradeoff. Decreasing-$\beta$ schedules combine these effects, matching the rapid initial
progress of strong residual preconditioning while attaining a smaller terminal error (Figure~\ref{fig:tabular-beta-schedules}). In all the plots, across all settings,
\algo{} converges faster than asynchronous Q-learning and Dueling Q-learning.

\begin{figure*}[t]
    \centering
    \begin{subfigure}[t]{0.48\textwidth}
        \centering
        \includegraphics[width=\linewidth]{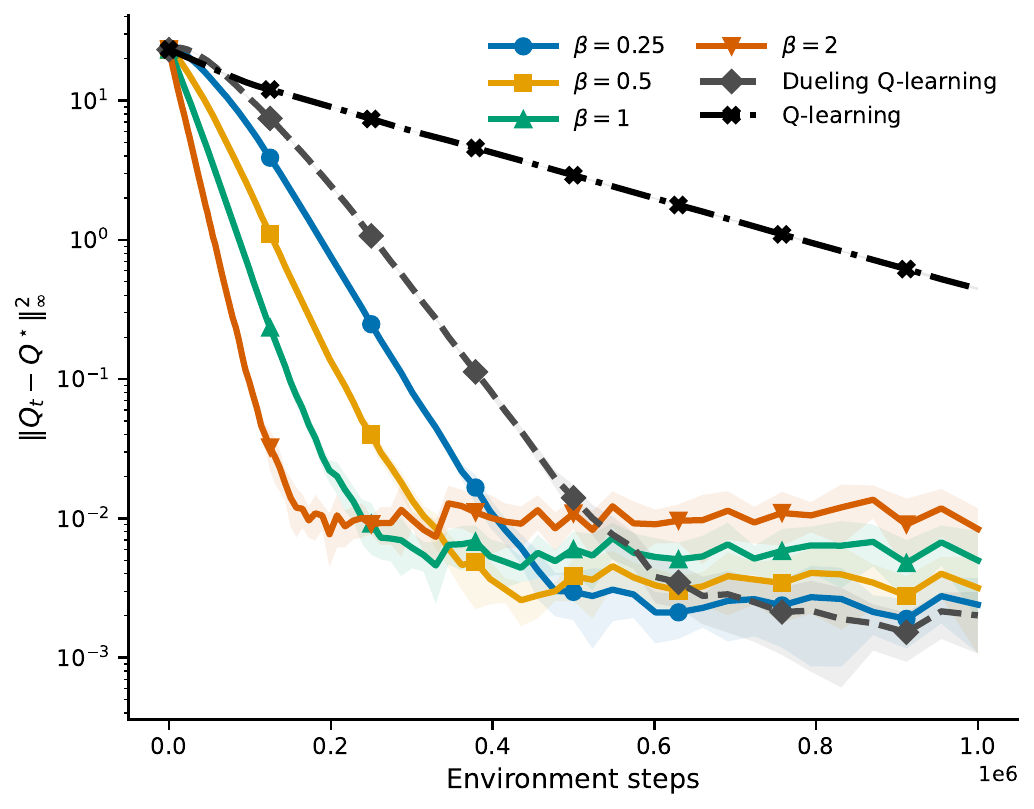}
        \caption{Constant stepsize.}
        \label{fig:tabular-constant}
    \end{subfigure}
    \hfill
    \begin{subfigure}[t]{0.48\textwidth}
        \centering
        \includegraphics[width=\linewidth]{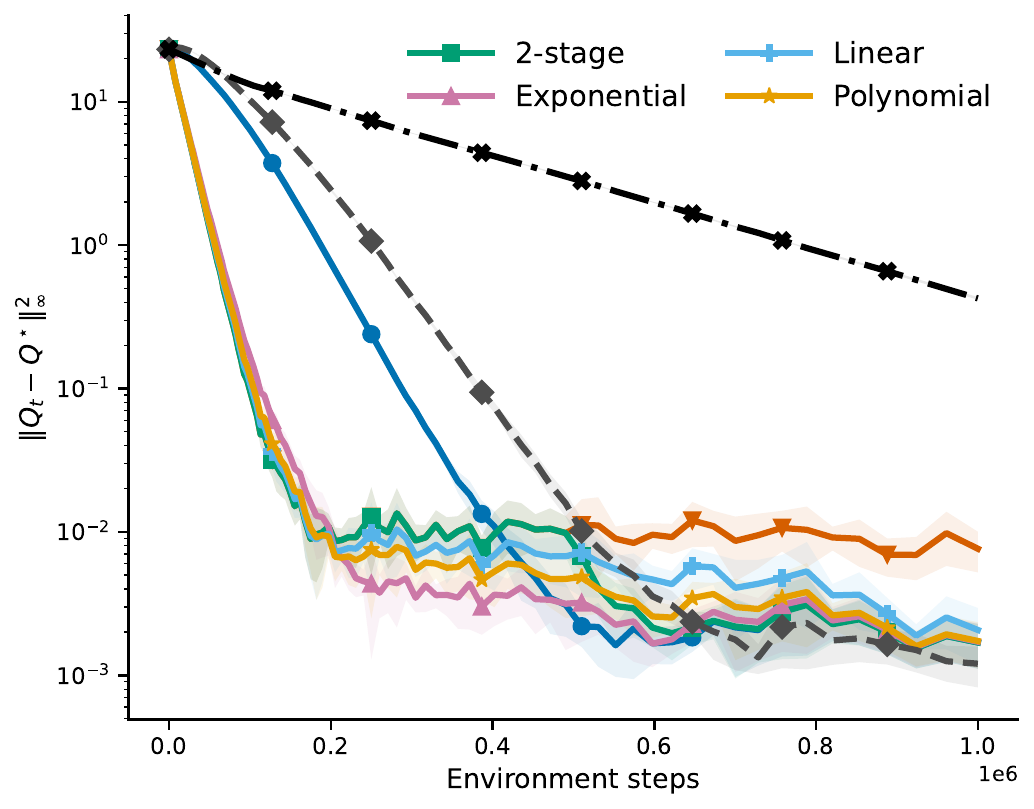}
        \caption{Decreasing residual-preconditioning schedules.}
        \label{fig:tabular-beta-schedules}
    \end{subfigure}
    \caption{
        \textbf{Tabular experiments on randomly generated MDPs.}
        (a): performance of single-timescale \algo (Algorithm \ref{alg:rpq_single}) for different residual
        weights $\beta$ under constant stepsizes, compared
        with asynchronous Q-learning and Dueling Q-learning.
        (b): comparison of decreasing residual-preconditioning
        schedules that interpolate from $\beta_{\mathrm{high}}$ to
        $\beta_{\mathrm{low}}$.
        Curves show the mean squared sup-norm error
        $\mathbb{E}\|Q_t-Q^\star\|_\infty^2$ over independent trajectories;
        shaded regions denote pointwise 95\% confidence intervals. The y-axis is in the logarithmic scale in all plots.
    }
    \label{fig:tabular-main}
\end{figure*}

\subsection{Deep RL}
\begin{table}[t]
    \centering
    \small
    \setlength{\tabcolsep}{3.5pt}
    \renewcommand{\arraystretch}{1.15}

    \resizebox{\textwidth}{!}{%
    \begin{tabular}{lccccccccc}
        \toprule
        \textbf{Environment}
        & \textbf{DQN}
        & \textbf{ME Layer}
        & \textbf{Soft RDQ}
        & \multicolumn{5}{c}{\textbf{\algo LR Multiplier}}
        & \textbf{\algo $\beta$-Scheduling} \\
        \cmidrule(lr){5-9}
        & & & &
        \textbf{$0.25\times$}
        & \textbf{$0.5\times$}
        & \textbf{$1\times$}
        & \textbf{$2\times$}
        & \textbf{$4\times$}
        & \\
        \midrule

        Acrobot-v1
        & $-443.12 \pm 101.39$
        & $-140.40 \pm 113.25$
        & $-85.44 \pm 9.30$
        & $\mathbf{-82.44 \pm 11.32}$
        & $-103.80 \pm 39.53$
        & $-87.24 \pm 4.65$
        & $-169.08 \pm 185.00$
        & $-167.24 \pm 186.44$
        & $-108.72 \pm 48.35$ \\

        CartPole-v1
        & $375.08 \pm 143.63$
        & $342.40 \pm 195.66$
        & $423.24 \pm 171.64$
        & $\mathbf{500.00 \pm 0.00}$
        & $\mathbf{500.00 \pm 0.00}$
        & $\mathbf{500.00 \pm 0.00}$
        & $464.00 \pm 77.52$
        & $\mathbf{500.00 \pm 0.00}$
        & $\mathbf{500.00 \pm 0.00}$ \\

        LunarLander-v2
        & $218.37 \pm 36.83$
        & $193.98 \pm 101.57$
        & $244.94 \pm 25.36$
        & $\mathbf{248.79 \pm 11.53}$
        & $227.80 \pm 51.22$
        & $202.93 \pm 4.76$
        & $105.49 \pm 98.36$
        & $-24.65 \pm 76.29$
        & $168.55 \pm 90.31$ \\

        MountainCar-v0
        & $-166.20 \pm 46.36$
        & $-138.04 \pm 16.81$
        & $\mathbf{-106.24 \pm 5.87}$
        & $-183.36 \pm 37.21$
        & $-136.92 \pm 9.06$
        & $-131.48 \pm 22.14$
        & $-133.60 \pm 17.56$
        & $-141.40 \pm 35.89$
        & $-129.28 \pm 4.98$ \\

        \midrule

        Asterix-v1
        & $42.33 \pm 14.28$
        & $41.90 \pm 12.77$
        & $\mathbf{55.97 \pm 9.13}$
        & $49.57 \pm 10.05$
        & $41.47 \pm 15.00$
        & $47.97 \pm 20.52$
        & $35.70 \pm 7.52$
        & $27.03 \pm 4.82$
        & $41.53 \pm 11.08$ \\

        Breakout-v1
        & $31.33 \pm 13.69$
        & $39.83 \pm 15.19$
        & $37.87 \pm 16.67$
        & $30.17 \pm 7.69$
        & $34.07 \pm 12.29$
        & $\mathbf{53.67 \pm 20.50}$
        & $40.63 \pm 14.13$
        & $33.93 \pm 12.58$
        & $40.20 \pm 15.59$ \\

        Freeway-v1
        & $61.77 \pm 1.50$
        & $61.40 \pm 1.48$
        & $\mathbf{63.10 \pm 1.22}$
        & $61.77 \pm 2.77$
        & $61.67 \pm 1.82$
        & $59.90 \pm 2.27$
        & $61.30 \pm 3.14$
        & $55.37 \pm 7.19$
        & $60.17 \pm 2.06$ \\

        Seaquest-v1
        & $74.73 \pm 22.11$
        & $119.00 \pm 32.18$
        & $122.40 \pm 46.38$
        & $\mathbf{132.07 \pm 53.50}$
        & $91.37 \pm 30.00$
        & $103.03 \pm 39.83$
        & $80.03 \pm 28.40$
        & $61.43 \pm 10.66$
        & $114.70 \pm 32.40$ \\

        SpaceInvaders-v1
        & $196.57 \pm 43.51$
        & $173.83 \pm 68.47$
        & $254.43 \pm 106.39$
        & $222.67 \pm 41.82$
        & $\mathbf{282.20 \pm 59.04}$
        & $192.70 \pm 91.94$
        & $136.67 \pm 66.44$
        & $134.57 \pm 67.59$
        & $134.33 \pm 68.78$ \\

        \bottomrule
    \end{tabular}%
    }
    \caption{\textbf{Deep RL results on Classic Control and MinAtar Environments.} We compare different configuration of \algo against DQN \citep{mnih2015human}, ME Layer \citep{nagarajan2026accelerating}, and Soft RDQ \citep{daley2025analysis} across four Classic Control and five MinAtar environments. Results report the mean $\pm$ standard deviation over five independent training seeds. The results show that while it is hard to directly apply the theoretical results to the deep RL setting, \algo tends to perform at par with Soft RDQ and significantly outperforms vanilla DQN.}
    \label{tab:adv_lr_ablation}
\end{table}

To verify whether our mechanism suggested by the tabular analysis remains effective in a more complex setting, we evaluate our method on both Classic Control and MinAtar environments. 
Specifically, we study three settings of \algo: single-timescale learning, two-timescale learning, and $\beta$ scheduling. 
In the \emph{single-timescale} setting, the value and advantage networks use the same learning rate. In the \emph{two-timescale} setting, we vary the advantage learning rate relative to the value learning rate, motivated by our analysis showing that timescale separation can remove the large-$\beta$ requirement. In the $\beta$ scheduling setting, the learning rates are kept equal while $\beta$ is linearly decayed from $100$ to $1$, motivated by the contraction-variance tradeoff.

Table~\ref{tab:adv_lr_ablation} summarizes our deep RL results across both MinAtar and Classic Control environments. Overall, \algo is competitive with the considered baselines, but its performance is sensitive to the relative learning rates of the value and advantage components. No single advantage learning-rate multiplier performs best across all tasks: smaller multipliers are preferred in several environments, while intermediate settings perform better in others. In particular, the best-performing \algo configuration achieves the highest mean return on six of the nine environments. At the same time, larger advantage learning-rate multipliers such as $2\times$ and $4\times$ can substantially degrade performance on some tasks. These results indicate that timescale separation can meaningfully affect deep RL, but the optimal relative update speed of the value and advantage components is environment dependent.

\section{Discussion and Future Work}

\textbf{Key take-aways.} We introduce \algo, an advantage-based RL algorithm, and provide the first asymptotic and non-asymptotic convergence guarantees for an advantage-based method under Markovian sampling. Our results suggest several practical principles for advantage-based value learning. \textbf{First}, they support separating the shared state-value component from action-specific advantages. \textbf{Second}, the residual weight $\beta$ controls how aggressively the algorithm corrects infrequently sampled actions, motivating stronger advantage updates early and weaker ones later. \textbf{Third}, when low-advantage actions are also sampled infrequently, applying a higher effective learning rate to them can improve learning. \textbf{Finally}, the two-timescale analysis supports using a larger effective learning rate for advantages than for state values in tabular MDPs, while our deep RL experiments suggest that reducing the advantage learning rate can improve stability by limiting noise.

\textbf{Future work.} Several directions remain open. First, while there is supporting intuition, our analysis does not yet explain the empirical acceleration of dueling Q-learning over standard Q-learning. A related question is whether value sharing can yield sharper problem-dependent or sample-complexity guarantees. Second, extending the analysis of \algo{} to function approximation is an important direction. Finally, our analysis assumes a fixed behavior policy; extending it to iterate-dependent sampling policies, such as $\epsilon$-greedy exploration, is another natural direction.

\subsection*{AI use statement}

In this work, we used generative AI tools solely for proofreading the manuscript and for suggestions to improve the writing style and clarity.

\subsubsection*{Acknowledgements}
This work was carried out as part of Aalto Science Institute International Summer Research Program 2026, project $\#2103$.
It was supported by the Research Council of Finland Flagship programme: Finnish Center for Artificial Intelligence FCAI,
and by the Finnish Ministry of Education and Culture through the Intelligent Work Machines Doctoral Education Pilot Program (IWM VN/3137/2024-OKM-4). 
We also acknowledge the computational resources provided by the Aalto Science-IT project.

\bibliographystyle{iclr-2027-style-files/iclr2027/iclr2027_conference}
\bibliography{cite}
\clearpage
\appendix
\part{Appendix}
\parttoc\newpage
\section*{Overview of the Appendix}
\begin{itemize}[leftmargin=*]
    \item \textbf{Appendix \ref{app:background-derivations}: Background and operator preliminaries.}
    Reviews asynchronous Q-learning and the value-advantage decomposition,
    derives the induced Q-dynamics of value sharing, and contrasts diagonal
    Q-learning updates with the cross-action preconditioning structure that
    motivates \algo{}.

    \item \textbf{Appendix \ref{app:single-rpq}: Single-timescale \algo{}: detailed analysis.}
    Develops the single-timescale stochastic-approximation representation,
    proves fixed-point preservation and the contraction regime, establishes
    almost-sure convergence under Markovian sampling, and derives the
    finite-time bounds for constant and diminishing stepsizes. It also studies
    the dependence on $\beta$ and compares the resulting guarantees with
    asynchronous Q-learning.

    \item \textbf{Appendix \ref{app:beta_schedules}: Decreasing residual-preconditioning schedules.}
    Analyzes time-varying choices of $\beta$. We first prove a sharp two-stage
    result showing how strong early contraction can be combined with a
    low-$\beta$ terminal noise floor, and then derive conservative guarantees
    for general decreasing schedules and their linear, polynomial, and
    exponential specializations.

    \item \textbf{Appendix \ref{app:two_timescale_rpq}: Two-timescale \algo{}: detailed proofs.}
    Develops the two-timescale stochastic-approximation formulation and the
    frozen-fast dynamics, proves almost-sure tracking for every fixed
    $\beta>0$, and derives finite-time tracking guarantees. The analysis also
    establishes auxiliary results for Markovian stochastic approximation with
    vanishing drift perturbations.
    \item \textbf{Appendix \ref{app:technical-lemma}: Useful Technical Lemmas.} Collects technical lemmas about the Moreau envelope construction used in the finite-sample analysis
    \item \textbf{Appendix \ref{app:experimental-details}: Experimental details.}
    Provides additional discrete-MDP experiments, implementation details,
    hyperparameter settings, and the deep-RL experimental setup used to study
    single-timescale, decreasing-$\beta$, and two-timescale variants of
    \algo{}.
\end{itemize}

\section{Background Derivations}
\label{app:background-derivations}

This appendix derives the sample recursions and stationary mean fields used in
Section~\ref{sec:background}, and recalls the standard sup-norm contraction
argument for asynchronous Q-learning.

\subsection{Sampling model and Bellman residual}

Let $(\mathcal S,\mathcal A,P,r,\gamma)$ be a finite discounted MDP with
$\gamma\in[0,1)$ and $|r(s,a)|\le 1$. The Bellman optimality operator is
$$
(\mathcal TQ)(s,a)
=
r(s,a)
+
\gamma\sum_{s'}P(s'|s,a)\max_{a'}Q(s',a').
$$
It is a $\gamma$-contraction in $\|\cdot\|_\infty$ and therefore has a unique
fixed point $Q^\star$.

Let $\pi_b(\cdot|s)$ be the behavior policy from
Assumption~\ref{ass:behavior}. The induced state chain is irreducible and
aperiodic with invariant distribution $d$, and
$$
Y_{t+1}:=(s_t,a_t,s_{t+1})
$$
is a finite-state irreducible and aperiodic Markov chain with invariant distribution
$$
d_Y(s,a,s')
=
d(s)\pi_b(a|s)P(s'|s,a).
$$
For $y=(s,a,s')$, define
$$
\delta(Q,y)
:=
r(s,a)
+
\gamma\max_{a'}Q(s',a')
-
Q(s,a).
$$
Its stationary conditional average over the successor state is
$$
\sum_{s'}P(s'|s,a)\delta(Q,(s,a,s'))
=
(\mathcal TQ-Q)(s,a).
$$

\subsection{Asynchronous Q-learning}
\label{app:qlearning-background}

Asynchronous Q-learning updates only the sampled state-action coordinate:
$$
Q_{t+1}(s,a)
=
Q_t(s,a)
+
\alpha_t
\mathbf 1\{s_t=s,a_t=a\}
\delta(Q_t,Y_{t+1}).
$$
Define the sample field
$$
H_{\rm Q}(Q,y)(s_0,a_0)
:=
\mathbf 1\{s_0=s,a_0=a\}\delta(Q,y),
\qquad
y=(s,a,s').
$$
Then
$$
Q_{t+1}
=
Q_t+\alpha_tH_{\rm Q}(Q_t,Y_{t+1}).
$$
Its stationary mean operator is
\begin{align*}
\bar H_{\rm Q}(Q)(s,a)
&:=
\mathbb E_{Y\sim d_Y}[H_{\rm Q}(Q,Y)(s,a)]
\\
&=
d(s)\pi_b(a|s)
\sum_{s'}P(s'|s,a)\delta(Q,(s,a,s'))
\\
&=
d(s)\pi_b(a|s)(TQ-Q)(s,a).
\end{align*}
Thus, with
$$
D_\mu
:=
\operatorname{diag}\{d(s)\pi_b(a|s):(s,a)\in\mathcal S\times\mathcal A\},
$$
the stationary mean operator and associated limiting ODE are
$$
\bar H_{\rm Q}(Q)=D_\mu (TQ-Q),
\qquad
\dot Q=D_\mu (TQ-Q).
$$

\subsubsection{Sup-norm contraction}

Define
$$
F_{\rm Q}(Q)
:=
Q+D_\mu (TQ-Q).
$$
For $i=(s,a)$, write $\mu_i:=d(s)\pi_b(a|s)$. Then
\begin{align*}
[F_{\rm Q}(Q)-F_{\rm Q}(Q')]_i
&=
(1-\mu_i)(Q_i-Q_i')
+
\mu_i[(\mathcal TQ)_i-(\mathcal TQ')_i],
\end{align*}
and hence
$$
|[F_{\rm Q}(Q)-F_{\rm Q}(Q')]_i|
\le
[1-\mu_i(1-\gamma)]
\|Q-Q'\|_\infty.
$$
Therefore
\begin{equation}
\label{eq:async_q_learning_contraction}
    \|F_{\rm Q}(Q)-F_{\rm Q}(Q')\|_\infty
\le
[1-\mu_{\min}(1-\gamma)]
\|Q-Q'\|_\infty,
\end{equation}
where
$$
\mu_{\min}:=\min_{s,a}d(s)\pi_b(a|s).
$$
Thus the contraction gap is $\mu_{\min}(1-\gamma)$. The argument works
because $D_\mu$ is diagonal: every  coordinate is a convex
combination of its previous value and its own Bellman update.

\subsection{Value-advantage update and induced Q-dynamics}
\label{app:hard-rdq-background}

Consider the decomposition
$$
Q_t(s,a)=V_t(s)+A_t(s,a)
$$
and the hard RDQ update \citep{daley2025analysis}
$$
V_{t+1}(s_t)
=
V_t(s_t)+\alpha_t\delta(Q_t,Y_{t+1}), \quad
A_{t+1}(s_t,a)
=
A_t(s_t,a)
+
\alpha_t\mathbf 1\{a_t=a\}
\delta(Q_t,Y_{t+1}),
$$
with all coordinates at states $s\neq s_t$ unchanged. The decomposition is
not identifiable: for any state-dependent scalar $c(s)$,
$$
V_t(s)\mapsto V_t(s)+c(s),
\qquad
A_t(s,a)\mapsto A_t(s,a)-c(s)
$$
leaves $Q_t$ unchanged. We therefore analyze the  iterate $Q_t$.

For every  coordinate $(s,a)$,
$$
Q_{t+1}(s,a)
=
Q_t(s,a)
+
\alpha_t\mathbf 1\{s_t=s\}
[1+\mathbf 1\{a_t=a\}]
\delta(Q_t,Y_{t+1}).
$$
The first term is the value increment shared across all actions at the sampled
state; the second is the sampled-action advantage increment.

For each state $s$, let
$$
b_s:=(\pi_b(a|s))_{a\in\mathcal A},
\qquad
B_s:=\textbf{e} b_s^\top,
\qquad
B:=\operatorname{blkdiag}(B_s:s\in\mathcal S),
$$
and define the lifted state-visitation matrix
$$
D:=\operatorname{diag}\{d(s):(s,a)\in\mathcal S\times\mathcal A\}.
$$
The stationary mean operator of the shared value update is
$$
DB(TQ-Q),
$$
because for every $(s,a)$,
\begin{align*}
\mathbb E_{Y\sim d_Y}
\left[
\mathbf 1\{S=s\}\delta(Q,Y)
\right]
&=
d(s)\sum_c \pi_b(c|s)(TQ-Q)(s,c).
\end{align*}
The sampled-action contribution is
$$
D_\mu (TQ-Q).
$$
Hence the stationary mean operator of the hard value-advantage update is
$$
\bar H_{\rm AV}(Q)
=
(DB+D_\mu)(TQ-Q),
$$
with associated limiting ODE
$$
\dot Q
=
(DB+D_\mu)(TQ-Q).
$$

For a 2-state, 2-action MDP, the preconditioners look like the following:
Let,
\[
        \pi_b(i, j):=\pi_b(a_j\mid s_i), \quad
        d_i=d(s_i).
\]
    where $d(\cdot)$ is the stationary distirbution induced by the behavior policy $\pi_b(\cdot | \cdot)$

    \[
        D_\mu
        =
        \begin{bmatrix}
        d_1\pi_b(1,1) & 0 & 0 & 0\\
        0 & d_1\pi_b(1,2) & 0 & 0\\
        0 & 0 & d_2\pi_b(2,1) & 0\\
        0 & 0 & 0 & d_2\pi_b(2,2)
        \end{bmatrix}.
    \]
    
    \[
        DB
        =
        D_\mu
        +
        \begin{bmatrix}
        d_1\pi_b(1,1) & d_1\pi_b(1,2) & 0 & 0\\
        d_1\pi_b(1,1) & d_1\pi_b(1,2) & 0 & 0\\
        0 & 0 & d_2\pi_b(2,1) & d_2\pi_b(2,2)\\
        0 & 0 & d_2\pi_b(2,1) & d_2\pi_b(2,2)
        \end{bmatrix}.
    \]

\subsection{Diagonal versus cross-action preconditioning}
\label{app:preconditioning-comparison}

The preceding derivations give
$$
\bar H_{\rm Q}(Q)=D_\mu (TQ-Q)
$$
for asynchronous Q-learning and
$$
\bar H_{\rm AV}(Q)=(DB+D_\mu)(TQ-Q)
$$
for the hard RDQ update.

The Q-learning preconditioner $D_\mu$ is diagonal. In contrast, the shared
value component satisfies
$$
[DB]_{(s,a),(s,c)}
=
d(s)\pi_b(c|s)
$$
for every pair of actions $a,c$ at the same state. Thus a Bellman residual
associated with action $c$ contributes to the mean update of every 
action at state $s$.

For the hard RDQ preconditioner,
$$
[DB+D_\mu]_{(s,a),(s,c)}
=
d(s)\pi_b(c|s)
+
\mathbf 1\{a=c\}d(s)\pi_b(a|s).
$$
In particular,
$$
[DB+D_\mu]_{(s,a),(s,a)}
=
2d(s)\pi_b(a|s),
$$
while for $c\neq a$,
$$
[DB+D_\mu]_{(s,a),(s,c)}
=
d(s)\pi_b(c|s)>0.
$$
This cross-action coupling is precisely what prevents the coordinatewise
convex-combination argument used for Q-learning (\eqref{eq:async_q_learning_contraction}) from carrying over.

\section{Single-Timescale \algo{}: Detailed Analysis}
\label{app:single-rpq}

This appendix proves the single-timescale results stated in
Sections~\ref{sec:single_rpq} and~\ref{sec:finite_time}. Throughout,
sampling is generated by the Markov chain $\{Y_t\}$ above. Recall
$B_s=\textbf{e}b_s^\top$,
$B=\operatorname{blkdiag}(B_s:s\in\mathcal S)$, and
$D_{(s,a),(s,a)}=d(s)$.

\subsection{\algo{} recursion and stationary mean operator}
\label{app:rpq_derivations}

The single-timescale \algo{} updates are
$$
V_{t+1}(s_t)
=
V_t(s_t)+\alpha_t\delta(Q_t,Y_{t+1}),
\quad 
A_{t+1}(s_t,a)
=
A_t(s_t,a)
+
\alpha_t
\frac{\beta\mathbf 1\{a_t=a\}}{\pi_b(a|s_t)}
\delta(Q_t,Y_{t+1}),
$$
where $Q_t=LV_t+A_t$. Therefore the  action-value iterate obeys
$$
Q_{t+1}(s,a)
=
Q_t(s,a)
+
\alpha_t\mathbf 1\{s_t=s\}
\left[
1+
\beta\frac{\mathbf 1\{a_t=a\}}{\pi_b(a|s)}
\right]
\delta(Q_t,Y_{t+1}).
$$

For $y=(s,a,s')$, define the sample field
$$
H_\beta(Q,y)(s_0,a_0)
:=
\mathbf 1\{s_0=s\}
\left[
1+
\beta\frac{\mathbf 1\{a_0=a\}}{\pi_b(a|s)}
\right]
\delta(Q,y).
$$
Then the recursion is
$$
Q_{t+1}
=
Q_t+\alpha_tH_\beta(Q_t,Y_{t+1}).
$$
Its stationary mean operator is
$$
h_\beta(Q)
:=
\mathbb E_{Y\sim d_Y}[H_\beta(Q,Y)].
$$
For a fixed coordinate $(s,a)$, the shared-value contribution is
\begin{align*}
&\mathbb E_{Y\sim d_Y}
\left[
\mathbf 1\{S=s\}\delta(Q,Y)
\right]=
d(s)\sum_{u}b(u|s)(TQ-Q)(s,u)
=
[DB(TQ-Q)](s,a),
\end{align*}
while importance weighting gives
\begin{align*}
&\mathbb E_{Y\sim d_Y}
\left[
\mathbf 1\{S=s\}
\frac{\mathbf 1\{A=a\}}{\pi_b(a|s)}
\delta(Q,Y)
\right]=
d(s)(TQ-Q)(s,a).
\end{align*}
Consequently,
$$
h_\beta(Q)
=
D(B+\beta I)(TQ-Q).
$$
The associated limiting ODE is
$$
\dot Q
=
D(B+\beta I)(\mathcal TQ-Q).
$$

For later use, write $
G_\beta:=D(B+\beta I).
$
For $i=(s,a)$,
\begin{equation}
\label{eq:app_G_rows}
(G_\beta)_{ii}
=
d(s)[\beta+\pi_b(a|s)],
\quad 
\sum_{j\neq i}(G_\beta)_{ij}
=
d(s)[1-\pi_b(a|s)],
\quad 
\sum_j(G_\beta)_{ij}
=
d(s)(\beta+1).
\end{equation}
Thus residual preconditioning leaves the cross-action sharing term $DB$
unchanged while adding a diagonal correction $\beta D$.




\subsection{Fixed-point preservation and contraction}
\label{app:rpq_contraction}

We first verify that residual preconditioning preserves the Bellman
fixed point.

\begin{lemma}[Fixed-point preservation]
\label{lem:app_rpq_fixed_point}
For every $\beta>0$, $Q^\star$ is the unique equilibrium of
$\dot Q=h_\beta(Q)$.
\end{lemma}

\begin{proof}
For every state $s$, the matrix
$\beta I+\textbf{e}b_s^\top$, where $I$ is the identity matrix in $\bbR^{|\cA|\times|\cA|}$, is invertible since the matrix determinant
lemma \citep{horn2012matrix} gives
$$
\det(\beta I+\textbf{e}b_s^\top)
=
\beta^{|\mathcal A|-1}(\beta+1)>0.
$$
Since $D$ is also invertible,
$D(B+\beta I)(\mathcal TQ-Q)=0$ iff $\mathcal TQ=Q$.
The Bellman optimality operator has the unique fixed point $Q^\star$.
\end{proof}

We next prove the contraction result used throughout the paper. Define
$$
F_{\eta,\beta}(Q)
:=
Q+\eta h_\beta(Q),
$$
and recall
\begin{equation}
L_\beta
:=
\max_{s,a}d(s)\bigl(\beta+\pi_b(a|s)\bigr),
\qquad
\kappa_\beta
:=
\min_{s,a}
d(s)
\left[
(1-\gamma)\beta-(1+\gamma)+2b(a|s)
\right].
\label{eq:app_rpq_constants}
\end{equation}

\begin{lemma}[Contraction of the \algo  fixed-point map]
\label{lem:app_rpq_contraction}
Suppose $\kappa_\beta>0$. For every
$0<\eta\le L_\beta^{-1}$,
$$
\|F_{\eta,\beta}(Q)-F_{\eta,\beta}(Q')\|_\infty
\le
(1-\eta\kappa_\beta)
\|Q-Q'\|_\infty.
$$
Consequently, it suffices that
$$
\beta>
\max\left\{
0,\,
\frac{1+\gamma-2b_{\min}}{1-\gamma}
\right\}.
$$
\end{lemma}

\begin{proof}
Let $z:=Q-Q'$, $M:=\|z\|_\infty$, and
$\Delta_T:=\mathcal TQ-\mathcal TQ'$.
The Bellman contraction gives
$\|\Delta_T\|_\infty\le\gamma M$.

For $i=(s,a)$,
\begin{align*}
[F_{\eta,\beta}(Q)-F_{\eta,\beta}(Q')]_i
=
\bigl(1-\eta(D(B+\beta I))_{ii}\bigr)z_i
-
\eta\sum_{j\neq i}(D(B+\beta I))_{ij}z_j
+
\eta\sum_j(D(B+\beta I))_{ij}(\Delta_T)_j.
\end{align*}
Since $\eta\le L_\beta^{-1}$,
$1-\eta(D(B+\beta I))_{ii}\ge0$.
Using \eqref{eq:app_G_rows},
\begin{align*}
|[F_{\eta,\beta}(Q)-F_{\eta,\beta}(Q')]_i|
&\le
\Big[
1-\eta d(s)(\beta+\pi_b(a|s))
+\eta d(s)(1-\pi_b(a|s))
+\eta\gamma d(s)(\beta+1)
\Big]M
\\
&=
\Big[
1-\eta d(s)
\big\{(1-\gamma)\beta-(1+\gamma)+2b(a|s)\big\}
\Big]M
\\
&\le
(1-\eta\kappa_\beta)M.
\end{align*}
Taking the maximum over $i$ proves the contraction claim. Positivity of
$\kappa_\beta$ follows whenever
$(1-\gamma)\beta-(1+\gamma)+2b(a|s)>0$ for every $(s,a)$.
\end{proof}

\subsection{Almost-sure convergence under Markovian sampling}
\label{app:rpq_asymptotic}

We prove almost-sure convergence of \algo{} directly under the Markovian
trajectory generated by the behavior policy, using the stochastic-approximation (SA)
framework of \citet{liu2025ode}, which is an extension of the Borkar-Meyn theorem \citep{doi:10.1137/S0363012997331639} to the Markovian noise setting. Their recursion has the form
$Q_{t+1}=Q_t+\alpha_t H(Q_t,Y_{t+1})$. We verify 
\citet[Assumptions 1-6]{liu2025ode} in order to apply \citet[Theorem 7, Corollary 8]{liu2025ode}.

\textbf{Markovian SA representation.}
Let $Y_{t+1}:=(s_t,a_t,s_{t+1})$. For $y=(s,a,s')$, define
$\delta(Q,y):=r(s,a)+\gamma\max_{a'}Q(s',a')-Q(s,a)$ and
$$
H_\beta(Q,y)(s_0,a_0)
:=
\mathbf 1\{s_0=s\}
\left(
1+\beta\frac{\mathbf 1\{a=a_0\}}{\pi_b(a|s)}
\right)
\delta(Q,y).
$$
Then the  \algo{} recursion is
\begin{equation}
Q_{t+1}=Q_t+\alpha_tH_\beta(Q_t,Y_{t+1}).
\label{eq:app_markov_sa}
\end{equation}
Under stationarity,
$\mathbb{P}(Y=(s,a,s'))=d(s)\pi_b(a|s)P(s'|s,a)$, and therefore
$$
h_\beta(Q):=\mathbb E_{Y \sim d_Y}[H_\beta(Q,Y)]
=
D(B+\beta I)(\mathcal TQ-Q).
$$

\textbf{Assumption 1: unique invariant distribution.}
Under Assumption~\ref{ass:behavior}, the state chain induced by $\pi_b$ is
finite, irreducible, and aperiodic. Since $\pi_b$ has full support, the
transition chain $Y_t=(s_{t-1},a_{t-1},s_t)$ is finite-state and ergodic on
its communicating class, with unique invariant distribution
$d_Y(s,u,s')=d(s)\pi_b(u|s)P(s'|s,u)$. Thus Assumption~1 holds.

\textbf{Assumption 2: learning-rate regularity.}
We take
$$
\alpha_t=\frac{\alpha_0}{(t+h)^p},
\qquad
\frac12<p\le1.
$$
Then $\alpha_t>0$, $\alpha_t\downarrow0$, and
$\sum_t\alpha_t=\infty$. 

\citet[Remark 1]{liu2025ode} implies that Assumption~2 holds.

\textbf{Assumption 3: scaled-update limit.}
For $c\ge1$, define $H_{\beta,c}(Q,y):=H_\beta(cQ,y)/c$ and $y=(s,a,s')$. Then
$$
H_{\beta,c}(Q,y)(s_0,a_0)
=
\mathbf 1\{s_0=s\}
\left(
1+\beta\frac{\mathbf 1\{a_0=a\}}{\pi_b(a|s)}
\right)
\left[
\frac{r(s,a)}{c}
+\gamma\max_{a'}Q(s',a')
-Q(s,a)
\right].
$$
Define
$$
H_{\beta,\infty}(Q,y)(s_0,a_0)
:=
\mathbf 1\{s_0=s\}
\left(
1+\beta\frac{\mathbf 1\{a_0=a\}}{\pi_b(a|s)}
\right)
\left[
\gamma\max_{a'}Q(s',a')
-Q(s,a)
\right].
$$
Hence
\begin{equation}
\label{eq:ass3-scaling}
H_{\beta,c}(Q,y)-H_{\beta,\infty}(Q,y)
=
\frac1cR_\beta(y),
\end{equation}
where
$$
R_\beta(y)(s_0,a_0)
:=
\mathbf 1\{s_0=s\}
\left(
1+\beta\frac{\mathbf 1\{a_0=a\}}{\pi_b(a|s)}
\right)r(s,a).
$$
Thus Assumption~3 holds with $\kappa(c)=1/c$ and
$b(Q,y)=R_\beta(y)$. Since $b(Q,y)$ is independent of $Q$, its Lipschitz
constant can be taken as $L_b(y)\equiv0$.

\textbf{Assumption 4: Lipschitz continuity and integrability.}
For any $Q,Q'$ and $y=(s,a,s')$,
$$
|\delta(Q,y)-\delta(Q',y)|
\le
(1+\gamma)\|Q-Q'\|_\infty.
$$
Therefore, with $b_{\min}:=\min_{s,a}\pi_b(a|s)>0$,
$$
\|H_\beta(Q,y)-H_\beta(Q',y)\|_\infty
\le
\left(1+\frac{\beta}{b_{\min}}\right)(1+\gamma)
\|Q-Q'\|_\infty.
$$
The same bound holds for $H_{\beta,\infty}$. Hence Assumption~4 holds with
the constant Lipschitz function
$L(y)\equiv(1+\beta/b_{\min})(1+\gamma)$, whose invariant expectation is
finite. Since the Markov state space is finite and rewards are bounded, the
remaining required integrability conditions also hold.

Averaging $H_{\beta,\infty}$ gives
\begin{equation}
h_{\beta,\infty}(Q)
=
D(B+\beta I)(\mathcal T_0Q-Q),
\label{eq:app_rpq_hinf}
\end{equation}
where
$(\mathcal T_0Q)(s,a)
:=
\gamma\sum_{s'}P(s'|s,a)\max_{a'}Q(s',a')$.

\textbf{Assumption 5: scaling limit and stability at infinity.}
Let $h_{\beta,c}(Q):=h_\beta(cQ)/c$. Taking expectations in \eqref{eq:ass3-scaling} gives
$$
h_{\beta,c}(Q)-h_{\beta,\infty}(Q)
=
\frac1c\mathbb E_{d_Y}[R_\beta(Y)].
$$
Hence, for every compact set $K$,
$$
\sup_{Q\in K}
\|h_{\beta,c}(Q)-h_{\beta,\infty}(Q)\|_\infty
\le
\frac{C_R}{c}
\to0,
$$
where $C_R:=\|\mathbb E_{d_Y}[R_\beta(Y)]\|_\infty<\infty$.

It remains to verify global asymptotic stability of the limiting ODE
$\dot Q=h_{\beta,\infty}(Q)$. Since $\mathcal T_0$ is a
$\gamma$-contraction with unique fixed point $0$, the argument of
Lemma~\ref{lem:app_rpq_contraction} gives, for every
$0<\eta\le L_\beta^{-1}$,
$$
\|Q+\eta h_{\beta,\infty}(Q)\|_\infty
\le
(1-\eta\kappa_\beta)\|Q\|_\infty.
$$

Let $Q(\cdot)$ solve $\dot Q=h_{\beta,\infty}(Q)$ and set
$W(Q):=\|Q\|_\infty$. Since $h_{\beta,\infty}$ is globally Lipschitz,
$$
Q(t+h)
=
Q(t)+h h_{\beta,\infty}(Q(t))+r_t(h),
\qquad
\frac{\|r_t(h)\|_\infty}{h}\to0
\quad(h\downarrow0).
$$
Using the $1$-Lipschitz property of the sup norm, for sufficiently small
$h>0$,
$$
W(Q(t+h))
\le
\|Q(t)+h h_{\beta,\infty}(Q(t))\|_\infty+\|r_t(h)\|_\infty
\le
(1-h\kappa_\beta)W(Q(t))+\|r_t(h)\|_\infty.
$$
Therefore
\begin{align}
D^+W(Q(t))
&:=
\limsup_{h\downarrow0}
\frac{W(Q(t+h))-W(Q(t))}{h}
\nonumber\\
&\le
-\kappa_\beta W(Q(t))
+
\limsup_{h\downarrow0}\frac{\|r_t(h)\|_\infty}{h}
=
-\kappa_\beta W(Q(t)).
\label{eq:app_markov_dini_inf}
\end{align}

Applying the comparison lemma
\citep[Lemma~3.4]{khalil2002nonlinear}
to the scalar ODE
$\dot u=-\kappa_\beta u$ with
$u(0)=W(x(0))$ yields
$$
W(x(t))
\le
e^{-\kappa_\beta t}W(x(0)).
$$
Since $W(x)=\|x\|_\infty$,
$$
\|Q(t)\|_\infty
\le
e^{-\kappa_\beta t}\|Q(0)\|_\infty,
$$
so the origin is globally exponentially stable. This verifies
Assumption~5.

\textbf{Assumption 6: Markovian averaging.}
Assumption~6 of \citet{liu2025ode} requires, for the relevant functions
$g$ and every initial state of the Markov chain,
$$
\alpha_n
\sum_{t=1}^n
\left(
g(Y_t)-\mathbb E_{d_Y}[g(Y)]
\right)
\to0
\qquad\text{a.s.}
$$
For $g=L_b$, this is immediate since $L_b\equiv0$, and for $g=L$ it is
immediate since $L$ is constant. For fixed $Q$, the function
$y\mapsto H_\beta(Q,y)$ is bounded on the finite state space. Since
$\{Y_t\}$ is irreducible and aperiodic, \citet[Remark 2]{liu2025ode} implies Assumption~6 holds.

\textbf{Global exponential stability of the limiting ODE.}
The mean drift
$h_\beta(Q)=D(B+\beta I)(\mathcal TQ-Q)$ is globally Lipschitz since
$$
\|h_\beta(Q)-h_\beta(Q')\|_\infty
\le
(1+\gamma)\|D(B+\beta I)\|_\infty
\|Q-Q'\|_\infty.
$$
Hence $\dot Q=h_\beta(Q)$ has a unique global solution.

Set $x(t):=Q(t)-Q^\star$ and $W(x):=\|x\|_\infty$. Since
$h_\beta(Q^\star)=0$, Lemma~\ref{lem:app_rpq_contraction} gives, for every
$0<\eta\le L_\beta^{-1}$,
$$
\|x+\eta h_\beta(Q^\star+x)\|_\infty
\le
(1-\eta\kappa_\beta)\|x\|_\infty.
$$
Along an ODE solution,
$$
x(t+h)
=
x(t)+h h_\beta(Q^\star+x(t))+r_t(h),
\qquad
\frac{\|r_t(h)\|_\infty}{h}\to0.
$$
Therefore, using again the $1$-Lipschitz property of the sup norm,
$$
W(x(t+h))
\le
(1-h\kappa_\beta)W(x(t))+\|r_t(h)\|_\infty
$$
for all sufficiently small $h>0$. Hence
\begin{align}
D^+W(x(t))
&:=
\limsup_{h\downarrow0}
\frac{W(x(t+h))-W(x(t))}{h}
\nonumber\\
&\le
-\kappa_\beta W(x(t)).
\label{eq:app_rpq_dini}
\end{align}

Applying the comparison lemma
\citep[Lemma~3.4]{khalil2002nonlinear}
to the scalar ODE
$\dot u=-\kappa_\beta u$ with
$u(0)=W(x(0))$ yields
$$
W(x(t))
\le
e^{-\kappa_\beta t}W(x(0)).
$$
Since $V(x)=\|x\|_\infty$,
\begin{equation}
\label{eq:app_rpq_ges}
\|Q(t)-Q^\star\|_\infty
\le
e^{-\kappa_\beta t}
\|Q(0)-Q^\star\|_\infty.
\end{equation}
Hence $Q^\star$ is globally exponentially stable for the limiting ODE.

\begin{theorem}[Almost-sure convergence of \algo{} under Markovian sampling]
\label{thm:app_rpq_asymptotic}
Suppose Assumption~\ref{ass:behavior} holds, rewards are bounded,
$\kappa_\beta>0$, and
$$
\alpha_t=\frac{\alpha_0}{(t+h)^p},
\qquad
\frac12<p\le1.
$$
Then
$$
\sup_t\|Q_t\|_\infty<\infty
\qquad\text{a.s.},
\qquad
Q_t\to Q^\star
\qquad\text{a.s.}
$$
\end{theorem}

\begin{proof}
The preceding arguments verify Assumptions~1-6 of
\citet{liu2025ode}. Their Theorem~7 therefore gives
$\sup_t\|Q_t\|_\infty<\infty$ almost surely.

Their Corollary~8 implies that the asymptotic limit set of
$\{Q_t\}$ is almost surely a bounded invariant set of the limiting ODE
$\dot Q=h_\beta(Q)$. By \eqref{eq:app_rpq_ges}, every trajectory of this
ODE converges exponentially to $Q^\star$, so its global attractor is the
singleton $\{Q^\star\}$. Hence the stochastic-approximation limit set is
contained in $\{Q^\star\}$, and therefore $Q_t\to Q^\star$ almost surely.
\end{proof}

\begin{remark}
The iterates generated by asynchronous Q-learning can be shown to be bounded a.s. by a simple induction argument \citep{gosavi2006boundedness}. However, the same induction based argument doesn't work for \algo. This is because the TD information pertaining to the state-action pair generated by the behvaiour policy is used to update the Q-values of multiple state-action pairs in \algo. Hence, we require the ODE scaling arguments for the  stability of iterates \citep{abounadi2002stochastic, liu2025ode}
\end{remark}
\subsection{Finite-time analysis under Markovian sampling}
\label{app:finite-time-rpq}

We now prove the finite-time guarantees for \algo{} under the Markovian
trajectory of Assumption~\ref{ass:behavior}. The argument follows the
contractive Markovian stochastic-approximation framework of
\citet{chen2024lyapunov}. We first cast \algo{} into their fixed-point form,
then construct the generalized Moreau-envelope Lyapunov function, and finally
specialize the resulting general recursion to different stepsize schedules.

\paragraph{Markovian contractive-SA representation.}
Recall the sample update map $H_\beta(Q,Y_{k+1})$ from
Appendix~\ref{app:rpq_asymptotic}, so that
$$
Q_{k+1}
=
Q_k+\alpha_kH_\beta(Q_k,Y_{k+1}),
$$
and
$\mathbb E_{d_Y}[H_\beta(Q,Y)]=h_\beta(Q)$.
For an auxiliary analysis parameter $\eta>0$, recall the \algo fixed point map,
$$
\mathsf F_{\eta,\beta}(Q,y)
:=
Q+\eta H_\beta(Q,y),
\qquad
F_{\eta,\beta}(Q)
:=
\mathbb E_{d_Y}[\mathsf F_{\eta,\beta}(Q,Y)]
=
Q+\eta h_\beta(Q).
$$
Writing
$$
\epsilon_k:=\frac{\alpha_k}{\eta},
$$
the \algo{} recursion becomes
\begin{equation}
Q_{k+1}
=
Q_k
+
\epsilon_k
\bigl[
\mathsf F_{\eta,\beta}(Q_k,Y_{k+1})-Q_k
\bigr].
\label{eq:app_markov_contract_sa}
\end{equation}
By Lemma~\ref{lem:app_rpq_contraction},
$F_{\eta,\beta}$ is a sup-norm contraction whenever
$\kappa_\beta>0$ and $0<\eta\le L_\beta^{-1}$, with contraction factor
$$
\rho_{\eta,\beta}:=1-\eta\kappa_\beta.
$$

We next verify the growth constants required by
\citet{chen2024lyapunov}. Let
$$
C_\beta:=1+\frac{\beta}{b_{\min}}.
$$
From the sample update,
$$
\|H_\beta(Q,y)-H_\beta(Q',y)\|_\infty
\le
C_\beta(1+\gamma)\|Q-Q'\|_\infty,
$$
and therefore
$$
\|\mathsf F_{\eta,\beta}(Q,y)
-\mathsf F_{\eta,\beta}(Q',y)\|_\infty
\le
A_1\|Q-Q'\|_\infty,
\qquad
A_1:=1+\eta C_\beta(1+\gamma).
$$
Moreover,
$$
\|\mathsf F_{\eta,\beta}(0,y)\|_\infty
\le
B_1,
\qquad
B_1:=\eta C_\beta .
$$
There is no additional martingale-difference term in
\eqref{eq:app_markov_contract_sa}, so in the notation of
\citet{chen2024lyapunov} we may take
$A_2=B_2=0$. Hence define
$$
A:=A_1+1
=
2+\eta C_\beta(1+\gamma),
\qquad
B:=B_1
=
\eta C_\beta ,
$$
and
\begin{equation}
\label{eq:app_c1_c2}
c_1
:=
\left(
\|Q_0-Q^\star\|_\infty
+\|Q_0\|_\infty
+\frac BA
\right)^2,
\qquad
c_2
:=
\left(
A\|Q^\star\|_\infty+B
\right)^2.
\end{equation}

\paragraph{Mixing time.}
Let $\mathcal M_Y$ denote the finite-state Markov chain $\{Y_k\}$ and define
its mixing time at precision $\delta>0$ by
$$
t_\delta
:=
\min\left\{
m\ge0:
\max_y
\|P^m(y,\cdot)-d_Y\|_{\mathrm{TV}}
\le\delta
\right\}.
$$
For the effective stepsizes $\epsilon_k$, write
$t_k:=t_{\epsilon_k}$ and, for $i\le j$,
$\epsilon_{i,j}:=\sum_{\ell=i}^j\epsilon_\ell$.
Under Assumption~\ref{ass:behavior}, $\mathcal M_Y$ is finite,
irreducible and aperiodic, and hence geometrically mixing. In particular,
$t_\delta=O(\log(1/\delta))$.

\paragraph{Generalized Moreau envelope.}
To exploit the sup-norm contraction of
$F_{\eta,\beta}$, define
$f(x):=\frac12\|x\|_\infty^2$ and
$g(x):=\frac12\|x\|_p^2$, with $p=2\log(|\cS||\cA|)$ and
$d:=|\mathcal S||\mathcal A|$. For $d\ge3$, the $\ell_p$ norm satisfies
$$
\ell_{cs}\|x\|_p
\le
\|x\|_\infty
\le
u_{cs}\|x\|_p,
\qquad
\ell_{cs}=\frac1{\sqrt e},
\qquad
u_{cs}=1,
$$
and $g$ is $L$-smooth with $L=p-1\le2\log(|\cS||\cA|)$.

For $\theta>0$, define the generalized Moreau envelope
$$
\mathcal M_\theta(x)
:=
\min_{u\in\mathbb R^d}
\left\{
\frac12\|u\|_\infty^2
+
\frac1{2\theta}\|x-u\|_p^2
\right\}.
$$

The properties of the generalized Moreau envelope are summarized in Lemma~\ref{lem:moreau_technical} and Corollary~\ref{cor:moreau_infty}.

Let $\rho:=\rho_{\eta,\beta}$ and choose
$$
\theta
:=
\left(\frac{1+\rho}{2\rho}\right)^2-1.
$$
The auxiliary constants of the Moreau analysis are
\begin{equation}
\phi_1
:=
\frac{1+\theta u_{cs}^2}
     {1+\theta\ell_{cs}^2},
\qquad
\phi_2
:=
1-\rho\sqrt{\phi_1},
\qquad
\phi_3
:=
\frac{114L(1+\theta u_{cs}^2)}
     {\theta\ell_{cs}^2}.
\label{eq:app_phi_constants}
\end{equation}
For the sup norm, Corollary~\ref{cor:moreau_infty} gives
\begin{equation}
\label{eq:app_phi_bounds}
\phi_1\le3,
\qquad
\phi_2\ge\frac{1-\rho}{2}
=\frac{\eta\kappa_\beta}{2},
\qquad
\phi_3
\le
\frac{456e\log(|\cS||\cA|)}{1-\rho}
=
\frac{456e\log(|\cS||\cA|)}{\eta\kappa_\beta}.
\end{equation}

\paragraph{One-step Markovian Lyapunov recursion.}
The Markovian analysis is difficult because 
$\mathsf F_{\eta,\beta}(Q_k,Y_{k+1})$ is not conditionally unbiased given
$Q_k$. Its effect is controlled by allowing the Markov chain to mix over the
preceding $t_k$ iterations.

Assume the effective stepsizes are nonincreasing and satisfy
\begin{equation}
\epsilon_{k-t_k,k-1}
\le
\min\left\{
\frac{\phi_2}{\phi_3A^2},
\frac1{4A}
\right\},
\qquad k\ge t_k.
\label{eq:app_markov_stepsize_condition}
\end{equation}
Then Lemma~A.6 of \citet{chen2024lyapunov}, specialized to
\eqref{eq:app_markov_contract_sa}, gives $\forall k>t_k$
$$
\mathbb E[
\mathcal M_\theta(Q_{k+1}-Q^\star)
]
\le
(1-\phi_2\epsilon_k)
\mathbb E[
\mathcal M_\theta(Q_k-Q^\star)
]
+
\frac{\phi_3c_2}{2u_{cm}^2}
\epsilon_k\epsilon_{k-t_k,k-1}.
$$
Thus the first term gives contractive drift, while the second captures the
additional error due to temporal dependence.

Let
$$
K:=\min\{k\ge0:k\ge t_k\}.
$$
Using the norm equivalence of the Moreau envelope and repeatedly applying the
preceding recursion gives, for every $k\ge K$,
\begin{equation}
\begin{aligned}
\mathbb E\|Q_k-Q^\star\|_\infty^2
\le\;&
\phi_1c_1
\prod_{j=K}^{k-1}(1-\phi_2\epsilon_j)+
\phi_3c_2
\sum_{i=K}^{k-1}
\epsilon_i\epsilon_{i-t_i,i-1}
\prod_{j=i+1}^{k-1}(1-\phi_2\epsilon_j).
\end{aligned}
\label{eq:app_markov_general_bound}
\end{equation}
Equation \ref{eq:app_markov_general_bound} is the finite-time bound
used below.

\begin{theorem}[Finite-time guarantees for \algo{} under Markovian sampling]
\label{thm:app_finite_rpq}
Suppose Assumption~\ref{ass:behavior} holds and $\kappa_\beta>0$. Set
$\eta=L_\beta^{-1}$ and define the effective stepsize
$\epsilon_k:=\alpha_kL_\beta$.
Define
$$
A_\beta
:=
2+
\frac{(1+\beta/b_{\min})(1+\gamma)}{L_\beta},
\qquad
B_\beta
:=
\frac{(1+\beta/b_{\min})}{L_\beta},
$$
and
$$
c_{1,\beta}
:=
\left(
\|Q_0-Q^\star\|_\infty+\|Q_0\|_\infty
+\frac{B_\beta}{A_\beta}
\right)^2,
\qquad
c_{2,\beta}
:=
\left(
A_\beta\|Q^\star\|_\infty+B_\beta
\right)^2.
$$
Assume throughout that the effective stepsizes are nonincreasing and satisfy
$\epsilon_{k-t_k,k-1}\le\min\left\{
\frac{\kappa_\beta^2}
{912eL_\beta^2A_\beta^2\log(|\cS||\cA|)},
\frac1{4A_\beta}
\right\}$ for every $k\ge t_k$, where
$\epsilon_{i,j}:=\sum_{\ell=i}^j\epsilon_\ell$. Then:

\medskip
\noindent
\textnormal{(i) Constant stepsize.}
If $\alpha_k\equiv\alpha$ and
$$
\alpha L_\beta\,t_{\alpha L_\beta}\le\min\left\{
\frac{\kappa_\beta^2}
{912eL_\beta^2A_\beta^2\log(|\cS||\cA|)},
\frac1{4A_\beta}
\right\},
$$
then, for every $k\ge t_{\alpha L_\beta}$,
\begin{equation}
\mathbb E\|Q_k-Q^\star\|_\infty^2
\le
3c_{1,\beta}
\left(1-\frac{\alpha\kappa_\beta}{2}\right)^{
k-t_{\alpha L_\beta}}
+
912e\,
\frac{\alpha L_\beta^3c_{2,\beta}\log(|\cS||\cA|)}
{\kappa_\beta^2}
\,t_{\alpha L_\beta}.
\label{eq:app_constant_markov_explicit}
\end{equation}
In particular, geometric mixing gives
$t_{\alpha L_\beta}=O(\log(1/(\alpha L_\beta)))$.

\medskip
\noindent
\textnormal{(ii) Inverse-linear stepsize.}
Let
$$
\alpha_k
=
\frac{4}{\kappa_\beta(k+h)},
$$
with $h$ sufficiently large that the above mixing-window condition holds.
Then, for every $k\ge K$,
\begin{equation}
\mathbb E\|Q_k-Q^\star\|_\infty^2
\le
3c_{1,\beta}
\left(\frac{K+h}{k+h}\right)^2
+
58368e^2
\frac{L_\beta^3c_{2,\beta}\log(|\cS||\cA|)}
{\kappa_\beta^3}
\frac{t_k}{k+h},
\label{eq:app_linear_markov_explicit}
\end{equation}
where
$$
t_k
=
t_{\frac{4L_\beta}{\kappa_\beta(k+h)}}.
$$
Hence, under geometric mixing,
$$
\mathbb E\|Q_k-Q^\star\|_\infty^2
=
O\left(\frac{\log k}{k}\right).
$$

\end{theorem}

\begin{proof}
Fix $\eta=L_\beta^{-1}$ throughout. The effective Markovian-SA recursion has
stepsizes $\epsilon_k=\alpha_kL_\beta$ and averaged contraction factor
$$
\rho
=
1-\frac{\kappa_\beta}{L_\beta}.
$$
For the generalized Moreau envelope introduced above, let
$\phi_1,\phi_2,\phi_3$ be the constants of
\citet{chen2024lyapunov}. By \eqref{eq:app_phi_bounds},
\begin{equation}
\phi_1\le3,
\qquad
\phi_2\ge\frac{\kappa_\beta}{2L_\beta},
\qquad
\phi_3\le
\frac{456eL_\beta\log(|\cS||\cA|)}{\kappa_\beta}.
\label{eq:app_phi_bounds_dummy}
\end{equation}

With $\eta=L_\beta^{-1}$, the sample-operator growth constants $c_1, c_2$ given in \eqref{eq:app_c1_c2} are

$$
A_\beta
=
2+
\frac{(1+\beta/b_{\min})(1+\gamma)}{L_\beta}, \quad 
B_\beta
=
B_1+B_2
=
\frac{(1+\beta/b_{\min})}{L_\beta}
$$
and consequently
$$
c_{1,\beta}
=\left(
\|Q_0-Q^\star\|_\infty+\|Q_0\|_\infty
+\frac{B_\beta}{A_\beta}
\right)^2, \quad 
c_{2, \beta}
=
\left(
A_\beta\|Q^\star\|_\infty+B_\beta
\right)^2.
$$

The stepsize requirement of \citet{chen2024lyapunov} is
$$
\epsilon_{k-t_k,k-1}
\le
\min\left\{
\frac{\phi_2}{\phi_3A_\beta^2},
\frac1{4A_\beta}
\right\}.
$$
Using \eqref{eq:app_phi_bounds_dummy},
$$
\frac{\phi_2}{\phi_3A_\beta^2}
\ge
\frac{\kappa_\beta^2}
{912eL_\beta^2A_\beta^2\log(|\cS||\cA|)},
$$
so the sufficient condition
$\epsilon_{k-t_k,k-1}\le\min\left\{
\frac{\kappa_\beta^2}
{912eL_\beta^2A_\beta^2\log(|\cS||\cA|)},
\frac1{4A_\beta}
\right\}$ used in the theorem implies the desired condition.

Their general Markovian-SA bound therefore gives, for every $k\ge K$,
\begin{equation}
\begin{aligned}
\mathbb E\|Q_k-Q^\star\|_\infty^2
\le\;&
\phi_1c_{1,\beta}
\prod_{j=K}^{k-1}(1-\phi_2\epsilon_j)+
\phi_3c_{2,\beta}
\sum_{i=K}^{k-1}
\epsilon_i\epsilon_{i-t_i,i-1}
\prod_{j=i+1}^{k-1}(1-\phi_2\epsilon_j).
\end{aligned}
\label{eq:app_general_markov_bound}
\end{equation}
We evaluate this expression for the various schedules separately.
\paragraph{(i) Constant stepsize.}
Let $\alpha_k\equiv\alpha$, so
$\epsilon_k=\epsilon:=\alpha L_\beta$,
$t_k=t_\epsilon=t_{\alpha L_\beta}$, and
$K=t_\epsilon$. Moreover,
$\epsilon_{i-t_i,i-1}=\epsilon t_\epsilon$.
Equation~\ref{eq:app_general_markov_bound} gives
$$
\mathbb E\|Q_k-Q^\star\|_\infty^2
\le
\phi_1c_{1,\beta}(1-\phi_2\epsilon)^{k-t_\epsilon}
+
\phi_3c_{2,\beta}\epsilon^2t_\epsilon
\sum_{i=t_\epsilon}^{k-1}
(1-\phi_2\epsilon)^{k-i-1}.
$$
Since the geometric sum is at most $(\phi_2\epsilon)^{-1}$,
$$
\mathbb E\|Q_k-Q^\star\|_\infty^2
\le
\phi_1c_{1,\beta}(1-\phi_2\epsilon)^{k-t_\epsilon}
+
\frac{\phi_3}{\phi_2}
c_{2,\beta}\epsilon t_\epsilon.
$$
Using \eqref{eq:app_phi_bounds_dummy} and
$\epsilon=\alpha L_\beta$,
$$
1-\phi_2\epsilon
\le
1-\frac{\alpha\kappa_\beta}{2},
\qquad
\frac{\phi_3}{\phi_2}\epsilon
\le
912e
\frac{\alpha L_\beta^3\log(|\cS||\cA|)}{\kappa_\beta^2}.
$$
Together with $\phi_1\le3$, this gives
\eqref{eq:app_constant_markov_explicit}.

\paragraph{(ii) Inverse-linear stepsize.}
Now let
$$
\alpha_k=\frac{4}{\kappa_\beta(k+h)},
\qquad
\epsilon_k
=
\frac{a}{k+h},
\qquad
a:=\frac{4L_\beta}{\kappa_\beta}.
$$
The inverse-linear specialization of
\citet{chen2024lyapunov} gives
$$
\mathbb E\|Q_k-Q^\star\|_\infty^2
\le
\phi_1c_{1,\beta}
\left(\frac{K+h}{k+h}\right)^{\phi_2a}
+
\frac{8ea^2\phi_3c_{2,\beta}}
{\phi_2a-1}
\frac{t_k}{k+h}.
$$
By \eqref{eq:app_phi_bounds_dummy},
$$
\phi_2a
\ge
\frac{\kappa_\beta}{2L_\beta}
\frac{4L_\beta}{\kappa_\beta}
=
2,
$$
so
$$
\left(\frac{K+h}{k+h}\right)^{\phi_2a}
\le
\left(\frac{K+h}{k+h}\right)^2,
\qquad
\phi_2a-1\ge1.
$$
Furthermore,
$$
8ea^2\phi_3
\le
8e
\frac{16L_\beta^2}{\kappa_\beta^2}
\frac{456eL_\beta\log(|\cS||\cA|)}{\kappa_\beta}
=
58368e^2
\frac{L_\beta^3\log(|\cS||\cA|)}{\kappa_\beta^3}.
$$
Together with $\phi_1\le3$, this proves
\eqref{eq:app_linear_markov_explicit}.

\end{proof}

\subsection{Dependence on $\beta$ and comparison with Q-learning}
\label{app:qlearning-comparison}

We first examine how the finite-time constants of \algo{} depend on the
residual-preconditioning parameter $\beta$. Recall that
$$
L_\beta
=
\max_{s,a}d(s)(\beta+\pi_b(a|s)),
\qquad
\kappa_\beta
=
\min_{s,a}
d(s)\bigl[(1-\gamma)\beta-(1+\gamma)+2\pi_b(a|s)\bigr].
$$
Hence, for large $\beta$,
\begin{equation}
L_\beta
\sim
d_{\max}\beta,
\qquad
\kappa_\beta
\sim
d_{\min}(1-\gamma)\beta.
\label{eq:app_large_beta}
\end{equation}

The Markovian finite-time bounds additionally depend on
$$
A_\beta
=
2+
\frac{(1+\beta/b_{\min})(1+\gamma)}{L_\beta},
\qquad
B_\beta
=
\frac{(1+\beta/b_{\min})}{L_\beta},
$$
and
$$
c_{2,\beta}
=
\left(
A_\beta\|Q^\star\|_\infty+B_\beta
\right)^2.
$$
Using \eqref{eq:app_large_beta},
$$
A_\beta
\longrightarrow
2+\frac{1+\gamma}{d_{\max}b_{\min}},
\qquad
B_\beta
\longrightarrow
\frac{1}{d_{\max}b_{\min}},
$$
and therefore
\begin{equation}
c_{2,\beta}
\longrightarrow
c_{2,\infty}
:=
\left[
\left(
2+\frac{1+\gamma}{d_{\max}b_{\min}}
\right)\|Q^\star\|_\infty
+
\frac{1}{d_{\max}b_{\min}}
\right]^2.
\label{eq:app_c2_large_beta}
\end{equation}
In particular, $c_{2,\infty}=O(b_{\min}^{-2})$ when the remaining MDP
parameters are held fixed.

\paragraph{Constant stepsizes.}
The stochastic radius in
Theorem~\ref{thm:app_finite_rpq}(i) is proportional to
$$
\frac{
\alpha L_\beta^3c_{2,\beta}
}{
\kappa_\beta^2
}
\,t_{\alpha L_\beta}.
$$
By \eqref{eq:app_large_beta},
$$
\frac{L_\beta^3c_{2,\beta}}{\kappa_\beta^2}
\sim
\frac{
d_{\max}^3
}{
d_{\min}^2(1-\gamma)^2
}
\,\beta\,c_{2,\infty}.
$$
Thus, for fixed raw stepsize $\alpha$, increasing $\beta$ improves the
deterministic contraction rate but increases the constant-stepsize stochastic
radius linearly in $\beta$, up to the mixing-time factor. This is the
contraction-variance tradeoff motivating the decreasing-$\beta$ schedules
of Section~\ref{sec:beta-schedules}.

\paragraph{Inverse-linear stepsizes.}
For
$\alpha_k=4/[\kappa_\beta(k+h)]$, the leading stochastic coefficient in
Theorem~\ref{thm:app_finite_rpq}(ii) is proportional to
$$
\frac{L_\beta^3c_{2,\beta}}{\kappa_\beta^3}.
$$
Unlike the constant-stepsize radius, this quantity saturates as
$\beta\to\infty$:
\begin{equation}
\frac{L_\beta^3c_{2,\beta}}{\kappa_\beta^3}
\longrightarrow
\frac{
d_{\max}^3
}{
d_{\min}^3(1-\gamma)^3
}
c_{2,\infty}.
\label{eq:app_large_beta_linear}
\end{equation}
Thus sufficiently large residual preconditioning improves the transient
contraction without worsening the leading $\beta$-dependence of the
inverse-linear asymptotic coefficient.

\paragraph{Comparison with asynchronous Q-learning.}
First, under constant stepsizes, asynchronous Q-learning has contraction gap
$
\mu_{\min}(1-\gamma)
$
with
$
\mu_{\min}:=\min_{s,a} d(s)\pi_b(a|s),
$
whereas for large $\beta$,
$
\kappa_\beta\sim d_{\min}(1-\gamma)\beta.
$
Thus \algo{} removes the multiplicative rare-action factor $b_{\min}$ from the leading deterministic contraction gap and hence obtains faster convergence as seen in Figure~\ref{fig:tabular-constant}.

Second, for inverse-linear stepsizes, \citet[Theorem~B.1(1)(c)]{chen2024lyapunov}
obtain a stochastic prefactor proportional to
$
(1-\beta_1)^{-3}
=
[\mu_{\min}(1-\gamma)]^{-3}
$
for asynchronous Q-learning.
For \algo{}, the corresponding coefficient scales as
$
L_\beta^3c_{2,\beta}/\kappa_\beta^3,
$
which saturates for large $\beta$, with
$
c_{2,\beta}=O(b_{\min}^{-2}).
$
Hence, under the aligned-minima condition
$
\mu_{\min}=d_{\min}b_{\min},
$
the same contractive-SA framework yields a $b_{\min}^{-3}$ dependence for
asynchronous Q-learning and $b_{\min}^{-2}$ for large-$\beta$ \algo{} and thus improves the dependence on the rare-action from $b_{\min}^{-3}$ to $b_{\min}^{-2}$ when using inverse linear stepsizes.

\section{Decreasing Residual-Preconditioning Schedules}
\label{app:beta_schedules}

We first analyze the two-stage schedule separately. Unlike a genuinely
time-varying schedule, each phase is a fixed contractive Markovian SA, so we
may use the phase-specific analysis scale $\eta_j=L_j^{-1}$ in each phase.
This yields a terminal stochastic radius determined by the
low-$\beta$ phase rather than by the largest residual weight used during
learning.

Throughout, let $d:=|\cS||\cA|$ and, for
$j\in\{h,\ell\}$, write
$
L_j:=L_{\beta_j},
$
$
\kappa_j:=\kappa_{\beta_j},
$
$
\epsilon_j:=\alpha L_j,
$
and
$
t_j:=t_{\epsilon_j}.
$
Let $A_j,B_j,c_{2,j}$ denote the fixed-$\beta_j$ constants from
Theorem~\ref{thm:finite_rpq}, and define
$$
\nu_j
:=
912e\,
\frac{\alpha L_j^3c_{2,j}\log d}{\kappa_j^2}\,t_j.
$$
We assume throughout that the fixed-$\beta_j$ Markovian mixing-window
condition of Theorem~\ref{thm:finite_rpq} holds for both phases.

\subsection{Two-stage schedule}
\label{app:two_stage_beta}

Consider
$$
\beta_k
=
\begin{cases}
\beta_h, & 0\le k<K,\\
\beta_\ell, & k\ge K,
\end{cases}
\qquad
\beta_h\ge\beta_\ell,
\qquad
\kappa_\ell>0.
$$
We assume $K\ge t_h$ and $T-K\ge t_\ell$.

\begin{theorem}[Two-stage residual preconditioning]
\label{thm:app_two_stage_beta}
Under the conditions above,
\begin{align}
\E\|Q_T-Q^\star\|_\infty^2
=
O\Bigg(&
c_{1,h}
\exp\left\{
-\frac{\alpha}{2}
\left[
\kappa_h(K-t_h)
+
\kappa_\ell(T-K-t_\ell)
\right]
\right\}
\nonumber\\
&+
\nu_h
\exp\left\{
-\frac{\alpha\kappa_\ell}{2}
(T-K-t_\ell)
\right\}
+
\nu_\ell
\Bigg),
\label{eq:app_two_stage_correct}
\end{align}
where $c_{1,h}$ is the initialization constant of
Theorem~\ref{thm:finite_rpq} for $\beta_h$.

Consequently, after a sufficiently long low-$\beta$ tail,
$$
\E\|Q_T-Q^\star\|_\infty^2
=
O(\nu_\ell)
=
O\left(
\frac{\alpha L_\ell^3c_{2,\ell}\log d}
{\kappa_\ell^2}\,
t_{\alpha L_\ell}
\right).
$$
Thus the terminal stochastic radius is the same order as that of running
constant-$\beta_\ell$ \algo{}, while the transient retains the stronger
contraction accumulated during the initial high-$\beta$ phase.
\end{theorem}

\begin{proof}
We apply the fixed-$\beta$ analysis separately in the two phases.

For the first phase, Theorem~\ref{thm:finite_rpq} gives
$$
\E\|Q_K-Q^\star\|_\infty^2
\le
3c_{1,h}
\exp\left\{
-\frac{\alpha\kappa_h}{2}(K-t_h)
\right\}
+
\nu_h.
$$

We now switch to $\beta_\ell$ and use the phase-specific analysis scale
$\eta_\ell=L_\ell^{-1}$, so the effective SA stepsize is
$\epsilon_\ell=\alpha L_\ell$. The first $t_\ell$ iterations after the
switch form a short mixing burn-in. By
\citet[Lemma~A.2]{chen2024lyapunov}, since
$\epsilon_\ell t_\ell\le1/(4A_\ell)$,
$$
\|Q_{K+t_\ell}-Q_K\|_\infty
\le
2\epsilon_\ell t_\ell
\left(A_\ell\|Q_K\|_\infty+B_\ell\right).
$$
Writing $e_k:=Q_k-Q^\star$ and using
$\|Q_K\|_\infty\le\|e_K\|_\infty+\|Q^\star\|_\infty$, we obtain
$$
\|e_{K+t_\ell}\|_\infty
\le
(1+2A_\ell\epsilon_\ell t_\ell)\|e_K\|_\infty
+
2\epsilon_\ell t_\ell
(A_\ell\|Q^\star\|_\infty+B_\ell).
$$
Since $A_\ell\epsilon_\ell t_\ell\le1/4$,
$$
\E\|e_{K+t_\ell}\|_\infty^2
\le
\frac92\E\|e_K\|_\infty^2
+
8\epsilon_\ell^2t_\ell^2c_{2,\ell}.
$$

From time $K+t_\ell$ onward, every mixing window lies entirely inside the
low-$\beta$ phase. Hence the fixed-$\beta_\ell$ one-step Moreau recursion
from Appendix~\ref{app:finite-time-rpq} applies with its own constants
$\phi_{1,\ell},\phi_{2,\ell},\phi_{3,\ell}$. Unrolling it gives
$$
\E\|e_T\|_\infty^2
\le
\phi_{1,\ell}
\exp\left\{
-\frac{\alpha\kappa_\ell}{2}(T-K-t_\ell)
\right\}
\E\|e_{K+t_\ell}\|_\infty^2
+
\nu_\ell.
$$
Using $\phi_{1,\ell}\le3$ and substituting the preceding two displays yields
\begin{align*}
\E\|e_T\|_\infty^2
=
O\Bigg(&
c_{1,h}
\exp\left\{
-\frac{\alpha}{2}
\left[
\kappa_h(K-t_h)
+
\kappa_\ell(T-K-t_\ell)
\right]
\right\}
\\
&+
\nu_h
\exp\left\{
-\frac{\alpha\kappa_\ell}{2}(T-K-t_\ell)
\right\}
\\
&+
\epsilon_\ell^2t_\ell^2c_{2,\ell}
\exp\left\{
-\frac{\alpha\kappa_\ell}{2}(T-K-t_\ell)
\right\}
+
\nu_\ell
\Bigg).
\end{align*}
Finally, the fixed-$\beta_\ell$ mixing-window condition,
$\epsilon_\ell t_\ell=O(1)$, together with
$\kappa_\ell\le L_\ell$, implies
$$
\epsilon_\ell^2t_\ell^2c_{2,\ell}
=
O\left(
\frac{\alpha L_\ell^3c_{2,\ell}\log d}
{\kappa_\ell^2}t_\ell
\right)
=
O(\nu_\ell).
$$
Absorbing this short burn-in contribution into $\nu_\ell$ proves
\eqref{eq:app_two_stage_correct}.
\end{proof}

\begin{remark}
The bound separates the two effects cleanly. The initialization error
benefits from contraction in both phases; stochastic error generated during
the high-$\beta$ phase is subsequently attenuated at the low-$\beta$
contraction rate; and the eventual stochastic radius depends only on
$(L_\ell,\kappa_\ell,c_{2,\ell})$. This is the advantage of analyzing the
two fixed phases separately rather than treating the entire schedule with a
single conservative analysis scale.
\end{remark}

\subsection{General decreasing schedules}
\label{app:general_beta_schedule}

We now consider a deterministic decreasing sequence
$
\beta_k\in[\beta_\ell,\beta_h]
$
with $\kappa_{\beta_\ell}>0$.
Ideally, at each iteration one would use the local analysis scale
$
\eta_k=L_{\beta_k}^{-1}.
$
This would produce the local effective stepsize
$
\epsilon_k=\alpha L_{\beta_k}
$
and would retain the phase-specific stochastic scale
$
L_{\beta_k}^3c_{2,\beta_k}/\kappa_{\beta_k}^2.
$
However, this leads to a genuinely time-varying contractive SA. In
particular, the contraction factor
$
\rho_k=1-\kappa_{\beta_k}/L_{\beta_k}
$
changes with $k$.

The generalized Moreau envelope used by
\citet{chen2024lyapunov} is not independent of this contraction factor:
For the sup norm $||\cdot||_\infty$, its smoothing parameter $\theta$ is chosen as a function of $\rho_k$.
Thus using $\eta_k = L_{\beta_k}^{-1}$ at each iteration would produce
a time-varying Lyapunov function $\mathcal M_k$. The standard one-step
argument would then contain an additional variation term
$
\mathcal M_{k+1}(Q_{k+1}-Q^\star)
-
\mathcal M_k(Q_{k+1}-Q^\star),
$
for which \citep{chen2024lyapunov} provides no control. A
contraction-factor-independent smooth approximation of
$\frac12\|\cdot\|_\infty^2$ would remove this difficulty, but such a
construction is not available in the framework we use here.

We therefore use the conservative common scale
$
\eta:=L_h^{-1},
$
where $L_h:=L_{\beta_h}$. The effective SA stepsize is then fixed:
$
\epsilon:=\alpha L_h.
$
For each $k$, define
$$
\mathsf F_k(Q,y)
:=
Q+\frac1{L_h}H_{\beta_k}(Q,y),
$$
whose stationary mean operator is
$$
\overline{\mathsf F}_k(Q)
=
Q+
\frac1{L_h}
D(B+\beta_kI)(\mathcal TQ-Q).
$$
This operator is a sup-norm contraction with factor
$
\rho_k=1-\kappa_k/L_h,
$
where $\kappa_k:=\kappa_{\beta_k}$. Since
$\kappa_k\ge\kappa_\ell$, all iterations are simultaneously controlled by
the single worst-case contraction
$
\rho_\ell:=1-\kappa_\ell/L_h.
$
We may therefore construct one fixed generalized Moreau envelope of $||\cdot||_\infty^2 $using
$\rho_\ell$ and use it throughout the schedule.

Let $\phi_1,\phi_3$ be the corresponding common Moreau constants and
$
\phi_{2,k}:=1-\rho_k\sqrt{\phi_1}.
$
Then
$
\phi_1\le3,
$
$
\phi_{2,k}\ge\kappa_k/(2L_h),
$
and
$
\phi_3\le456eL_h\log d/\kappa_\ell.
$

Define
$$
A(\beta)
:=
2+\frac{(1+\beta/b_{\min})(1+\gamma)}{L_h},
\qquad
B(\beta)
:=
\frac{(1+\beta/b_{\min})R_{\max}}{L_h},
$$
and
$
c_2(\beta):=(A(\beta)\|Q^\star\|_\infty+B(\beta))^2.
$
Let $A_h:=A(\beta_h)$, and, because the schedule is decreasing, set
$
\widetilde c_{2,k}
:=
c_2(\beta_{(k-t_\epsilon)_+}),
$
where $t_\epsilon:=t_{\alpha L_h}$.

A direct adaptation of Lemmas~A.2-A.6 of
\citet{chen2024lyapunov}, using the fixed Moreau envelope above and the
uniform growth constants corresponding to $\beta_h$, gives
$$
\E\mathcal M(Q_{k+1}-Q^\star)
\le
(1-\phi_{2,k}\epsilon)
\E\mathcal M(Q_k-Q^\star)
+
\frac{\phi_3}{2u_{cm}^2}
\widetilde c_{2,k}\epsilon^2t_\epsilon,
$$
provided
$$
\epsilon t_\epsilon
\le
\min\left\{
\frac{\kappa_\ell^2}
{912eL_h^2A_h^2\log d},
\frac1{4A_h}
\right\}.
$$

Unrolling gives
\begin{align*}
\E\|Q_T-Q^\star\|_\infty^2
\le\;&
3c_1
\exp\left\{
-\frac{\alpha}{2}
\sum_{k=t_\epsilon}^{T-1}\kappa_k
\right\}
\\
&+
456e\,
\frac{\alpha^2L_h^3\log d}{\kappa_\ell}
t_\epsilon
\sum_{i=t_\epsilon}^{T-1}
\widetilde c_{2,i}
\exp\left\{
-\frac{\alpha}{2}
\sum_{j=i+1}^{T-1}\kappa_j
\right\}.
\end{align*}

\begin{theorem}[General decreasing residual-preconditioning schedule]
\label{thm:app_general_beta_schedule}
Let $\{\beta_k\}_{k\ge0}$ be any deterministic non-increasing sequence
satisfying
$
\beta_k\in[\beta_\ell,\beta_h]
$
and
$
\kappa_{\beta_\ell}>0.
$
Let
$
L_h:=L_{\beta_h},
$
$
\epsilon:=\alpha L_h,
$
and
$
t_\epsilon:=t_{\alpha L_h}.
$
Suppose
\[
\epsilon t_\epsilon
\le
\min\left\{
\frac{\kappa_\ell^2}
{912eL_h^2A_h^2\log d},
\frac1{4A_h}
\right\},
\]
where
$
\kappa_\ell:=\kappa_{\beta_\ell}
$
and
$
A_h:=A(\beta_h).
$

Define
\[
\widetilde c_{2,k}
:=
c_2\!\left(\beta_{(k-t_\epsilon)_+}\right),
\]
and
\[
c_1
:=
\left(
\|Q_0-Q^\star\|_\infty
+
\|Q_0\|_\infty
+
\frac{B_h}{A_h}
\right)^2,
\qquad
B_h:=B(\beta_h).
\]
Then, adapting the Markovian contractive-SA argument of
\citet{chen2024lyapunov} to the deterministic sequence of sample operators
$\{\mathsf F_k\}$ gives
\[
\begin{aligned}
\E\|Q_T-Q^\star\|_\infty^2
\le\;&
3c_1
\exp\left\{
-\frac{\alpha}{2}
\sum_{k=t_\epsilon}^{T-1}\kappa_k
\right\}+
456e\,
\frac{\alpha^2L_h^3\log d}{\kappa_\ell}
t_\epsilon
\sum_{i=t_\epsilon}^{T-1}
\widetilde c_{2,i}
\exp\left\{
-\frac{\alpha}{2}
\sum_{j=i+1}^{T-1}\kappa_j
\right\},
\end{aligned}
\]
where
$
\kappa_k:=\kappa_{\beta_k}.
$
\end{theorem}

\begin{proof}
Using the fixed analysis scale $\eta=L_h^{-1}$ and the common generalized
Moreau envelope constructed from the worst contraction factor
$\rho_\ell=1-\kappa_\ell/L_h$, the proof follows the one-step Markovian
Lyapunov argument of \citet[Lemmas~A.2--A.6]{chen2024lyapunov}, with the
sample-map growth constants taken uniformly over the mixing window.
This gives
\[
\E\mathcal M(Q_{k+1}-Q^\star)
\le
(1-\phi_{2,k}\epsilon)
\E\mathcal M(Q_k-Q^\star)
+
\frac{\phi_3}{2u_{cm}^2}
\widetilde c_{2,k}\epsilon^2t_\epsilon.
\]
Unrolling the recursion, using
$
1-x\le e^{-x},
$
$
\phi_1\le3,
$
$
\phi_{2,k}\epsilon\ge\alpha\kappa_k/2,
$
and
$
\phi_3\le456eL_h\log d/\kappa_\ell,
$
yields the stated bound.
\end{proof}

\begin{remark}[Conservatism of the general-schedule bound]
Unlike the two-stage result, the bound above does not recover the
fixed-$\beta_\ell$ terminal radius with $L_\ell^3$. This is not an
algorithmic effect, but a consequence of using a single Lyapunov function
for the entire time-varying schedule: the common scale
$\eta=L_h^{-1}$ propagates $L_h$ into the stochastic coefficient.
Obtaining a bound with the local factor $L_{\beta_k}$ at each iteration
would require a finite-time theory for time-varying contractive Markovian
SA, or equivalently a way to control the variation of the associated
Moreau envelopes. We leave such an extension for future work.
\end{remark}

\subsection{Smoothly decreasing schedules}
\label{app:smooth_beta_schedules}

For explicit cumulative-contraction bounds, define
$
\underline\kappa_\beta
:=
d_{\min}\bigl[(1-\gamma)\beta-(1+\gamma)+2b_{\min}\bigr].
$
Since $\kappa_\beta$ is the pointwise minimum of finitely many affine
functions,
$
0<\underline\kappa_\beta\le\kappa_\beta
$
throughout the contraction regime.

We also use that $c_2(\beta)$ is Lipschitz on
$[\beta_\ell,\beta_h]$. Thus, for some $L_c<\infty$,
$
|c_2(\beta)-c_2(\beta')|
\le
L_c|\beta-\beta'|.
$
Write
$
\underline\kappa_h:=\underline\kappa_{\beta_h},
$
$
\underline\kappa_\ell:=\underline\kappa_{\beta_\ell},
$
and
$
c_{2,\ell}:=c_2(\beta_\ell).
$
Recall that
$
\epsilon=\alpha L_h
$
and
$
t_\epsilon=t_{\alpha L_h}.
$

\begin{corollary}[Linear decay]
\label{cor:app_linear_beta}
Consider
\[
\beta_k
=
\beta_\ell
+
(\beta_h-\beta_\ell)
\left(1-\frac{k}{T-1}\right).
\]
Under the conditions of
Theorem~\ref{thm:app_general_beta_schedule},
\begin{align}
\E\|Q_T-Q^\star\|_\infty^2
\le\;&
3c_1
\exp\left\{
-\frac{\alpha T}{4}
(\underline\kappa_h+\underline\kappa_\ell)
+O(\alpha t_\epsilon)
\right\}
\nonumber\\
&+
O\left(
\frac{
\alpha L_h^3 c_{2,\ell}\log d
}{
\kappa_\ell^2
}
t_\epsilon
\right)
\nonumber\\
&+
O\left(
\frac{
L_h^3L_c\log d
}{
\kappa_\ell^3
}
\frac{t_\epsilon}{T}
\right)
+
O\left(
\frac{
\alpha L_h^3L_c\log d
}{
\kappa_\ell^2
}
\frac{t_\epsilon^2}{T}
\right).
\label{eq:app_linear_beta_markov}
\end{align}
\end{corollary}

\begin{proof}
Affineness of $\underline\kappa_\beta$ gives
$
\sum_{k=0}^{T-1}\underline\kappa_{\beta_k}
=
T(\underline\kappa_h+\underline\kappa_\ell)/2.
$
Since Theorem~\ref{thm:app_general_beta_schedule} starts the contraction
sum at $t_\epsilon$, this contributes an $O(\alpha t_\epsilon)$
correction in the exponent. For $r:=T-1-i$,
$
|\beta_{(i-t_\epsilon)_+}-\beta_\ell|
=
O((r+t_\epsilon)/T),
$
and hence
$
\widetilde c_{2,i}
\le
c_{2,\ell}+O(L_c(r+t_\epsilon)/T).
$
Using
$
\sum_{j=i+1}^{T-1}\kappa_j\ge\kappa_\ell r,
$
together with
$
\sum_{r\ge0}e^{-\alpha\kappa_\ell r/2}
=
O((\alpha\kappa_\ell)^{-1})
$
and
$
\sum_{r\ge0}re^{-\alpha\kappa_\ell r/2}
=
O((\alpha^2\kappa_\ell^2)^{-1}),
$
gives the result.
\end{proof}

\begin{corollary}[Polynomial decay]
\label{cor:app_poly_beta}
For $p>0$, consider
\[
\beta_k
=
\beta_\ell
+
(\beta_h-\beta_\ell)
\left(1-\frac{k}{T-1}\right)^p.
\]
Under the conditions of
Theorem~\ref{thm:app_general_beta_schedule},
\begin{align}
\E\|Q_T-Q^\star\|_\infty^2
\le\;&
3c_1
\exp\left\{
-\frac{\alpha T}{2}
\left[
\underline\kappa_\ell
+
\frac{\underline\kappa_h-\underline\kappa_\ell}{p+1}
\right]
+O(\alpha t_\epsilon)
\right\}
\nonumber\\
&+
O\left(
\frac{
\alpha L_h^3c_{2,\ell}\log d
}{
\kappa_\ell^2
}
t_\epsilon
\right)
\nonumber\\
&+
O\left(
\frac{
L_h^3L_c\log d
}{
\kappa_\ell^{p+2}
}
\frac{
t_\epsilon
}{
\alpha^{p-1}T^p
}
\right)
+
O\left(
\frac{
\alpha L_h^3L_c\log d
}{
\kappa_\ell^2
}
\frac{
t_\epsilon^{p+1}
}{
T^p
}
\right).
\label{eq:app_poly_beta_markov}
\end{align}
\end{corollary}

\begin{proof}
A Riemann-sum estimate gives
$
\sum_{k=0}^{T-1}\underline\kappa_{\beta_k}
=
T[\underline\kappa_\ell+
(\underline\kappa_h-\underline\kappa_\ell)/(p+1)]
+O(1).
$
Truncating the contraction sum at $t_\epsilon$ adds
$O(\alpha t_\epsilon)$ to the exponent. Moreover, for
$r:=T-1-i$,
$
|\beta_{(i-t_\epsilon)_+}-\beta_\ell|
=
O((r+t_\epsilon)^p/T^p).
$
For fixed $p>0$,
$
(r+t_\epsilon)^p
\le
C_p(r^p+t_\epsilon^p).
$
Using Lipschitz continuity of $c_2(\beta)$,
$
\sum_{j=i+1}^{T-1}\kappa_j\ge\kappa_\ell r,
$
$
\sum_{r\ge0}r^pe^{-\alpha\kappa_\ell r/2}
=
O((\alpha\kappa_\ell)^{-(p+1)}),
$
and
$
\sum_{r\ge0}e^{-\alpha\kappa_\ell r/2}
=
O((\alpha\kappa_\ell)^{-1})
$
in Theorem~\ref{thm:app_general_beta_schedule} gives the result.
\end{proof}

\begin{corollary}[Normalized exponential decay]
\label{cor:app_exp_beta}
For $\lambda>0$, define
\[
g_\lambda(x)
:=
\frac{e^{-\lambda x}-e^{-\lambda}}
{1-e^{-\lambda}},
\qquad
\beta_k
=
\beta_\ell
+
(\beta_h-\beta_\ell)
g_\lambda\left(\frac{k}{T-1}\right),
\]
and
$
A_\lambda
:=
\int_0^1g_\lambda(x)\,dx
=
\lambda^{-1}-(e^\lambda-1)^{-1}.
$
Then, under the conditions of
Theorem~\ref{thm:app_general_beta_schedule},
\begin{align}
\E\|Q_T-Q^\star\|_\infty^2
\le\;&
3c_1
\exp\left\{
-\frac{\alpha T}{2}
\left[
\underline\kappa_\ell
+
(\underline\kappa_h-\underline\kappa_\ell)A_\lambda
\right]
+O(\alpha t_\epsilon)
\right\}
\nonumber\\
&+
O\left(
\frac{
\alpha L_h^3c_{2,\ell}\log d
}{
\kappa_\ell^2
}
t_\epsilon
\right)
\nonumber\\
&+
O_\lambda\left(
\frac{
L_h^3L_c\log d
}{
\kappa_\ell^3
}
\frac{t_\epsilon}{T}
\right)
+
O_\lambda\left(
\frac{
\alpha L_h^3L_c\log d
}{
\kappa_\ell^2
}
\frac{t_\epsilon^2}{T}
\right).
\label{eq:app_exp_beta_markov}
\end{align}
\end{corollary}

\begin{proof}
A Riemann-sum approximation gives
$
\sum_{k=0}^{T-1}\underline\kappa_{\beta_k}
=
T[\underline\kappa_\ell+
(\underline\kappa_h-\underline\kappa_\ell)A_\lambda]
+O(1).
$
Truncating the contraction sum at $t_\epsilon$ contributes
$O(\alpha t_\epsilon)$ in the exponent. Near the terminal point,
$
g_\lambda(1-u)
=
\lambda u/(e^\lambda-1)+O_\lambda(u^2),
$
so for $r:=T-1-i$,
$
|\beta_{(i-t_\epsilon)_+}-\beta_\ell|
=
O_\lambda((r+t_\epsilon)/T).
$
The remainder is identical to the linear-decay argument.
\end{proof}
\section{Two-Timescale \algo: Detailed Proofs}
\label{app:two_timescale_rpq}
\begin{algorithm}[H]
\caption{\algo{}}
\label{alg:rpq_two}
\begin{algorithmic}[1]
\State Initialize $V_0,A_0$, state $S_0$, stepsizes $\{\alpha_t, \eta_t\}_{t=0}^\infty$, and choose $\beta>0$.
\For{$t=0,1,\ldots$}
    \State Sample $a_t\sim \pi_b(\cdot\mid s_t)$.
    \State Observe $R_{t+1}=r(s_t,a_t)$ and
    $s_{t+1}\sim P(\cdot\mid s_t,a_t)$.
    \State Compute
    $$
    \delta_t
    =
    R_{t+1}
    +
    \gamma\max_{a'}Q_t(s_{t+1},a')
    -
    Q_t(s_t,a_t).
    $$
    \State Update, for all $s\in\cS$ and $a\in\cA$,
    $$
    V_{t+1}(s)
    =
    V_t(s)
    +
    \mathbf 1\{s_t=s\}\alpha_t\delta_t, \quad 
    A_{t+1}(s,a)
    =
    A_t(s,a)
    +
    \eta_t
    \frac{\beta\,\mathbf 1\{s_t=s,a_t=a\}}
    {\pi_b(a\mid s_t)}
    \delta_t.
    $$
\EndFor
\end{algorithmic}
\end{algorithm}
This appendix proves the two-timescale results stated in
Section~\ref{sec:two_timescale}. Recall that
$(LV)(s,a)=V(s)$,
$B=\operatorname{blkdiag}(\mathbf 1b_s^\top:s\in\mathcal S)$,
and
$D_{(s,a),(s,a)}=d(s)$.
The matrix $B$ averages a state-action vector according to the behavior
policy within each state and lifts the resulting statewise quantity to
all actions at that state.

Throughout, $\beta>0$ is fixed and
$$
\tau_t:=\beta\eta_t
$$
denotes the effective advantage stepsize.

\subsection{Stochastic-approximation form and frozen fast dynamics}
\label{app:two_sa}

The two-timescale \algo{} updates are
\begin{align*}
V_{t+1}(s_t)
&=
V_t(s_t)+\alpha_t\delta_t,
&
A_{t+1}(s_t,a)
&=
A_t(s_t,a)
+
\tau_t
\frac{\mathbf 1\{a_t=a\}}{\pi_b(a|s_t)}
\delta_t,
\end{align*}
where $Q_t=LV_t+A_t$ and
$$
\delta_t
=
r_{s_t, a_t}
+
\gamma\max_{a'}Q_t(s_{t+1},a')
-
Q_t(s_t,a_t).
$$
We consider the regime
$$
\frac{\alpha_t}{\tau_t}\to0,
$$
so that the advantage recursion evolves on the fast timescale and the value
recursion on the slow timescale.

Let
$
Y_{t+1}:=(s_t,a_t,s_{t+1})
$
denote the Markovian observation process. Under
Assumption~\ref{ass:behavior}, $\{Y_t\}$ is a finite-state irreducible and
aperiodic Markov chain with invariant distribution
$$
d_Y(s,a,s')
=
d(s)\pi_b(a|s)P(s'|s,a).
$$
For $Q\in\mathbb R^{|\mathcal S||\mathcal A|}$ and
$y=(s,a,s')$, define the sample temporal-difference residual
$$
\delta(Q,y)
:=
r(s,a)
+
\gamma\max_{a'}Q(s',a')
-
Q(s,a).
$$

The sample advantage field is
$$
H_A(Q,y)(s_0,a_0)
:=
\mathbf 1\{s_0=s\}
\frac{\mathbf 1\{a_0=a\}}{\pi_b(a|s)}
\delta(Q,y),
$$
while the lifted value field is
$$
H_V(Q,y)(s_0,a_0)
:=
\mathbf 1\{s_0=s\}\delta(Q,y).
$$
The updates can therefore be written as
$$
A_{t+1}
=
A_t+\tau_t H_A(Q_t,Y_{t+1}), \quad 
L V_{t+1}
=
L V_t+\alpha_t H_V(Q_t,Y_{t+1}).
$$

Let
$
R(Q):=\mathcal TQ-Q.
$
The mean updates are obtained by integrating with respect to
$d_Y$. For the advantage branch,
\begin{align*}
\bar H_A(Q)(s,a)
&:=
\mathbb E_{Y\sim d_Y}[H_A(Q,Y)(s,a)]
\\
&=
\sum_{s'}
d(s)\pi_b(a|s)P(s'|s,a)
\frac{1}{\pi_b(a|s)}
\left[
r(s,a)
+
\gamma\max_{a'}Q(s',a')
-
Q(s,a)
\right]
\\
&=
d(s)R(Q)(s,a).
\end{align*}
Hence
$
\bar H_A(Q)=DR(Q).
$
Similarly, for the lifted value branch,
\begin{align*}
\bar H_V(Q)(s,a)
&:=
\mathbb E_{Y\sim d_Y}[H_V(Q,Y)(s,a)]
\\
&=
d(s)
\sum_{u}\pi_b(u|s)
\left[
r(s,u)
+
\gamma\sum_{s'}P(s'|s,u)\max_{a'}Q(s',a')
-
Q(s,u)
\right].
\end{align*}
Since $B_s=\textbf{e} b_s^\top$, this is exactly
$$
\bar H_V(Q)=DBR(Q).
$$
Thus the stationary mean updates of the two branches are
$$
\bar H_A(Q)=DR(Q),
\qquad
\bar H_V(Q)=DBR(Q).
$$
We next identify the frozen fast dynamics. Fix an arbitrary value vector
$V$ and consider the limiting ODE of the advantage update:
$$
\dot A
=
D\bigl[
\mathcal T(LV+A)-(LV+A)
\bigr].
$$
Its unique equilibrium is
$
\lambda(V)
:=
Q^\star-LV.
$
Indeed, under the affine change of variables
$
Q=LV+A,
$
the frozen fast limiting ODE becomes
$
\dot Q
=
D(\mathcal TQ-Q).
$
Define
$$
F(Q):=Q+D(\mathcal TQ-Q).
$$
For every $Q,Q'$ and every $(s,a)$,
\begin{align*}
|F(Q)(s,a)-F(Q')(s,a)|
&\le
(1-d(s))|Q(s,a)-Q'(s,a)|
+
d(s)
|\mathcal TQ(s,a)-\mathcal TQ'(s,a)|
\\
&\le
\bigl[1-d(s)(1-\gamma)\bigr]
\|Q-Q'\|_\infty.
\end{align*}
Therefore
\begin{equation}
\label{eq:fast_contraction}
    \|F(Q)-F(Q')\|_\infty
\le
(1-\kappa)\|Q-Q'\|_\infty,
\qquad
\kappa:=d_{\min}(1-\gamma)>0.
\end{equation}
Let $Q(\cdot)$ solve the frozen fast limiting ODE and set
$
V(Q):=\|Q-Q^\star\|_\infty.
$
The preceding contraction \eqref{eq:fast_contraction} implies by the same arguments as \eqref{eq:app_markov_dini_inf},
$$
D^+V(Q(t))
\le
-\kappa V(Q(t)).
$$
Applying the comparison lemma
\citep[Lemma~3.4]{khalil2002nonlinear} gives
$$
\|Q(t)-Q^\star\|_\infty
\le
e^{-\kappa t}
\|Q(0)-Q^\star\|_\infty.
$$
Hence $Q^\star$ is the unique globally exponentially stable equilibrium of
the frozen fast $Q$-dynamics, and therefore
$
\lambda(V)=Q^\star-LV
$
is the unique globally exponentially stable equilibrium of the frozen fast
advantage limiting ODE for every fixed $V$.
\subsection{Almost-sure tracking under timescale separation}
\label{app:two_timescale_asymptotic}

We now prove Theorem~\ref{thm:two_timescale_asymptotic} under Markovian
sampling. A standard two-timescale argument~\citep{borkar2022stochastic} would study $(A_t,V_t)$ jointly:
for a frozen value iterate $V$, the fast advantage recursion has equilibrium
$\lambda(V)=Q^\star-LV$, and the slow iterate evolves along this equilibrium
manifold. In our setting, however, the reduced slow drift vanishes identically,
since
$
LV+\lambda(V)=Q^\star.
$
\[
\dot V = DB(T(LV+\lambda(V))-(LV+\lambda(V))) = 0
\]
Thus every $V$ is an equilibrium of the reduced slow ODE, reflecting the
non-identifiability of the value-advantage decomposition. Standard
two-timescale results \citep{mahadevan2026convergence} requiring convergence of the slow iterate to a unique
equilibrium are therefore unnecessarily strong for our purpose.

The key idea is that the fast tracking error (i.e. the error between the fast iterate and it's equilibrium) is exactly the
 $Q$-error:
$$
\Delta_t
:=
A_t-\lambda(V_t)
=
A_t-(Q^\star-LV_t)
=
Q_t-Q^\star.
$$
After absorbing the fixed $\beta>0$ into the effective fast gain
$\tau_t:=\beta\eta_t$, the recursion for $\Delta_t$ can be written as a
single-timescale Markovian stochastic approximation with a vanishing
perturbation. This strategy of writing the two-timescale stochastic approximation as a single-timescale stochastic approximation with a vanishing perturbation for the faster timescale was used in the proof of \citet[Lemma 8.1]{borkar2022stochastic}, for the martingale difference noise setting. We first establish a general extension of the Markovian ODE
method of \citet{liu2025ode} to such recursions and then specialize it to
\algo{}. This extension of the ODE method \citep{liu2025ode} to vanishing perturbations is analogous to \citet[Extension 3, Section 2.2]{borkar2022stochastic}, which was for martingale difference noise.

We also note that if the noise was indeed Martingale-difference noise i.e. each sample was obtained directly from the stationary distribution of the Markov chain, then \citet[Lemma 8.1]{borkar2022stochastic} would imply
\[
A_t - \lambda(V_t) = Q_t - Q^* \rightarrow 0 \quad a.s.
\]
since $\lambda(V)$ is the globally exponentially stable equilibrium for the fast limiting ODE. In fact, the proof of \citep[Lemma 8.1]{borkar2022stochastic} relies on results for single-timescale SA with vanishing perturbations~\citep[Extension 3, Section 2.2]{borkar2022stochastic}.

\subsubsection{Markovian stochastic approximation with vanishing perturbations}
\label{app:markov_vanishing}

Consider the recursion
\begin{equation}
x_{n+1}
=
x_n
+
a_n\left[
H(x_n,Y_{n+1})
+
r_nG(x_n,Y_{n+1})
\right],
\label{eq:app_vanishing_sa}
\end{equation}
where $\{Y_n\}$ is a Markov chain and $r_n\ge0$ is deterministic with
$r_n\to0$.  Define
$
h(x):=\E_{d_Y}[H(x,Y)]
$
and, for $c\ge1$,
$
H_c(x,y):=H(cx,y)/c.
$
Let $H_\infty$ and
$
h_\infty(x):=\E_{d_Y}[H_\infty(x,Y)]
$
be the scaling limits appearing in Assumptions~3-5 of
\citet{liu2025ode}.

\begin{proposition}[Markovian SA with a vanishing drift perturbation]
\label{prop:markov_vanishing_perturbation}
Suppose Assumptions~1-5 and either Assumption~6 or Assumption~$6'$ of
\citet{liu2025ode} hold for the unperturbed recursion
$
x_{n+1}=x_n+a_nH(x_n,Y_{n+1}).
$
In particular, assume that the ODE
$
\dot x=h_\infty(x)
$
has the origin as its globally asymptotically stable equilibrium.

Assume additionally that, uniformly in $y$,
\[
\|G(x,y)-G(x',y)\|
\le
L_G\|x-x'\|,
\qquad
\sup_y\|G(0,y)\|
\le
G_0
\]
for some finite $L_G,G_0$. If $r_n\to0$, then the iterates of
\eqref{eq:app_vanishing_sa} are almost surely stable:
\[
\sup_n\|x_n\|<\infty
\qquad\text{a.s.}
\]
Moreover, their asymptotic limit set is almost surely a bounded invariant
set of the unperturbed limiting ODE
$
\dot x=h(x).
$
Consequently, if $x^\star$ is the unique globally asymptotically stable
equilibrium of $\dot x=h(x)$, then
$
x_n\to x^\star
$
almost surely.
\end{proposition}

\begin{proof}
We work on an almost-sure sample path on which all conclusions of the
Markovian averaging lemmas of \citet{liu2025ode} hold.  The proof follows
their proof of Theorem~7, with the additional task of showing that the
$r_nG$ term disappears both in the stability rescaling and, after stability,
in the ordinary ODE interpolation. We list the steps to our proof as follows:

\paragraph{Step 1: ODE-time blocks and stability rescaling.}
Following \citet{liu2025ode}, define
\[
t(0):=0,
\qquad
t(n):=\sum_{i=0}^{n-1}a_i,
\qquad
m(t):=\max\{n:t(n)\le t\}.
\]
Fix an arbitrary ODE-time horizon $T>0$ and define block endpoints
\[
T_0:=0,
\qquad
T_{j+1}:=
t\!\left(m(T_j+T)+1\right).
\]
Let $\bar x(t):=x_{m(t)}$ be the piecewise-constant interpolation and set
\[
R_j:=\max\{1,\|\bar x(T_j)\|\},
\qquad
\widehat x_j(s)
:=
\frac{\bar x(T_j+s)}{R_j},
\qquad
0\le s<T_j^{+}:=T_{j+1}-T_j.
\]
Thus
\[
\|\widehat x_j(0)\|\le1
\]
for every block, irrespective of whether the original iterates are bounded.

For $c\ge1$, define the rescaled perturbation
\[
G_c(x,y):=\frac{G(cx,y)}{c}.
\]
The assumptions on $G$ imply the uniform linear-growth estimate
\begin{equation}
\|G_c(x,y)\|
\le
L_G\|x\|+G_0,
\qquad
c\ge1.
\label{eq:G-scaled-growth}
\end{equation}
Indeed,
\[
\frac{\|G(cx,y)\|}{c}
\le
\frac{\|G(cx,y)-G(0,y)\|}{c}
+
\frac{\|G(0,y)\|}{c}
\le
L_G\|x\|+G_0.
\]

For
\[
n_j:=m(T_j),
\qquad
n_j(s):=m(T_j+s),
\]
the normalized recursion on block $j$ is
\[
\widehat x_j(s)
=
\widehat x_j(0)
+
\sum_{i=n_j}^{n_j(s)-1}
a_i
\left[
H_{R_j}
\bigl(
\widehat x_j(t(i)-T_j),Y_{i+1}
\bigr)
+
r_iG_{R_j}
\bigl(
\widehat x_j(t(i)-T_j),Y_{i+1}
\bigr)
\right].
\]

\paragraph{Step 2: uniform boundedness and equicontinuity of the normalized blocks.}

Since $r_n\to0$, there exists $N_0$ such that
$r_n\le1$ for all $n\ge N_0$. For such $n$, Assumption~4 of
\citet{liu2025ode} and \eqref{eq:G-scaled-growth} give
\[
\|H_c(x,y)+r_nG_c(x,y)\|
\le
\|H_c(0,y)\|+G_0+
\bigl(L(y)+L_G\bigr)\|x\|.
\]
Thus the perturbed scaled recursion has the same affine linear-growth
structure as the scaled recursion in the stability proof of
\citet{liu2025ode}, with $L(y)$ replaced by $L(y)+L_G$ and the
origin contribution enlarged by the deterministic constant $G_0$.

The finite-block estimates used by \citet{liu2025ode} remain valid after
these replacements. In particular, their Lemmas~37--39 provide, on the
fixed almost-sure sample path, finite uniform bounds for
\[
\sum_i a_i H_c(0,Y_{i+1})
\qquad\text{and}\qquad
\sum_i a_i L(Y_{i+1})
\]
over ODE-time blocks of fixed length. Adding $L_G$ does not alter any
centered Markovian averaging condition, since
\[
[L(Y_i)+L_G]-[\E L(Y)+L_G]
=
L(Y_i)-\E L(Y),
\]
and the additional origin contribution satisfies
\[
G_0\sum_i a_i=O_T(1)
\]
on every fixed ODE-time block.

Consequently, the same discrete-Gr\"onwall argument as in
\citet[Lemma~40]{liu2025ode} yields, for every fixed $T>0$, a finite
sample-path-dependent constant $C_T(\omega)$ and an index $J_T(\omega)$
such that
\begin{equation}
\sup_{j\ge J_T}
\sup_{0\le s\le T}
\|\widehat x_j(s)\|
\le
C_T(\omega).
\label{eq:normalized-block-bound}
\end{equation}
The finitely many blocks $j<J_T$ can be absorbed by increasing
$C_T(\omega)$, so \eqref{eq:normalized-block-bound} may be taken to hold
for all $j$.  For brevity we henceforth write $C_T$ for this finite
sample-path-dependent constant.

The additional perturbation also preserves the extended equicontinuity
argument of \citet[Lemma~15]{liu2025ode}. Indeed, for
$0\le s_1\le s_2\le T$,
\[
\begin{aligned}
&
\left\|
\sum_{i=n_j(s_1)}^{n_j(s_2)-1}
a_ir_i
G_{R_j}
\bigl(
\widehat x_j(t(i)-T_j),Y_{i+1}
\bigr)
\right\|
\\
&\qquad\le
C_{G,T}
\sum_{i=n_j(s_1)}^{n_j(s_2)-1}a_i,
\end{aligned}
\]
where
\[
C_{G,T}:=L_GC_T+G_0.
\]
The ODE-time increment on the right converges uniformly to
$|s_2-s_1|$, up to the vanishing mesh
$\sup_{i\ge n_j}a_i$. Hence the perturbation contributes a modulus of
continuity that vanishes with $|s_2-s_1|$, and
$\{\widehat x_j(\cdot)\}_j$ is equicontinuous in the extended sense.

We will also need the analogue of the one-block growth estimate used in
the proof of \citet[Lemma~16]{liu2025ode}. Repeating the argument of
\citet[Lemma~42]{liu2025ode} with the same modified growth bound gives
finite sample-path-dependent constants $C_1,C_2$ such that
\begin{equation}
R_{j+1}\le C_1R_j+C_2
\qquad\text{for all sufficiently large }j.
\label{eq:perturbed-block-growth}
\end{equation}
After modifying $C_1,C_2$ to absorb finitely many initial blocks,
\eqref{eq:perturbed-block-growth} may be assumed to hold for every $j$.

\paragraph{Step 3: the perturbation vanishes under the stability rescaling.}

Using \eqref{eq:normalized-block-bound}, for every $0\le s\le T$,
\[
\begin{aligned}
&
\left\|
\sum_{i=n_j}^{n_j(s)-1}
a_ir_i
G_{R_j}
\bigl(
\widehat x_j(t(i)-T_j),Y_{i+1}
\bigr)
\right\|
\\
&\qquad\le
C_{G,T}
\left(\sup_{i\ge n_j}r_i\right)
\sum_{i=n_j}^{n_j(T)-1}a_i .
\end{aligned}
\]
By the definition of $m(\cdot)$,
\[
\sum_{i=n_j}^{n_j(T)-1}a_i
\le
T+\sup_{i\ge n_j}a_i
=
T+o(1).
\]
Since $n_j\to\infty$ and $r_n\to0$,
\[
\sup_{i\ge n_j}r_i\to0.
\]
Hence
\begin{equation}
\sup_{0\le s\le T}
\left\|
\sum_{i=n_j}^{n_j(s)-1}
a_ir_i
G_{R_j}
\bigl(
\widehat x_j(t(i)-T_j),Y_{i+1}
\bigr)
\right\|
\longrightarrow0.
\label{eq:scaled-perturbation-vanishes}
\end{equation}
Thus the perturbation contributes zero to every subsequential limit of
the stability-rescaled interpolations.

\paragraph{Step 4: stability.}

Suppose, toward a contradiction, that
\[
\sup_jR_j=\infty.
\]
The block-growth estimate
\eqref{eq:perturbed-block-growth} allows the same subsequence
construction as in the proof of
\citet[Lemma~16]{liu2025ode}. Hence there exists a subsequence
$\{j_k\}$ such that
\begin{equation}
R_{j_k}\longrightarrow\infty,
\qquad
R_{j_k+1}>R_{j_k}
\quad\text{for every }k.
\label{eq:increasing-scale-subsequence}
\end{equation}
By the uniform boundedness and extended equicontinuity established in
Step~2, Arzel\`a--Ascoli yields a further subsequence, not relabeled,
along which the normalized trajectories
$\widehat x_{j_k}(\cdot)$ converge uniformly on $[0,T]$.

We next identify their limit.  Since neither $H$ nor the driving Markov
chain has been modified, the averaging and scaling arguments for the
$H$-part are exactly those of \citet[Lemmas~17--19]{liu2025ode}.
These lemmas identify the limiting contribution of the $H$-increments
with the ODE at infinity
\[
\dot z=h_\infty(z).
\]
The only additional term in the present recursion is the perturbation,
and by \eqref{eq:scaled-perturbation-vanishes} its cumulative contribution
is $o(1)$ uniformly on $[0,T]$. Therefore the Gr\"onwall argument used in
\citet[Lemma~20]{liu2025ode} continues to show that every normalized
subsequential limit of the perturbed recursion solves the same ODE
\[
\dot z=h_\infty(z).
\]

By Assumption~5 of \citet{liu2025ode}, the origin is globally
asymptotically stable for this ODE.  With the analogues of Lemmas~16 and
20 established above, the contraction argument of
\citet[Lemma~21]{liu2025ode} applies verbatim: choosing $T$ as in that
proof, a sufficiently large normalized block must satisfy
\[
R_{j_k+1}\le R_{j_k}
\]
for some $k$, contradicting
\eqref{eq:increasing-scale-subsequence}. Hence
\[
\sup_jR_j<\infty.
\]

It remains to pass from bounded block endpoints to bounded individual
iterates.  The within-block estimate of
\citet[Lemma~22]{liu2025ode} depends only on the same finite-block growth
bounds used above. Replacing $L(y)$ by $L(y)+L_G$ and enlarging the
origin contribution by $G_0$ therefore gives the same conclusion:
\[
\sup_n\|x_n\|<\infty
\qquad\text{a.s.}
\]
This proves stability of the perturbed recursion.

\paragraph{Step 5: the perturbation vanishes on ordinary two-sided
ODE-time windows.}

Now that stability has been proved, fix the sample path and let
$K(\omega)$ be a compact set containing all iterates
$\{x_n(\omega)\}$. By global Lipschitzness of $G$,
\[
C_K
:=
\sup_{x\in K(\omega),\,y}\|G(x,y)\|
\le
L_G\sup_{x\in K(\omega)}\|x\|+G_0
<\infty.
\]

Let $n_j\to\infty$ be arbitrary and fix $T>0$.  For the forward window,
uniformly for $0\le s\le T$,
\[
\begin{aligned}
&
\left\|
\sum_{i=n_j}^{m(t(n_j)+s)-1}
a_ir_iG(x_i,Y_{i+1})
\right\|
\\
&\qquad\le
C_K
\left(\sup_{i\ge n_j}r_i\right)
\sum_{i=n_j}^{m(t(n_j)+T)-1}a_i
\\
&\qquad\le
C_K
\left(\sup_{i\ge n_j}r_i\right)
\left(
T+\sup_{i\ge n_j}a_i
\right)
\longrightarrow0.
\end{aligned}
\]

For the backward window, define, for all sufficiently large $j$,
\[
\underline n_j:=m(t(n_j)-T).
\]
Then $\underline n_j\to\infty$.  Uniformly for $-T\le s\le0$,
\[
\begin{aligned}
&
\left\|
\sum_{i=m(t(n_j)+s)}^{n_j-1}
a_ir_iG(x_i,Y_{i+1})
\right\|
\\
&\qquad\le
C_K
\left(\sup_{i\ge\underline n_j}r_i\right)
\sum_{i=\underline n_j}^{n_j-1}a_i
\\
&\qquad\le
C_K
\left(\sup_{i\ge\underline n_j}r_i\right)
\left(
T+\sup_{i\ge\underline n_j}a_i
\right)
\longrightarrow0.
\end{aligned}
\]
Consequently, the perturbation has zero asymptotic rate of change
uniformly on every fixed two-sided ODE-time window $[-T,T]$.

\paragraph{Step 6: identification of the limiting ODE and invariant set.}

We now repeat the stable-interpolation argument used in the proof of
\citet[Corollary~8]{liu2025ode}.  That argument considers arbitrary
two-sided windows $[-T,T]$, extracts uniformly convergent subsequences,
and identifies their limits as complete trajectories of the limiting
ODE.

The Markovian averaging arguments for the $H$-increments are unchanged,
because neither $H$ nor $\{Y_n\}$ has been modified.  Step~5 shows that,
on every such two-sided window, the cumulative contribution of the
additional $r_nG$ term converges uniformly to zero. Therefore the
analogues of the interpolation and discretization arguments in the proof
of \citet[Corollary~8]{liu2025ode} yield exactly the same limiting
equation:
\[
\dot x=h(x).
\]

Hence every subsequential limit trajectory of the stable interpolation
is a complete trajectory of the unperturbed ODE $\dot x=h(x)$.  It
follows, exactly as in the proof of
\citet[Corollary~8]{liu2025ode}, that the asymptotic limit set of
$\{x_n\}$ is almost surely a nonempty bounded invariant set of
\[
\dot x=h(x).
\]

If this ODE has a unique globally asymptotically stable equilibrium
$x^\star$, then its only bounded invariant set is $\{x^\star\}$.
Consequently,
\[
x_n\longrightarrow x^\star
\qquad\text{a.s.}
\]
\end{proof}

\subsubsection{Specialization to two-timescale \algo{}}
\label{app:two_timescale_markov_specialization}

We now apply Proposition~\ref{prop:markov_vanishing_perturbation} to the
two-timescale \algo{} recursion. Recall
$$
V_{t+1}(s_t)
=
V_t(s_t)+\alpha_t\delta_t,
\quad 
A_{t+1}(s_t,a)
=
A_t(s_t,a)
+
\tau_t
\frac{\mathbf 1\{a_t=a\}}{\pi_b(a|s_t)}
\delta_t,
\qquad
\tau_t:=\beta\eta_t.
$$
Here
$$
Q_t=LV_t+A_t,
\qquad
\delta_t
=
r(s_t, a_t)
+
\gamma\max_{a'}Q_t(s_{t+1},a')
-
Q_t(s_t,a_t).
$$

Let
$$
\Delta_t:=Q_t-Q^\star
=
A_t-(Q^\star-LV_t).
$$
For
$y=(s,a,s')$, define
$$
\delta_\Delta(\Delta,y)
:=
r(s,a)
+
\gamma\max_{a'}
\bigl[Q^\star(s',a')+\Delta(s',a')\bigr]
-
Q^\star(s,a)-\Delta(s,a).
$$
The action-specific fast update is
$$
H_A(\Delta,y)(s_0,a_0)
:=
\mathbf 1\{s_0=s\}
\frac{\mathbf 1\{a_0=a\}}{\pi_b(a|s)}
\delta_\Delta(\Delta,y),
$$
while the shared-value update is
$$
H_V(\Delta,y)(s_0,a_0)
:=
\mathbf 1\{s_0=s\}
\delta_\Delta(\Delta,y).
$$
The  error recursion is therefore
\begin{equation}
\Delta_{t+1}
=
\Delta_t
+
\tau_t
\left[
H_A(\Delta_t,Y_{t+1})
+
\frac{\alpha_t}{\tau_t}
H_V(\Delta_t,Y_{t+1})
\right].
\label{eq:app_two_delta_markov}
\end{equation}
Thus \eqref{eq:app_two_delta_markov} has precisely the form
\eqref{eq:app_vanishing_sa}, with
$$
a_t=\tau_t,
\qquad
r_t=\frac{\alpha_t}{\tau_t},
\qquad
H=H_A,
\qquad
G=H_V.
$$

\paragraph{Verification of the Markovian SA assumptions.}
We verify the assumptions of
Proposition~\ref{prop:markov_vanishing_perturbation} for the dominant
field $H_A$.

Under Assumption~\ref{ass:behavior},
$Y_{t+1}=(s_t,a_t,s_{t+1})$ is a finite-state irreducible and aperiodic
Markov chain with invariant distribution
$$
d_Y(s,a,s')
=
d(s)\pi_b(a|s)P(s'|s,a).
$$
For every $\Delta,\Delta'$ and $y$,
$$
|\delta_\Delta(\Delta,y)-\delta_\Delta(\Delta',y)|
\le
(1+\gamma)\|\Delta-\Delta'\|_\infty.
$$
Hence, with
$b_{\min}:=\min_{s,a}\pi_b(a|s)>0$,
$$
\|H_A(\Delta,y)-H_A(\Delta',y)\|_\infty
\le
\frac{1+\gamma}{b_{\min}}
\|\Delta-\Delta'\|_\infty,
$$
and
$$
\|H_V(\Delta,y)-H_V(\Delta',y)\|_\infty
\le
(1+\gamma)
\|\Delta-\Delta'\|_\infty.
$$
Thus both sample fields are globally Lipschitz with linear growth.
Since the Markov state space is finite and rewards are bounded, all required
integrability conditions hold.

Averaging $H_A$ under $d_Y$ gives
\begin{equation}
h_A(\Delta)
=
D\left[
\mathcal T(Q^\star+\Delta)
-
(Q^\star+\Delta)
\right].
\label{eq:app_two_fast_mean}
\end{equation}
Indeed, importance weighting cancels the behavior probability:
$$
d(s)\pi_b(a|s)
\frac1{\pi_b(a|s)}
=
d(s).
$$
Similarly,
$$
h_V(\Delta)
=
DB\left[
\mathcal T(Q^\star+\Delta)
-
(Q^\star+\Delta)
\right].
$$
Only $h_A$ determines the limiting ODE because
$\alpha_t/\tau_t\to0$.

For $c\ge1$, define
$H_{A,c}(\Delta,y):=H_A(c\Delta,y)/c$. Since $Q^\star$ and the reward are
$O(1)$ under this scaling,
$H_{A,c}$ converges uniformly on compact sets to
$$
H_{A,\infty}(\Delta,y)(s_0,a)
=
\mathbf 1\{s_0=s\}
\frac{\mathbf 1\{u=a\}}{\pi_b(a|s)}
\left[
\gamma\max_{a'}\Delta(s',a')
-\Delta(s,u)
\right].
$$
Its stationary average is
\begin{equation}
h_{A,\infty}(\Delta)
=
D(\mathcal T_0\Delta-\Delta),
\label{eq:app_two_fast_infty}
\end{equation}
where
$$
(\mathcal T_0\Delta)(s,a)
=
\gamma
\sum_{s'}P(s'|s,a)
\max_{a'}\Delta(s',a').
$$
The finite-state LLN/LIL conditions required by
\citet{liu2025ode} hold exactly as in
Appendix~\ref{app:rpq_asymptotic}.

\paragraph{Global exponential stability of the dominant limiting ODE.}
Let
$$
\kappa_{\rm fast}
:=
d_{\min}(1-\gamma).
$$
Since $\mathcal T$ is a $\gamma$-contraction, the diagonally preconditioned
Bellman map satisfies, for all $\zeta>0, \quad \zeta < d_{\max}^{-1}$,
$$
\left\|
\Delta
+
\zeta h_A(\Delta)
\right\|_\infty
\le
(1-\zeta\kappa_{\rm fast})
\|\Delta\|_\infty.
$$
The proof is identical to the strategy used to prove   \eqref{eq:async_q_learning_contraction}.
Let $\Delta(\cdot)$ solve $\dot\Delta=h_A(\Delta)$ and set
$W(\Delta):=\|\Delta\|_\infty$. Since $h_A$ is globally Lipschitz,
$$
\Delta(t+h)
=
\Delta(t)
+
h h_A(\Delta(t))
+
r_t(h),
\qquad
\frac{\|r_t(h)\|_\infty}{h}\to0.
$$
Using the $1$-Lipschitz property of the sup norm and the preceding
contraction inequality,
$$
W(\Delta(t+h))
\le
(1-h\kappa_{\rm fast})W(\Delta(t))
+
\|r_t(h)\|_\infty.
$$
Therefore
$$
D^+W(\Delta(t))
\le
-\kappa_{\rm fast}W(\Delta(t)).
$$
Applying the comparison lemma
\citep[Lemma~3.4]{khalil2002nonlinear} to
$\dot u=-\kappa_{\rm fast}u$ gives
$$
\|\Delta(t)\|_\infty
\le
e^{-\kappa_{\rm fast}t}
\|\Delta(0)\|_\infty.
$$
Hence the origin is the unique globally exponentially stable equilibrium of
the dominant limiting ODE.

The same argument applies to the limiting ODE
$\dot\Delta=h_{A,\infty}(\Delta)$ because $\mathcal T_0$ is a
$\gamma$-contraction with unique fixed point $0$. Thus
$$
D^+\|\Delta(t)\|_\infty
\le
-\kappa_{\rm fast}\|\Delta(t)\|_\infty,
$$
and the origin is globally exponentially stable for the ODE at infinity.
All stability hypotheses of
Proposition~\ref{prop:markov_vanishing_perturbation} are therefore satisfied.

\begin{theorem}[Almost-sure convergence of two-timescale \algo{}]
\label{thm:app_two_asymptotic}
Fix any $\beta>0$ and let $\tau_t:=\beta\eta_t$. Suppose
Assumption~\ref{ass:behavior} holds and
$$
\tau_t=\frac{\tau_0}{(t+h)^p},
\qquad
\alpha_t=\frac{\alpha_0}{(t+h)^q},
\qquad
\frac12<p<q\le1.
$$
Then
$$
A_t-\bigl(Q^\star-LV_t\bigr)\to0
\qquad\text{a.s.},
$$
and consequently
$$
Q_t=LV_t+A_t\to Q^\star
\qquad\text{a.s.}
$$
\end{theorem}

\begin{proof}
The recursion \eqref{eq:app_two_delta_markov} satisfies
Proposition~\ref{prop:markov_vanishing_perturbation} with
$a_t=\tau_t$ and
$$
r_t
=
\frac{\alpha_t}{\tau_t}
=
\frac{\alpha_0}{\tau_0}
(t+h)^{p-q}
\longrightarrow0,
$$
since $p<q$. Moreover,
$\tau_t=\tau_0/(t+h)^p$ satisfies the Markovian stochastic-approximation
stepsize conditions because $1/2<p\le1$.

The dominant limiting ODE
$$
\dot\Delta=h_A(\Delta)
$$
has the origin as its unique globally exponentially stable equilibrium, and
the corresponding limiting ODE at infinity is also globally exponentially
stable, as shown above. The sample maps satisfy the required global
Lipschitz and integrability assumptions, while the finite-state ergodic
behavior chain satisfies the required Markov averaging conditions.

Proposition~\ref{prop:markov_vanishing_perturbation} therefore gives
$$
\Delta_t\to0
\qquad\text{a.s.}
$$
Since
$$
\Delta_t
=
A_t-(Q^\star-LV_t)
=
Q_t-Q^\star,
$$
both claims follow.
\end{proof}

\begin{remark}[Why no threshold on $\beta$ is required]
The dominant fast dynamics in
\eqref{eq:app_two_delta_markov} are governed by $H_A$ after the fixed
$\beta>0$ is absorbed into the effective gain $\tau_t=\beta\eta_t$.
Consequently, the limiting ODE has contraction scale
$d_{\min}(1-\gamma)$ and contains no residual-preconditioning threshold.
The shared-value dynamics enter only through the vanishing perturbation
$\alpha_t/\tau_t$. Thus, unlike single-timescale \algo{}, the two-timescale
result holds for every fixed $\beta>0$.
\end{remark}

\begin{remark}[Non-identifiability of the decomposition]
The theorem does not assert that $V_t$ or $A_t$ converges individually. What converges is the  action-value iterate
$Q_t=LV_t+A_t$, which is the quantity of interest.
\end{remark}

\subsection{Finite-time Markovian stochastic approximation with a small perturbation}
\label{app:perturbed_markovian_sa}

We require the following extension of the contractive Markovian
stochastic-approximation framework of \citet{chen2024lyapunov}. The key
difference is the presence of a state-dependent Markovian perturbation whose
magnitude is controlled by a deterministic sequence $\{r_k\}$.

Consider the recursion
\begin{equation}
x_{k+1}
=
x_k
+
a_k
\left[
F(x_k,Y_k)-x_k+r_kG(x_k,Y_k)
\right],
\label{eq:perturbed_markov_sa}
\end{equation}
where $\{Y_k\}$ is a finite-state Markov chain, $\{a_k\}$ is a
non-increasing stepsize sequence, and
\begin{equation}
0\le r_k\le \bar r
\qquad\text{for all }k.
\label{eq:rbar}
\end{equation}
Let $\mu$ denote the stationary distribution of $\{Y_k\}$ and define
$$
\bar F(x):=\E_{Y\sim\mu}[F(x,Y)],
\qquad
\bar G(x):=\E_{Y\sim\mu}[G(x,Y)].
$$

Throughout this subsection, $\|\cdot\|_c$ denotes an arbitrary norm on
$\R^d$.

\begin{assumption}[Dominant contractive operator]
\label{ass:perturbed_F}
There exist $A_F,B_F>0$ such that, for every $x,x'\in\R^d$ and
$y\in\mathcal Y$,
$$
\|F(x,y)-F(x',y)\|_c
\le
A_F\|x-x'\|_c,
\qquad
\|F(0,y)\|_c
\le
B_F.
$$
Moreover, $\bar F$ is a $\rho$-contraction in $\|\cdot\|_c$, for some
$\rho\in(0,1)$, with unique fixed point $x^\star$:
$$
\|\bar F(x)-\bar F(x')\|_c
\le
\rho\|x-x'\|_c,
\qquad
\bar F(x^\star)=x^\star.
$$
\end{assumption}

\begin{assumption}[Perturbation operator]
\label{ass:perturbed_G}
There exist $A_G,B_G,L_G>0$ such that, for every $x,x'\in\R^d$ and
$y\in\mathcal Y$,
$$
\|G(x,y)-G(x',y)\|_c
\le
A_G\|x-x'\|_c,
\qquad
\|G(0,y)\|_c
\le
B_G,
$$
and
$$
\|\bar G(x)-\bar G(x')\|_c
\le
L_G\|x-x'\|_c.
$$
We additionally assume that the perturbation preserves the target:
\begin{equation}
\bar G(x^\star)=0.
\label{eq:G_fixed_point}
\end{equation}
\end{assumption}

\begin{assumption}[Geometric mixing]
\label{ass:perturbed_mixing}
The Markov chain $\{Y_k\}$ admits the unique stationary distribution $\mu$
and there exist $C_{\rm mix}>0$ and $\sigma\in(0,1)$ such that
$$
\max_{y\in\mathcal Y}
\|P^m(y,\cdot)-\mu\|_{\rm TV}
\le
C_{\rm mix}\sigma^m,
\qquad m\ge0.
$$
For $\delta>0$, let
$$
t_\delta
:=
\min\left\{
m\ge0:
\max_y\|P^m(y,\cdot)-\mu\|_{\rm TV}\le\delta
\right\},
$$
and define
$$
t_k:=\max\{1,t_{a_k}\},
\qquad
a_{i,j}:=\sum_{\ell=i}^j a_\ell.
$$
\end{assumption}

We use the following constants. First, the complete sample increment in
\eqref{eq:perturbed_markov_sa} satisfies an affine growth bound with
\begin{equation}
A_0:=A_F+1+\bar r A_G,
\qquad
B_0:=B_F+\bar r B_G.
\label{eq:A0B0}
\end{equation}
For the Markovian-noise estimates, define the slightly enlarged constants
\begin{equation}
\widehat A
:=
A_0+A_G+L_G,
\qquad
\widehat B
:=
B_0+B_G.
\label{eq:AhatBhat}
\end{equation}
In particular,
$A_0\le\widehat A$, $A_G+L_G\le\widehat A$,
$B_0\le\widehat B$, and $B_G\le\widehat B$.

\paragraph{Fixed Moreau envelope.}
Importantly, the Lyapunov function is constructed using only the contraction
factor $\rho$ of the dominant operator $\bar F$ and is therefore independent
of $r_k$.

Let
$$
f(x):=\frac12\|x\|_c^2,
\qquad
g(x):=\frac12\|x\|_s^2,
$$
where $g$ is $L$-smooth with respect to $\|\cdot\|_s$. Let
$\ell_{cs},u_{cs}>0$ satisfy
$$
\ell_{cs}\|x\|_s
\le
\|x\|_c
\le
u_{cs}\|x\|_s.
$$
Choose $\theta>0$ such that
\begin{equation}
\rho^2
<
\frac{1+\theta\ell_{cs}^2}
     {1+\theta u_{cs}^2},
\label{eq:theta_perturbed}
\end{equation}
and define
$$
M(x)
:=
\min_{u\in\R^d}
\left\{
\frac12\|u\|_c^2
+
\frac{1}{2\theta}\|x-u\|_s^2
\right\}.
$$
By the generalized Moreau-envelope properties used in
\citet{chen2024lyapunov}, there exists a norm $\|\cdot\|_m$ such that
$$
M(x)=\frac12\|x\|_m^2.
$$
Moreover, with
$$
\ell_{cm}:=\sqrt{1+\theta\ell_{cs}^2},
\qquad
u_{cm}:=\sqrt{1+\theta u_{cs}^2},
$$
we have
$$
\ell_{cm}\|x\|_m
\le
\|x\|_c
\le
u_{cm}\|x\|_m.
$$
Following \citet{chen2024lyapunov}, define
\begin{equation}
\phi_1
:=
\frac{1+\theta u_{cs}^2}
     {1+\theta\ell_{cs}^2},
\qquad
\phi_2
:=
1-\rho\sqrt{\phi_1},
\qquad
\phi_3
:=
\frac{114L(1+\theta u_{cs}^2)}
     {\theta\ell_{cs}^2}.
\label{eq:phi_perturbed}
\end{equation}
By~\eqref{eq:theta_perturbed}, $\phi_2>0$.

Finally, define the perturbation-drift constant
\begin{equation}
C_G
:=
2L_G\frac{u_{cm}}{\ell_{cm}}.,
\label{eq:CG}
\end{equation}
and
\begin{equation}
\widetilde\phi_3
:=
\frac{(114+112\bar r)Lu_{cm}^2}
     {\theta\ell_{cs}^2}.
\label{eq:phitilde3}
\end{equation}

We now state the finite-time extension.

\begin{theorem}[Contractive Markovian SA with a small Markovian perturbation]
\label{thm:perturbed_markov_sa}
Suppose Assumptions~\ref{ass:perturbed_F}-
\ref{ass:perturbed_mixing} hold. Assume
\begin{equation}
\bar r
\le
\frac{\phi_2}{2C_G},
\label{eq:perturbation_smallness}
\end{equation}
and that the non-increasing stepsizes satisfy, for every $k\ge t_k$,
\begin{equation}
a_{k-t_k,k-1}
\le
\min\left\{
\frac{\phi_2}
     {2\widetilde\phi_3\widehat A^2},
\frac{1}{4A_0}
\right\}.
\label{eq:perturbed_window_condition}
\end{equation}
Let
$$
K:=\min\{k\ge0:k\ge t_k\},
$$
and define
\begin{equation}
c_1
:=
\left(
\|x_0-x^\star\|_c
+
\|x_0\|_c
+
\frac{B_0}{A_0}
\right)^2,
\qquad
c_2
:=
\left(
\widehat A\|x^\star\|_c+\widehat B
\right)^2.
\label{eq:perturbed_c1c2}
\end{equation}
Then, for every $k\ge K$,
\begin{align}
\E\|x_k-x^\star\|_c^2
\le&
\phi_1c_1
\prod_{j=K}^{k-1}(1-\phi_2a_j)+
\widetilde\phi_3c_2
\sum_{i=K}^{k-1}
a_i a_{i-t_i,i-1}
\prod_{j=i+1}^{k-1}
(1-\phi_2a_j).
\label{eq:perturbed_final_bound}
\end{align}
Thus the perturbation preserves the finite-time order of the original
contractive Markovian-SA bound whenever its magnitude is uniformly small
enough to be absorbed into the dominant negative drift.
\end{theorem}

\begin{proof}
Write
$$
e_k:=x_k-x^\star.
$$
By the $L/\theta$-smoothness of $M$ and
\eqref{eq:perturbed_markov_sa},
\begin{align}
M(e_{k+1})
\le{}&
M(e_k)
+
\underbrace{
a_k
\langle
\nabla M(e_k),
\bar F(x_k)-x_k
\rangle
}_{T_1}
+
\underbrace{
a_kr_k
\langle
\nabla M(e_k),
\bar G(x_k)
\rangle
}_{T_1'}
\nonumber\\
&+
\underbrace{
a_k
\langle
\nabla M(e_k),
F(x_k,Y_k)-\bar F(x_k)
\rangle
}_{T_3}
+
\underbrace{
a_kr_k
\langle
\nabla M(e_k),
G(x_k,Y_k)-\bar G(x_k)
\rangle
}_{T_3'}
\nonumber\\
&+
\underbrace{
\frac{La_k^2}{2\theta}
\|
F(x_k,Y_k)-x_k+r_kG(x_k,Y_k)
\|_s^2
}_{T_4}.
\label{eq:perturbed_moreau_expansion}
\end{align}

We now bound each of the terms as follows:


\begin{lemma}[Dominant Moreau drift]
\label{lem:perturbed_dominant_drift}
\[
T_1
\le
-2\phi_2a_kM(e_k) \quad \forall k\ge 1 .
\]
\end{lemma}

\begin{proof}
This is a direct restatement of
\citet[Lemma~A.1]{chen2024lyapunov}, applied to
$e_k=x_k-x^\star$ and multiplied by $a_k$
\end{proof}


\begin{lemma}[Stationary perturbation drift]
\label{lem:perturbed_stationary_drift}
Suppose Assumption~\ref{ass:perturbed_G} holds and
$\bar G(x^\star)=0$. Then
\[
T_1'
\le
C_G a_kr_k M(e_k),
\qquad
C_G
:=
2L_G\frac{u_{cm}}{\ell_{cm}}.
\]
Consequently, if
\[
\bar r
\le
\frac{\phi_2}{2C_G},
\]
then
\[
T_1+T_1'
\le
-\frac{3\phi_2}{2}a_kM(e_k).
\]
\end{lemma}

\begin{proof}
Since
$
M(e)=\frac12\|e\|_m^2,
$
we may write
\[
\nabla M(e_k)
=
\|e_k\|_m v_{e_k},
\qquad
v_{e_k}\in\partial\|e_k\|_m.
\]
By Lemma~\ref{lem:moreau_technical},
$
\|v_{e_k}\|_{m,*}\le1.
$
Hence, by H\"older's inequality,
\[
\begin{aligned}
\langle \nabla M(e_k),\bar G(x_k)\rangle
&=
\|e_k\|_m
\langle v_{e_k},\bar G(x_k)\rangle
\\
&\le
\|e_k\|_m
\|\bar G(x_k)\|_m.
\end{aligned}
\]
Since $\bar G(x^\star)=0$, the Lipschitz assumption and norm equivalence give
\[
\|\bar G(x_k)\|_m
\le
\frac1{\ell_{cm}}
\|\bar G(x_k)-\bar G(x^\star)\|_c
\le
L_G\frac{u_{cm}}{\ell_{cm}}\|e_k\|_m.
\]
Therefore
\[
T_1'
\le
L_G\frac{u_{cm}}{\ell_{cm}}
a_kr_k\|e_k\|_m^2
=
2L_G\frac{u_{cm}}{\ell_{cm}}
a_kr_kM(e_k),
\]
which proves the first claim. Combining this with
Lemma~\ref{lem:perturbed_dominant_drift} and
$r_k\le\bar r\le\phi_2/(2C_G)$ yields
\[
T_1+T_1'
\le
-\left(2\phi_2-C_G\bar r\right)a_kM(e_k)
\le
-\frac{3\phi_2}{2}a_kM(e_k).
\]
\end{proof}


\begin{lemma}[Short-window iterate movement]
\label{lem:perturbed_movement}
Suppose
\[
\|F(x_k,Y_k)-x_k+r_kG(x_k,Y_k)\|_c
\le
A_0\|x_k\|_c+B_0.
\]
If
\[
a_{k-t_k,k-1}
\le
\frac1{4A_0},
\]
then, for every $j\in[k-t_k,k]$,
\[
\|x_j-x_{k-t_k}\|_c
\le
2a_{k-t_k,k-1}
\left(
A_0\|x_{k-t_k}\|_c+B_0
\right),
\]
and
\[
\|x_j-x_{k-t_k}\|_c
\le
4a_{k-t_k,k-1}
\left(
A_0\|x_k\|_c+B_0
\right).
\]
\end{lemma}

\begin{proof}
This is the iterate-movement argument of
\citet[Lemma~A.2]{chen2024lyapunov}.
Its proof uses only the pathwise affine growth bound on the sample
increment and the condition
$a_{k-t_k,k-1}\le1/(4A_0)$.
Replacing the sample increment there by
$
F(x_k,Y_k)-x_k+r_kG(x_k,Y_k)
$
and taking $A=A_0$, $B=B_0$ therefore gives the two stated bounds.
\end{proof}


\begin{lemma}[Dominant Markovian-noise term]
\label{lem:perturbed_T3}
Suppose the assumptions of
Lemma~\ref{lem:perturbed_movement} hold and let
$\widehat A,\widehat B,c_2$ be as defined above. Then
\[
\E[T_3]
\le
\frac{112L\widehat A^2u_{cm}^2}
     {\theta\ell_{cs}^2}
a_ka_{k-t_k,k-1}
\E[M(e_k)]
+
\frac{56L}
     {\theta\ell_{cs}^2}
a_ka_{k-t_k,k-1}c_2.
\]
\end{lemma}

\begin{proof}
This is the Markovian-noise estimate of
\citet[Lemmas~A.3-A.4]{chen2024lyapunov}.
Their proof uses only the smoothness of $M$, the short-window movement
bounds, the Lipschitz and affine-growth bounds for $F$, and the mixing
estimate for $\{Y_k\}$.

Lemma~\ref{lem:perturbed_movement} provides the required movement bounds
for the present recursion, while
$
A_F+1\le\widehat A
$
and
$
B_F\le\widehat B.
$
Thus the same calculation applies with the constants enlarged to
$\widehat A,\widehat B$, yielding the stated inequality.
\end{proof}

\begin{lemma}[Perturbation Markovian-noise term]
\label{lem:perturbed_T3prime}
Under Assumption~\ref{ass:perturbed_G} and the conditions of
Lemma~\ref{lem:perturbed_movement},
\[
\E[T_3']
\le
r_k
\frac{112L\widehat A^2u_{cm}^2}
     {\theta\ell_{cs}^2}
a_ka_{k-t_k,k-1}\E[M(e_k)]
+
r_k
\frac{56L}
     {\theta\ell_{cs}^2}
a_ka_{k-t_k,k-1}c_2 .
\]
\end{lemma}

\begin{proof}
Write $t:=t_k$ and $a_{k-t,k-1}:=\sum_{j=k-t}^{k-1}a_j$.
As in the Markovian-noise decomposition of
\citet[Lemma~A.3]{chen2024lyapunov}, insert the delayed iterate
$x_{k-t}$ and decompose
\[
T_3'=T_{31}'+T_{32}'+T_{33}',
\]
where
\begin{align*}
T_{31}'
:={}&
a_kr_k
\left\langle
\nabla M(e_k)-\nabla M(e_{k-t}),
G(x_k,Y_k)-\bar G(x_k)
\right\rangle,
\\
T_{32}'
:={}&
a_kr_k
\left\langle
\nabla M(e_{k-t}),
G(x_k,Y_k)-G(x_{k-t},Y_k)
\right.
\\[-1mm]
&\hspace{38mm}\left.
+\bar G(x_{k-t})-\bar G(x_k)
\right\rangle,
\\
T_{33}'
:={}&
a_kr_k
\left\langle
\nabla M(e_{k-t}),
G(x_{k-t},Y_k)-\bar G(x_{k-t})
\right\rangle .
\end{align*}

For convenience, define
\[
S_k
:=
\widehat A\|e_k\|_c
+
\widehat A\|x^\star\|_c
+
\widehat B .
\]
By the definitions of $\widehat A,\widehat B$ and
Lemma~\ref{lem:perturbed_movement},
\[
\|x_k-x_{k-t}\|_c
\le
4a_{k-t,k-1}S_k .
\]
Moreover, using the reverse movement estimate exactly as in
\citet[Lemma~A.3]{chen2024lyapunov},
\[
\|e_{k-t}\|_c
\le
2\left(
\|e_k\|_c+\|x^\star\|_c+\frac{\widehat B}{\widehat A}
\right).
\]
Since $M$ is $L/\theta$-smooth with respect to $\|\cdot\|_s$,
\begin{align}
\|\nabla M(e_k)-\nabla M(e_{k-t})\|_{s,*}
&\le
\frac{L}{\theta}\|x_k-x_{k-t}\|_s
\nonumber\\
&\le
\frac{4L}{\theta\ell_{cs}}
a_{k-t,k-1}S_k,
\label{eq:T31prime_grad}
\end{align}
and, because $\nabla M(0)=0$,
\begin{align}
\|\nabla M(e_{k-t})\|_{s,*}
&\le
\frac{L}{\theta}\|e_{k-t}\|_s
\nonumber\\
&\le
\frac{2L}{\theta\ell_{cs}}
\left(
\|e_k\|_c+\|x^\star\|_c+\frac{\widehat B}{\widehat A}
\right).
\label{eq:T32prime_grad}
\end{align}

We now bound the three terms separately.

\paragraph{Bounding $T_{31}'$.}
Since $\bar G(x^\star)=0$, the sample growth and stationary Lipschitz
bounds imply
\begin{align*}
\|G(x_k,Y_k)-\bar G(x_k)\|_c
&\le
\|G(x_k,Y_k)\|_c+\|\bar G(x_k)\|_c
\\
&\le
A_G\|x_k\|_c+B_G+L_G\|e_k\|_c
\\
&\le
\widehat A\|e_k\|_c
+
\widehat A\|x^\star\|_c
+
\widehat B
=
S_k .
\end{align*}
Hence, by H\"older's inequality, norm equivalence, and
\eqref{eq:T31prime_grad},
\begin{align*}
T_{31}'
&\le
\frac{a_kr_k}{\ell_{cs}}
\|\nabla M(e_k)-\nabla M(e_{k-t})\|_{s,*}
\|G(x_k,Y_k)-\bar G(x_k)\|_c
\\
&\le
r_k
\frac{4L}{\theta\ell_{cs}^2}
a_ka_{k-t,k-1}S_k^2.
\end{align*}
Using $(a+b)^2\le2a^2+2b^2$,
\[
S_k^2
\le
2\widehat A^2\|e_k\|_c^2
+
2\left(
\widehat A\|x^\star\|_c+\widehat B
\right)^2.
\]
Since
$
\|e_k\|_c^2\le2u_{cm}^2M(e_k)
$
and
$
c_2:=(\widehat A\|x^\star\|_c+\widehat B)^2,
$
we obtain
\begin{equation}
T_{31}'
\le
r_k
\frac{16L\widehat A^2u_{cm}^2}
     {\theta\ell_{cs}^2}
a_ka_{k-t,k-1}M(e_k)
+
r_k
\frac{8L}
     {\theta\ell_{cs}^2}
a_ka_{k-t,k-1}c_2.
\label{eq:T31prime_full}
\end{equation}

\paragraph{Bounding $T_{32}'$.}
By the sample and stationary Lipschitz properties,
\begin{align*}
\|G(x_k,Y_k)-G(x_{k-t},Y_k)
+\bar G(x_{k-t})-\bar G(x_k)\|_c&\le
(A_G+L_G)\|x_k-x_{k-t}\|_c
\\
\qquad&\le
2\widehat A\|x_k-x_{k-t}\|_c
\\
\qquad&\le
8\widehat A
a_{k-t,k-1}S_k,
\end{align*}
where in the second inequality we used
$A_G,L_G\le\widehat A$.

Combining this with H\"older's inequality,
\eqref{eq:T32prime_grad}, and norm equivalence gives
\begin{align*}
T_{32}'
&\le
\frac{a_kr_k}{\ell_{cs}}
\|\nabla M(e_{k-t})\|_{s,*} \times
\|G(x_k,Y_k)-G(x_{k-t},Y_k)
+\bar G(x_{k-t})-\bar G(x_k)\|_c
\\
&\le
r_k
\frac{16L}{\theta\ell_{cs}^2}
a_ka_{k-t,k-1}S_k^2.
\end{align*}
Applying the same bound on $S_k^2$ as above yields
\begin{equation}
T_{32}'
\le
r_k
\frac{64L\widehat A^2u_{cm}^2}
     {\theta\ell_{cs}^2}
a_ka_{k-t,k-1}M(e_k)
+
r_k
\frac{32L}
     {\theta\ell_{cs}^2}
a_ka_{k-t,k-1}c_2.
\label{eq:T32prime_full}
\end{equation}

\paragraph{Bounding $T_{33}'$.}
Let $\mathcal F_{k-t}$ denote the sigma-field generated by the history
through time $k-t$. Since $x_{k-t}$ and $Y_{k-t}$ are
$\mathcal F_{k-t}$-measurable, the Markov property gives
\begin{align*}
&
\E\!\left[
G(x_{k-t},Y_k)
\,\middle|\,
\mathcal F_{k-t}
\right]
-\bar G(x_{k-t})=
\sum_{y\in\mathcal Y}
\left(
P^t(Y_{k-t},y)-\mu(y)
\right)
G(x_{k-t},y).
\end{align*}
Therefore, using the definition of total-variation distance,
\begin{align*}
&
\left\|
\E[
G(x_{k-t},Y_k)
\mid
\mathcal F_{k-t}]
-\bar G(x_{k-t})
\right\|_c\le
2
\|P^t(Y_{k-t},\cdot)-\mu\|_{\rm TV}
\sup_{y\in\mathcal Y}
\|G(x_{k-t},y)\|_c.
\end{align*}
By the definition $t=t_{a_k}$,
\[
\|P^t(Y_{k-t},\cdot)-\mu\|_{\rm TV}
\le
a_k.
\]
Moreover,
\[
\|G(x_{k-t},y)\|_c
\le
A_G\|x_{k-t}\|_c+B_G.
\]
Using the reverse movement estimate and the definitions of
$\widehat A,\widehat B$ exactly as above,
\[
A_G\|x_{k-t}\|_c+B_G
\le
2S_k.
\]
Consequently,
\begin{equation}
\left\|
\E[
G(x_{k-t},Y_k)
\mid
\mathcal F_{k-t}]
-\bar G(x_{k-t})
\right\|_c
\le
4a_kS_k.
\label{eq:G_mixing_bound}
\end{equation}

Taking conditional expectation in $T_{33}'$ and applying
H\"older's inequality,
\eqref{eq:T32prime_grad}, and
\eqref{eq:G_mixing_bound},
\begin{align*}
\E[T_{33}'\mid\mathcal F_{k-t}]
&\le
\frac{a_kr_k}{\ell_{cs}}
\|\nabla M(e_{k-t})\|_{s,*}\times
\left\|
\E[
G(x_{k-t},Y_k)
\mid
\mathcal F_{k-t}]
-\bar G(x_{k-t})
\right\|_c
\\
&\le
r_k
\frac{8L}{\theta\ell_{cs}^2}
a_k^2S_k^2 .
\end{align*}
Thus,
\begin{align}
\E[T_{33}']
\le{}&
r_k
\frac{32L\widehat A^2u_{cm}^2}
     {\theta\ell_{cs}^2}
a_k^2\E[M(e_k)]+
r_k
\frac{16L}
     {\theta\ell_{cs}^2}
a_k^2c_2.
\label{eq:T33prime_full}
\end{align}

Finally, since $t_k\ge1$ and $\{a_k\}$ is non-increasing,
\[
a_k
\le
a_{k-t_k,k-1}.
\]
Hence \eqref{eq:T33prime_full} can be bounded using
$a_k^2\le a_ka_{k-t_k,k-1}$. Summing
\eqref{eq:T31prime_full},
\eqref{eq:T32prime_full}, and
\eqref{eq:T33prime_full}, and using
\[
16+64+32=112,
\qquad
8+32+16=56,
\]
gives
\[
\E[T_3']
\le
r_k
\frac{112L\widehat A^2u_{cm}^2}
     {\theta\ell_{cs}^2}
a_ka_{k-t_k,k-1}\E[M(e_k)]
+
r_k
\frac{56L}
     {\theta\ell_{cs}^2}
a_ka_{k-t_k,k-1}c_2,
\]
as claimed.
\end{proof}

\begin{lemma}[Second-order Moreau remainder]
\label{lem:perturbed_T4}
Under the uniform increment bound
\[
\|F(x_k,Y_k)-x_k+r_kG(x_k,Y_k)\|_c
\le
A_0\|x_k\|_c+B_0,
\]
we have
\[
T_4
\le
\frac{2L\widehat A^2u_{cm}^2}
     {\theta\ell_{cs}^2}
a_ka_{k-t_k,k-1}M(e_k)
+
\frac{L}
     {\theta\ell_{cs}^2}
a_ka_{k-t_k,k-1}c_2.
\]
\end{lemma}

\begin{proof}
By norm equivalence,
\[
\begin{aligned}
T_4
&=
\frac{La_k^2}{2\theta}
\|F(x_k,Y_k)-x_k+r_kG(x_k,Y_k)\|_s^2
\\
&\le
\frac{La_k^2}{2\theta\ell_{cs}^2}
(A_0\|x_k\|_c+B_0)^2.
\end{aligned}
\]
Using $(a+b)^2\le2a^2+2b^2$,
$A_0\le\widehat A$, the norm equivalence
$
\|e_k\|_c^2\le2u_{cm}^2M(e_k),
$
and the definition of $c_2$ gives
\[
T_4
\le
\frac{2L\widehat A^2u_{cm}^2}
     {\theta\ell_{cs}^2}
a_k^2M(e_k)
+
\frac{L}
     {\theta\ell_{cs}^2}
a_k^2c_2.
\]
Finally,
$
a_k\le a_{k-t_k,k-1}
$
implies the stated result.
\end{proof}
\paragraph{One-step Lyapunov recursion.}
Combining
Lemmas~\ref{lem:perturbed_dominant_drift},
\ref{lem:perturbed_stationary_drift},
\ref{lem:perturbed_T3},
\ref{lem:perturbed_T3prime}, and
\ref{lem:perturbed_T4},
and using $r_k\le \bar r$,, gives
\begin{align}
\E[M(e_{k+1})]
\le&
\Bigg[
1-\frac{3\phi_2}{2}a_k
+
\frac{(114+112\bar r)L\widehat A^2u_{cm}^2}
     {\theta\ell_{cs}^2}
a_ka_{k-t_k,k-1}
\Bigg]
\E[M(e_k)]
\nonumber\\
&+
\frac{(57+56\bar r)L}
     {\theta\ell_{cs}^2}
a_ka_{k-t_k,k-1}c_2.
\label{eq:perturbed_one_step_preabsorb}
\end{align}
By~\eqref{eq:phitilde3}, this is
\begin{align}
\E[M(e_{k+1})]
\le&
\left[
1-\frac{3\phi_2}{2}a_k
+
\widetilde\phi_3\widehat A^2
a_ka_{k-t_k,k-1}
\right]
\E[M(e_k)]+
\frac{\widetilde\phi_3}
     {2u_{cm}^2}
a_ka_{k-t_k,k-1}c_2.
\label{eq:perturbed_one_step_compact}
\end{align}
Under~\eqref{eq:perturbed_window_condition},
$$
\widetilde\phi_3\widehat A^2
a_{k-t_k,k-1}
\le
\frac{\phi_2}{2},
$$
and hence
\begin{equation}
\E[M(e_{k+1})]
\le
(1-\phi_2a_k)\E[M(e_k)]
+
\frac{\widetilde\phi_3}
     {2u_{cm}^2}
a_ka_{k-t_k,k-1}c_2.
\label{eq:perturbed_scalar_recursion}
\end{equation}

\paragraph{Unrolling the recursion.}
Recursively applying~\eqref{eq:perturbed_scalar_recursion} from $K$ to
$k-1$ gives
\begin{align*}
\E[M(e_k)]
\le&
\E[M(e_K)]
\prod_{j=K}^{k-1}(1-\phi_2a_j)+
\frac{\widetilde\phi_3c_2}
     {2u_{cm}^2}
\sum_{i=K}^{k-1}
a_i a_{i-t_i,i-1}
\prod_{j=i+1}^{k-1}(1-\phi_2a_j).
\end{align*}
Since
$$
\|e_k\|_c^2
\le
2u_{cm}^2M(e_k),
\qquad
M(e_K)
\le
\frac{1}{2\ell_{cm}^2}\|e_K\|_c^2,
$$
we obtain
\begin{align}
\E\|e_k\|_c^2
\le&
\phi_1
\E\|e_K\|_c^2
\prod_{j=K}^{k-1}(1-\phi_2a_j)+
\widetilde\phi_3c_2
\sum_{i=K}^{k-1}
a_i a_{i-t_i,i-1}
\prod_{j=i+1}^{k-1}(1-\phi_2a_j).
\label{eq:perturbed_unrolled}
\end{align}

Finally, during the initial interval $k<K$, the affine growth bound with
constants $(A_0,B_0)$ and the proof of
\citet[Lemma~A.2]{chen2024lyapunov} give
$$
\|x_K-x^\star\|_c^2
\le
\left(
\|x_0-x^\star\|_c
+
\|x_0\|_c
+
\frac{B_0}{A_0}
\right)^2
=
c_1.
$$
Substituting this into~\eqref{eq:perturbed_unrolled} proves
\eqref{eq:perturbed_final_bound}.
\end{proof}

\begin{remark}[Interpretation]
\label{rem:perturbed_markov_sa_interpretation}
Relative to the analysis of \citet{chen2024lyapunov}, the perturbation
introduces exactly two new first-order terms. The stationary component
$T_1'$ consumes part of the negative drift and is controlled by the
uniform smallness condition on $\bar r$, while the centered Markovian
component $T_3'$ contributes another mixing-window error of the same order
$a_ka_{k-t_k,k-1}$ as the original Markovian noise. Consequently, a
sufficiently small state-dependent Markovian perturbation changes only the
constants, and not the finite-time order of the contractive Markovian-SA
bound.
\end{remark}

\subsection{Finite-time tracking under timescale separation}
\label{app:two_timescale_finite}

We now apply Theorem~\ref{thm:perturbed_markov_sa} to the two-timescale
\algo{} recursion. Recall the tracking error
\[
\Delta_t
:=
A_t-\lambda(V_t)
=
LV_t+A_t-Q^\star
=
Q_t-Q^\star.
\]
Define
\[
H(\Delta)
:=
\mathcal T(Q^\star+\Delta)-(Q^\star+\Delta).
\]
Since $\mathcal TQ^\star=Q^\star$ and $\mathcal T$ is a
$\gamma$-contraction,
\[
\|H(\Delta)-H(\Delta')\|_\infty
\le
(1+\gamma)\|\Delta-\Delta'\|_\infty,
\qquad
\|H(\Delta)\|_\infty
\le
(1+\gamma)\|\Delta\|_\infty.
\]

For $y=(s,a,s')$, let
\[
\delta_\Delta(\Delta,y)
:=
r(s,a)
+
\gamma\max_{a'}
\bigl[Q^\star(s',a')+\Delta(s',a')\bigr]
-
Q^\star(s,a)
-
\Delta(s,a).
\]
Define the fast and slow sample operators by
\[
H_A(\Delta,y)(s_0,a_0)
:=
\mathbf 1\{s_0=s\}
\frac{\mathbf 1\{a_0=a\}}{\pi_b(a|s)}
\delta_\Delta(\Delta,y),
\]
and
\[
H_V(\Delta,y)(s_0,a_0)
:=
\mathbf 1\{s_0=s\}\delta_\Delta(\Delta,y).
\]

Writing
$
r_t:=\alpha_t/\tau_t,
$
the tracking error satisfies the recursion
\begin{equation}
\Delta_{t+1}
=
\Delta_t
+
\tau_t
\left[
H_A(\Delta_t,Y_{t+1})
+
r_tH_V(\Delta_t,Y_{t+1})
\right].
\label{eq:app_two_tracking_exact}
\end{equation}

\paragraph{Reduction to perturbed Markovian SA.}
To match Theorem~\ref{thm:perturbed_markov_sa}, set
\[
F(\Delta,y)
:=
\Delta+H_A(\Delta,y),
\qquad
G(\Delta,y)
:=
H_V(\Delta,y).
\]
Under the invariant distribution of the Markov chain,
\[
\bar F(\Delta)
=
\Delta+DH(\Delta),
\qquad
\bar G(\Delta)
=
DBH(\Delta).
\]
Both operators vanish at the target $\Delta^\star=0$.

Let
\[
\kappa
:=
d_{\min}(1-\gamma).
\]
The dominant stationary operator satisfies
\[
\|\bar F(x)-\bar F(y)\|_\infty
\le
(1-\kappa)\|x-y\|_\infty.
\]
Hence the fixed contraction factor used throughout the Moreau analysis is
\[
\rho:=1-\kappa.
\]

For the perturbation,
\[
\|\bar G(x)-\bar G(y)\|_\infty
\le
L_G\|x-y\|_\infty,
\qquad
L_G:=d_{\max}(1+\gamma).
\]

\paragraph{Sample regularity.}
Let
\[
b_{\min}
:=
\min_{s,a}\pi_b(a|s),
\qquad
G_\star
:=
1+(1+\gamma)\|Q^\star\|_\infty,
\]
where we used $r(s,a)\in[0,1]$.

The fast sample operator satisfies
\[
\|H_A(x,y)-H_A(x',y)\|_\infty
\le
\frac{1+\gamma}{b_{\min}}
\|x-x'\|_\infty,
\]
and
\[
\|H_A(0,y)\|_\infty
\le
\frac{G_\star}{b_{\min}}.
\]
Likewise,
\[
\|H_V(x,y)-H_V(x',y)\|_\infty
\le
(1+\gamma)\|x-x'\|_\infty,
\qquad
\|H_V(0,y)\|_\infty
\le
G_\star.
\]

Therefore, in the notation of
Theorem~\ref{thm:perturbed_markov_sa},
\[
A_F
=
1+\frac{1+\gamma}{b_{\min}},
\qquad
B_F
=
\frac{G_\star}{b_{\min}},
\]
and
\[
A_G
=
1+\gamma,
\qquad
B_G
=
G_\star.
\]

Suppose
$
r_t\le r_{\max}
$
for all $t$. Define
\[
A_0
:=
A_F+1+r_{\max}A_G,
\qquad
B_0
:=
B_F+r_{\max}B_G,
\]
and
\[
\widehat A
:=
A_0+A_G+L_G,
\qquad
\widehat B
:=
B_0+B_G.
\]

\paragraph{Fixed Moreau envelope and timescale separation.}
Let
\[
d:=|\mathcal S||\mathcal A|,
\qquad
p:=2\log d,
\qquad
g(x):=\frac12\|x\|_p^2.
\]
We use the $\ell_\infty$ specialization of the generalized Moreau envelope
given in Corollary~\ref{cor:moreau_infty}. In particular, with
\[
\theta
:=
\left(\frac{1+\rho}{2\rho}\right)^2-1,
\]
we may choose
\[
\ell_{cs}=e^{-1/2},
\qquad
u_{cs}=1,
\]
and obtain
\[
\phi_1\le3,
\qquad
\phi_2\ge\frac{\kappa}{2},
\qquad
\phi_3
\le
\frac{456e\log d}{\kappa}.
\]

Let
\[
\ell_{cm}
:=
\sqrt{1+\theta\ell_{cs}^2},
\qquad
u_{cm}
:=
\sqrt{1+\theta u_{cs}^2}.
\]
By Lemma~\ref{lem:perturbed_stationary_drift}, the stationary perturbation
satisfies
\[
T_1'
\le
C_G\tau_tr_tM(\Delta_t),
\qquad
C_G
:=
2L_G\frac{u_{cm}}{\ell_{cm}}.
\]
Hence we impose the uniform timescale-separation condition
\begin{equation}
\frac{\alpha_t}{\tau_t}
=
r_t
\le
r_{\max}
\le
c_{\rm sep}
:=
\frac{\phi_2}{2C_G}
=
\frac{\phi_2\ell_{cm}}
     {4L_Gu_{cm}},
\qquad
\forall t.
\label{eq:app_two_sep}
\end{equation}
Under this condition, the slow value update consumes at most one quarter
of the dominant negative Moreau drift.

The Markovian-noise and second-order terms are controlled by the constant
\[
\widetilde\phi_3
:=
\frac{(114+112r_{\max})Lu_{cm}^2}
     {\theta\ell_{cs}^2}.
\]
Using Corollary~\ref{cor:moreau_infty},
\[
\widetilde\phi_3
\le
\frac{456e(1+r_{\max})\log d}{\kappa}.
\]

\paragraph{Mixing-window condition.}
Let
\[
t_t
:=
\max\{1,t_{\tau_t}\},
\qquad
\tau_{i,j}
:=
\sum_{\ell=i}^j\tau_\ell.
\]
The mixing-window condition required by
Theorem~\ref{thm:perturbed_markov_sa} is
\begin{equation}
\tau_{t-t_t,t-1}
\le
\chi_{\rm 2T}
:=
\min\left\{
\frac{\kappa^2}
{1824e(1+r_{\max})\widehat A^2\log d},
\frac1{4A_0}
\right\}.
\label{eq:app_two_window}
\end{equation}

Since the target is $\Delta^\star=0$, define
\[
c_1
:=
\left(
2\|\Delta_0\|_\infty+\frac{B_0}{A_0}
\right)^2,
\qquad
c_2
:=
\widehat B^2,
\]
and let
\[
K
:=
\min\{t:t\ge t_t\}.
\]

Applying Theorem~\ref{thm:perturbed_markov_sa} gives, for every $t\ge K$,
\begin{equation}
\begin{aligned}
\E\|Q_t-Q^\star\|_\infty^2
\le\;&
3c_1
\exp\left\{
-\frac{\kappa}{2}
\sum_{j=K}^{t-1}\tau_j
\right\}+
\frac{456e(1+r_{\max})c_2\log d}{\kappa}
\sum_{i=K}^{t-1}
\tau_i\tau_{i-t_i,i-1}
\exp\left\{
-\frac{\kappa}{2}
\sum_{j=i+1}^{t-1}\tau_j
\right\}.
\end{aligned}
\label{eq:app_two_finite_bound}
\end{equation}
Since
$
A_t-\lambda(V_t)=Q_t-Q^\star,
$
the same bound holds for
$
\E\|A_t-\lambda(V_t)\|_\infty^2.
$

\begin{theorem}[Finite-time tracking of two-timescale \algo{}]
\label{thm:app_two_finite}
Suppose the timescale-separation condition
\eqref{eq:app_two_sep} and the Markovian mixing-window condition
\eqref{eq:app_two_window} hold. Then
\eqref{eq:app_two_finite_bound} holds for every $t\ge K$.

In particular, the slow value update enters the finite-time recursion only
through the ratio
$
r_t=\alpha_t/\tau_t.
$
Once this ratio is uniformly small enough to be absorbed into the dominant
fast drift, it changes only the constants of the Markovian finite-time bound
and does not generate a separate moving-target error term.
\end{theorem}

\paragraph{Polynomial gains.}
Consider
\[
\tau_t
=
\frac{\tau_0}{(t+h)^p},
\qquad
\alpha_t
=
\frac{\alpha_0}{(t+h)^q},
\qquad
\frac12<p<q\le1.
\]
Then
\[
r_t
=
\frac{\alpha_0}{\tau_0}
(t+h)^{-(q-p)}
\]
is non-increasing. Hence the separation condition
\eqref{eq:app_two_sep} is satisfied whenever
\[
\frac{\alpha_0}{\tau_0}
h^{-(q-p)}
\le
c_{\rm sep}.
\]

For sufficiently large $h$, the mixing-window condition
\eqref{eq:app_two_window} also holds. Moreover, for $t\ge K$,
\[
\tau_{t-t_t,t-1}
\le
2t_t\tau_t.
\]
Substituting this into
\eqref{eq:app_two_finite_bound} and applying the standard polynomial-gain
unrolling yields
\begin{equation}
\E\|Q_t-Q^\star\|_\infty^2
=
O\left(
\frac{t_t}{(t+h)^p}
\right).
\label{eq:app_two_poly_rate}
\end{equation}
Since the Markov chain is geometrically mixing,
$
t_t=O(\log(t+h)),
$
and therefore
\[
\E\|Q_t-Q^\star\|_\infty^2
=
O\left(
\frac{\log t}{t^p}
\right).
\]
The same rate holds for
$
\E\|A_t-\lambda(V_t)\|_\infty^2.
$

\begin{remark}[Why no $\alpha_t^2/\tau_t$ term appears] Standard two-timescale finite-time analysis involves a moving target term of order $\alpha_t^2/\tau_t$~\citep{chandak2025finite}. However our bound does not involve such a term. This is because the identity
$
A_t-\lambda(V_t)=Q_t-Q^\star
$
makes the tracking error autonomous. The slow value update therefore does
not move the target of the fast recursion. Instead, it appears as the
state-dependent Markovian perturbation
$
r_tG(\Delta_t,Y_{t+1}).
$
Once $r_t$ is small enough to be absorbed into the fixed fast Moreau drift,
it changes only the constants in the one-step recursion.
\end{remark}
\section{Useful Technical Lemmas}
\label{app:technical-lemma}
\begin{lemma}[Generalized Moreau-envelope properties]
\label{lem:moreau_technical}
Let
$$
f(x):=\frac12\|x\|_c^2,
\qquad
g(x):=\frac12\|x\|_s^2,
$$
where $g$ is $L$-smooth with respect to $\|\cdot\|_s$. Suppose the norms
$\|\cdot\|_c$ and $\|\cdot\|_s$ satisfy
$$
\ell_{cs}\|x\|_s
\le
\|x\|_c
\le
u_{cs}\|x\|_s
\qquad
\forall x\in\mathbb R^d,
$$
for some $\ell_{cs},u_{cs}>0$. For $\theta>0$, define the generalized
Moreau envelope
$$
\mathcal M_\theta(x)
:=
\min_{u\in\mathbb R^d}
\left\{
\frac12\|u\|_c^2
+
\frac1{2\theta}\|x-u\|_s^2
\right\}.
$$
Then the following hold.

\begin{enumerate}
    \item $\mathcal M_\theta$ is convex and $L/\theta$-smooth with
    respect to $\|\cdot\|_s$. In particular, for all $x,y$,
    $$
    \mathcal M_\theta(y)
    \le
    \mathcal M_\theta(x)
    +
    \langle\nabla\mathcal M_\theta(x),y-x\rangle
    +
    \frac{L}{2\theta}\|y-x\|_s^2.
    $$

    \item The envelope is equivalent to $f$:
    $$
    \left(1+\theta{l_{cs}^2}\right)
    \mathcal M_\theta(x)
    \le
    f(x)
    \le
    \left(1+\theta u_{cs}^2\right)
    \mathcal M_\theta(x).
    $$

    \item There exists a norm $\|\cdot\|_m$ such that
    $$
    \mathcal M_\theta(x)
    =
    \frac12\|x\|_m^2.
    $$
    Moreover, defining
    $$
    \ell_{cm}
    :=
    \sqrt{1+\theta\ell_{cs}^2},
    \qquad
    u_{cm}
    :=
    \sqrt{1+\theta u_{cs}^2},
    $$
    we have
    $$
    \ell_{cm}\|x\|_m
    \le
    \|x\|_c
    \le
    u_{cm}\|x\|_m.
    $$
    \item Let $\|\cdot\|_m^\ast$ denote the dual norm of
    $\|\cdot\|_m$. For every $x\neq0$ and every
    $v_x\in\partial\|x\|_m$,
    $$
    \|v_x\|_m^\ast\le1.
    $$
\end{enumerate}
\end{lemma}

\begin{proof}
Items 1-3 are the generalized Moreau-envelope properties of
\citet[Lemma~2.1 and Appendix~A.1]{NEURIPS2020_5d44ee6f}, rewritten using
$\theta$ as the smoothing parameter.

For the final claim, the map $x\mapsto\|x\|_m$ is convex and
$1$-Lipschitz with respect to $\|\cdot\|_m$. Hence, by
\citet[Lemma~A.1]{NEURIPS2020_5d44ee6f}, every
$v_x\in\partial\|x\|_m$ satisfies
$$
\|v_x\|_m^\ast\le1.
$$
\end{proof}

We restate \citep[Lemma 2.1(2)]{chen2024lyapunov} for convenience.

\begin{corollary}[$\ell_\infty$ specialization]
\label{cor:moreau_infty}
Let $\|\cdot\|_c=\|\cdot\|_\infty$,
$\|\cdot\|_s=\|\cdot\|_p$ with
$p=2\log d$. Then one may take
$
L=p-1\le2\log d,
$
$
\ell_{cs}=e^{-1/2},
$
and
$
u_{cs}=1.
$
For
$$
\theta
:=
\left(\frac{1+\rho}{2\rho}\right)^2-1,
$$
the constants above satisfy
$$
\phi_1\le3,
\qquad
\phi_2\ge\frac{1-\rho}{2},
\qquad
\phi_3
\le
\frac{456e\log d}{1-\rho}.
$$
\end{corollary}

\section{Experimental Details}
\label{app:experimental-details}
\subsection{Discrete MDP}
\label{app:experiments-mdp}

We provide further experiments of \algo and it's variants in various settings in Figure~\ref{fig:tabular-app}. Figure~\ref{fig:tabular-app}(b)
reveals an interesting separation between the roles of $\beta$ and the
timescale ratio $r:=\alpha_V/\alpha_A$. Empirically, $\beta$ largely determines
the terminal noise floor, whereas $r$ primarily controls how rapidly the
algorithm approaches that floor. This suggests a potentially useful
two-timescale design principle: the advantage stepsize may determine the
accuracy ultimately reached, while the relative value stepsize controls the
transient speed. A natural direction for future work is to determine whether
this separation can be formalized into sharper sample-complexity guarantees,
with $\alpha_A$ (or $\beta$) controlling the attainable $\epsilon$-accuracy and $\alpha_V$
controlling how quickly that accuracy is reached.
\begin{figure*}[t]
    \centering
    \begin{subfigure}[t]{0.48\textwidth}
        \centering
        \includegraphics[width=\linewidth]{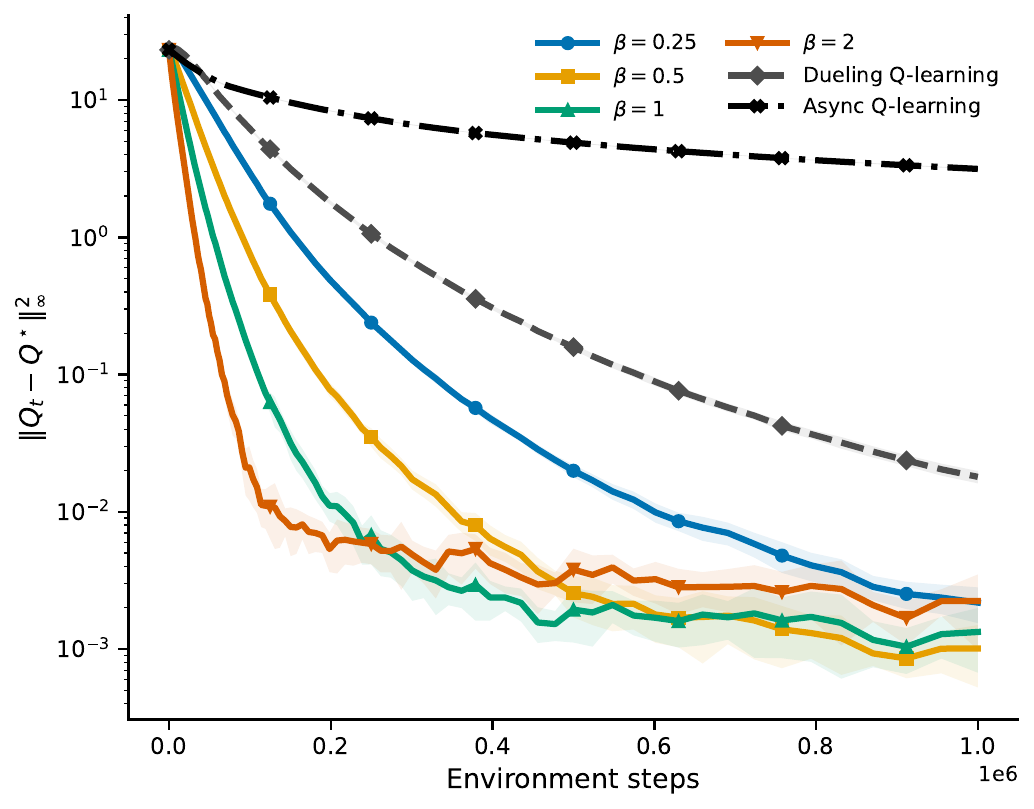}
        \caption{Inverse linear  stepsize.}
        \label{fig:tabular-inverse}
    \end{subfigure}
    \hfill
    \begin{subfigure}[t]{0.48\textwidth}
        \centering
        \includegraphics[width=\linewidth]{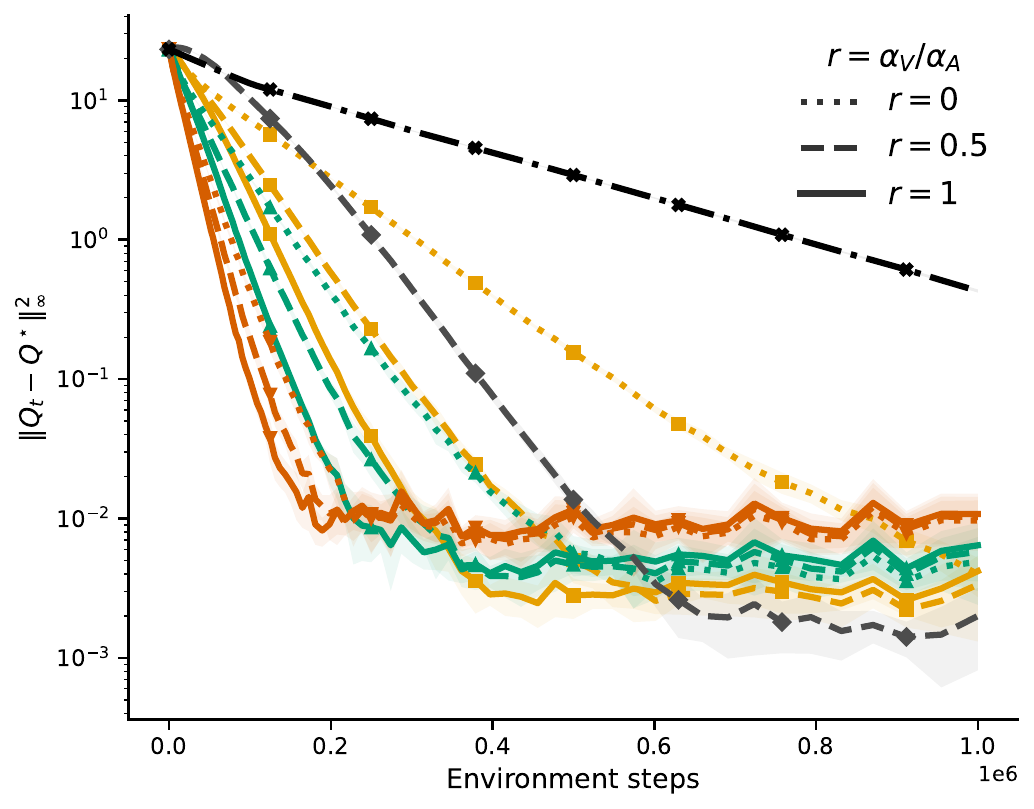}
        \caption{Two-timescale \algo}
        \label{fig:tabular-two-timescale}
    \end{subfigure}
    \caption{
        \textbf{Tabular experiments on randomly generated MDPs.}
        (a): performance of single-timescale \algo (Algorithm \ref{alg:rpq_single}) for different residual
        weights $\beta$ under inverse linear stepsizes, compared
        with asynchronous Q-learning and Dueling Q-learning.
        (b): comparison of two-timescale \algo on with varying stepsize ratios and $\beta$ values using constant stepsizes.
        Curves show the mean squared sup-norm error
        $\mathbb{E}\|Q_t-Q^\star\|_\infty^2$ over independent trajectories;
        shaded regions denote pointwise 95\% confidence intervals. The y-axis is in the logarithmic scale in all plots.
    }
    \label{fig:tabular-app}
\end{figure*}

\subsection{Deep RL Experiments}
Unlike a conventional dueling architecture with a shared feature extractor, our implementation uses separate neural networks for the value and advantage components. This design allows their learning rates to be varied independently, as motivated by the two-timescale analysis in Section~\ref{sec:two_timescale}.

\paragraph{Classic Control.}
For Classic Control, each component is represented by a multilayer perceptron (MLP) based on the CleanRL DQN implementation \citep{huang2022cleanrl}, with hidden dimensions 120 and 84. Each hidden linear layer is followed by LayerNorm and a ReLU activation. The value network terminates in a scalar output \(V(s)\), whereas the advantage network outputs an \(|\mathcal{A}|\)-dimensional vector \(A(s,\cdot)\).

The training horizon is environment dependent: $200,000$ environment steps for Acrobot-v1 and CartPole-v1,$ 500,000$ steps for MountainCar-v0, and \(10^6\) steps for LunarLander-v2. Learning starts after $1,000$ initial random-action steps. We use Adam with a base learning rate of \(2.5\times10^{-4}\), discount factor \(\gamma=0.99\), a replay-buffer capacity of \(10^4\), and minibatches of $128$ transitions, with one optimization step every $10$ environment steps.
The exploration rate is linearly annealed from \(1\) to \(0.05\) over the first \(50\%\) of training.
Unless otherwise specified, we use \(\beta=1\), separate value and advantage networks, and a target network. For the learning-rate-ratio experiments, the value network uses the base learning rate, while the advantage-network learning rate is multiplied by \(0.25\), \(0.5\), \(1\), \(2\), or \(4\).
For the \(\beta\)-scheduling variant, \(\beta\) is linearly decreased from \(100\) to \(1\) over the first \(10\%\) of training and remains fixed at \(1\) thereafter.

We evaluate greedily every \(10\%\) of the corresponding training horizon using five fixed evaluation seeds, \(0,1,2,3,4\), with one episode per evaluation seed. Results are aggregated over five independent training seeds, \(1,\ldots,5\).

\paragraph{MinAtar.}
For MinAtar, each network consists of a convolutional layer with $16$ $(3\times3)$ filters and stride $1$, followed by a ReLU activation, a fully connected layer with 128 hidden units, and another ReLU activation. The advantage network terminates in a linear layer with $(|\mathcal{A}|)$ outputs, whereas the value network terminates in a linear layer with a single scalar output.

Experience is stored in a replay buffer with capacity \(10^6\) transitions. Learning begins after the first 80,000 environment steps, after which the agent samples a minibatch of 32 transitions from the replay buffer every four environment steps. Both the value and advantage networks are optimized independently using Adam with a base learning rate of \(10^{-4}\). During training, actions are selected using an \(\epsilon\)-greedy policy. Exploration begins with \(\epsilon=1\) and is linearly annealed to \(0.01\) over the first $10\%$ of training, corresponding to the first one million environment steps for a 10-million-step run. Thereafter, \(\epsilon\) remains fixed at \(0.01\). 

We periodically evaluate the learned policy every 200,000 environment steps. We evaluate greedily using five fixed evaluation seeds, \(0,1,2,3,4\), running one episode per seed at each evaluation point. We use no exploratory actions during evaluation.

For both Classic Control and MinAtar, the two components are combined as
$$
Q(s,a;\theta_V,\theta_A)
=
V(s;\theta_V)+A(s,a;\theta_A).
$$
For a sampled transition
\((S,A,R,S',\pi_b(A\mid S))\), the bootstrap target is:
\[
Y
=
R+\gamma
\left[
V(S')+\max_{a'}A(S',a')
\right].
\]
The corresponding temporal-difference (TD) residual is
\[
\delta
=
Y-\left[V(S)+A(S,A)\right],
\]
and the TD objective is given by
\[
L_{\mathrm{TD}}
=
\mathbb{E}
\left[
\left(
Y-V(S)-A(S,A)
\right)^2
\right].
\]
Gradients from \(L_{\mathrm{TD}}\) are not propagated
to \(\theta_A\), since the advantage network is updated separately using the preconditioned update rule.

In the deep RL setting, individual updates cannot be directly controlled, and the inverse behavior probability \(1/\pi_b(A\mid S)\) can become large for infrequently selected actions. Such large updates may adversely affect training stability. 
We set \(\beta=1\) by default. For Classic Control, we use the residual-preconditioning weight $ w = \frac{1}{\pi_b(A\mid S)}$ as large inverse behavior probabilities were not a practical issue in these environments.
For MinAtar, we clip the inverse behavior probability by
defining
\[
\rho(S,A)
=
\min\left\{
\frac{1}{\pi_b(A\mid S)},
\rho_{\max}
\right\},
\]
where \(\rho_{\max}=5\). The residual-preconditioning weight is then defined as
\[
w=\beta\rho(S,A) = \rho(S,A).
\]
Using this weight, we construct the auxiliary
advantage target
\[
\hat{A}(S,A)
=
A(S,A)+w\delta.
\]
The corresponding advantage loss is
\[
L_{\mathrm{adv}}
=
\frac{1}{2}
\mathbb{E}
\left[
\left(
A(S,A)-\hat{A}(S,A)
\right)^2
\right].
\]

In addition to the TD and advantage objectives, we apply an explicit
regularization penalty to both components of the action-value decomposition.
This regularization is introduced to address the identifiability issue that
arises in advantage--value decomposition methods
\citep{daley2025analysis,wang2016dueling}. Specifically, we use
\[
L_{\mathrm{reg}}
=
\frac{\lambda}{2}
\mathbb{E}
\left[
V(S)^2+
\sum_{a\in\mathcal{A}}A(S,a)^2
\right],
\]
where \(\lambda=5\times10^{-3}\), we did not optimize for the best \(\lambda\).

The complete training objective is therefore
\[
L
=
L_{\mathrm{TD}}
+
L_{\mathrm{adv}}
+
L_{\mathrm{reg}}.
\]
Importantly, although these terms are written as a single objective, gradients
from \(L_{\mathrm{TD}}\) are stopped with respect to the advantage-network
parameters \(\theta_A\); consequently, \(\theta_A\) is updated through
\(L_{\mathrm{adv}}\) and \(L_{\mathrm{reg}}\), while the TD loss directly
updates the value-network parameters.

\begin{figure}[H]
    \centering
    \includegraphics[width=\linewidth]{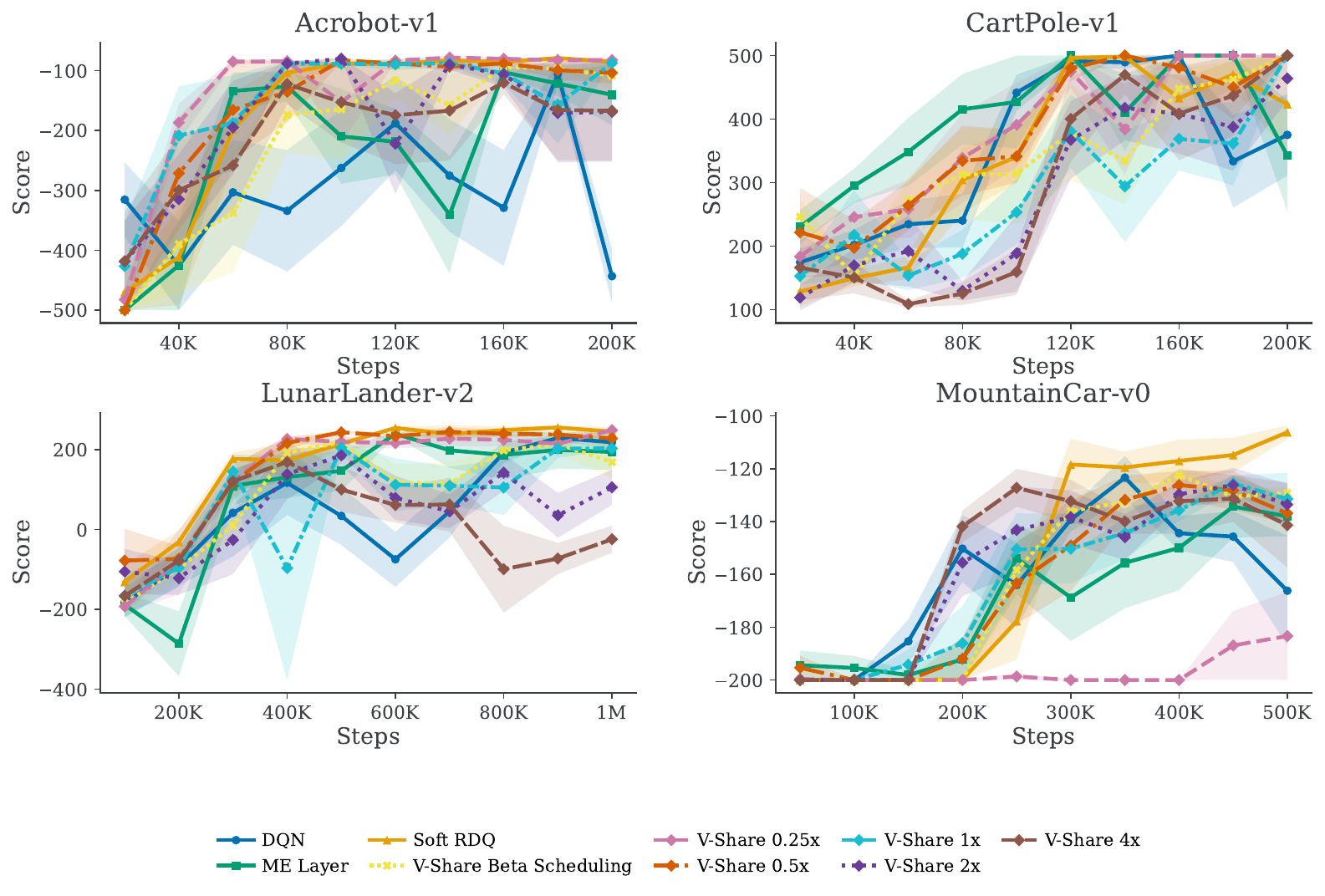}
    \caption{Learning curves for DQN, ME Layer, Soft RDQ, and \algo across four classic-control environments. We evaluate \algo with different advantage learning-rate multipliers and with decreasing-$\beta$ scheduling to study the effects of timescale separation and residual-preconditioning schedules in the deep RL setting. Curves show the mean evaluation return over five independent training seeds, with shaded regions indicating $\pm 1$ SEM.}
    \label{fig:classic RL learning curve}
\end{figure}

\begin{figure}[H]
    \centering
    \includegraphics[width=\linewidth]{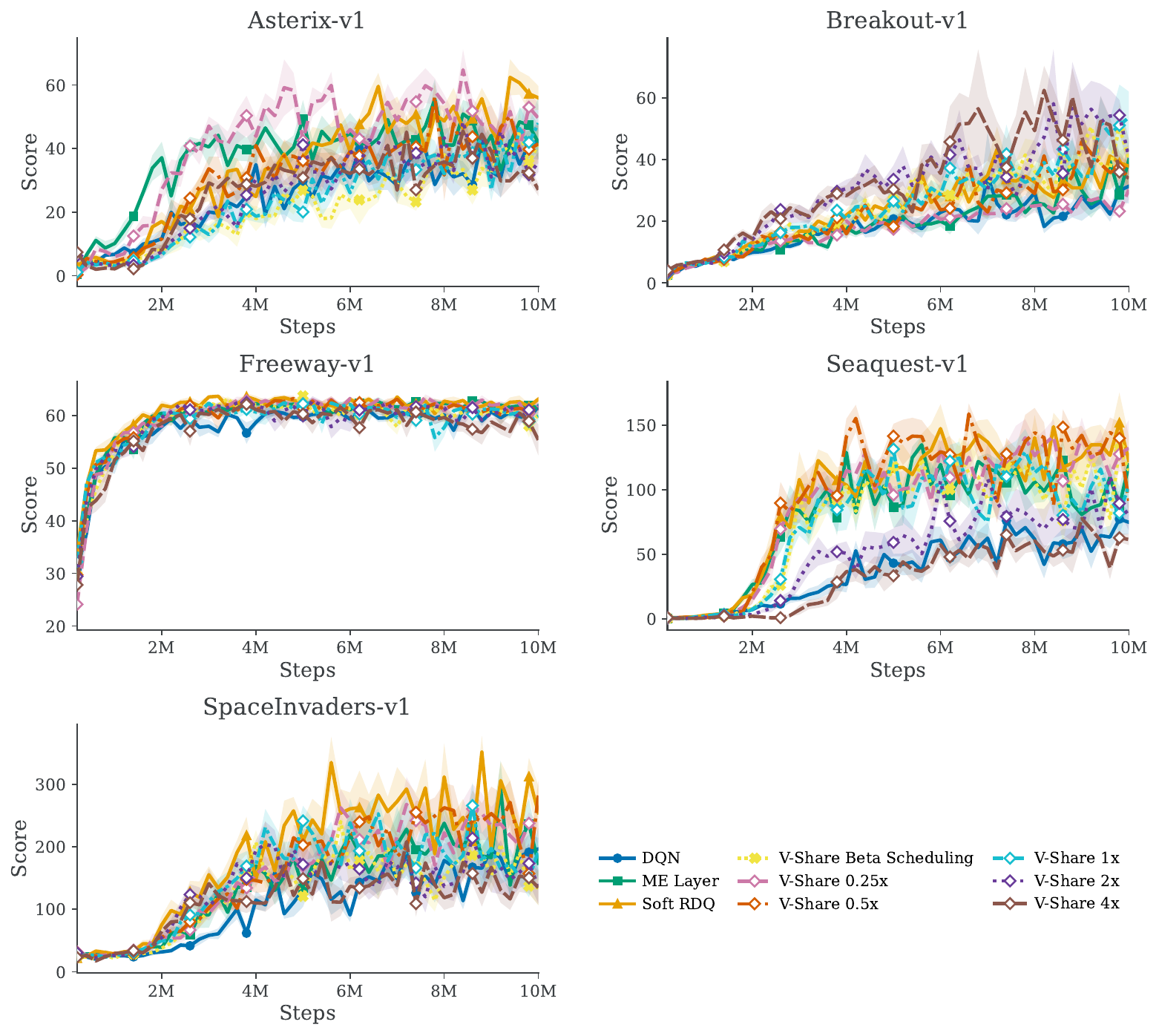}
    \caption{Learning curves for DQN, ME Layer, Soft RDQ, and \algo across five MinAtar environments. We evaluate \algo with different advantage learning-rate multipliers and with decreasing-$\beta$ scheduling to study the effects of timescale separation and residual-preconditioning schedules in the deep RL setting. Curves show the mean evaluation return over six independent training seeds, with shaded regions indicating $\pm 1$ SEM.}
    \label{fig:deep rl learning curve}
\end{figure}

\end{document}